\documentclass[10pt,journal,compsoc]{IEEEtran}
\ifCLASSOPTIONcompsoc
  \usepackage[nocompress]{cite}
\else
  \usepackage{cite}
\fi
\ifCLASSINFOpdf
\else
\fi
\usepackage{xcolor}         
\usepackage{subfigure}
\usepackage{graphicx}
\usepackage{amsmath}
\usepackage{amsthm}
\usepackage{amsfonts}
\usepackage{hyperref}
\hypersetup{
colorlinks=true,
linkcolor=blue,
anchorcolor=blue,
citecolor=blue}

\newtheorem{proposition}{Proposition}[section]
\newtheorem{lemma}{Lemma}[section]
\newtheorem{definition}{Definition}[section]

\newtheorem{remark}{Remark}

\usepackage{algorithm}
\usepackage{algorithmic}
\usepackage{booktabs}
\usepackage{multirow}
\usepackage[normalem]{ulem}
\useunder{\uline}{\ul}{}
\definecolor{revise}{RGB}{0, 0, 0}
\definecolor{R2}{RGB}{0, 0, 0}
\definecolor{R3}{RGB}{0, 0, 0}

\begin{document}
%
\title{Node-oriented Spectral Filtering for Graph Neural Networks}
%
%
%
%

\author{Shuai Zheng,
        Zhenfeng Zhu*,
        Zhizhe Liu,
        Youru Li,
        and~Yao~Zhao,~\IEEEmembership{Fellow,~IEEE}

\IEEEcompsocitemizethanks{\IEEEcompsocthanksitem S. Zheng, Z. Zhu, Z. Liu, and Y. Zhao are with the Institute of Information Science, Beijing Jiaotong University, Beijing 100044, China, and also with the Beijing Key
	Laboratory of Advanced Information Science and Network Technology,
	Beijing 100044, China \protect\\
E-mail: {zs1997, zhfzhu, zhzliu, liyouru, yzhao}@bjtu.edu.cn.}
\thanks{This work was supported in part by Science and Technology Innovation 2030 - "New Generation Artificial Intelligence" Major Project under Grant No. 2018AAA0102101, and in part by the National Natural Science Foundation of China under Grants ( No.61976018, No.U1936212, No.62120106009 ).}
\thanks{*Corresponding author: Zhenfeng Zhu.}
\thanks{Manuscript received April 19, 2005; revised August 26, 2015.}}

%
%

\markboth{Journal of \LaTeX\ Class Files,~Vol.~14, No.~8, August~2015}%
{Shell \MakeLowercase{\textit{et al.}}: Bare Demo of IEEEtran.cls for Computer Society Journals}
%



\IEEEtitleabstractindextext{%
\begin{abstract}
Graph neural networks (GNNs) have shown remarkable performance on homophilic graph data while being far less impressive when handling non-homophilic graph data due to the inherent low-pass filtering property of GNNs. In general, since real-world graphs are often complex mixtures of diverse subgraph patterns, learning a universal spectral filter on the graph from the global perspective as in most current works may still suffer from great difficulty in adapting to the variation of local patterns. On the basis of the theoretical analysis of local patterns, we rethink the existing spectral filtering methods and propose the \textbf{\underline{N}}ode-oriented spectral \textbf{\underline{F}}iltering for \textbf{\underline{G}}raph \textbf{\underline{N}}eural \textbf{\underline{N}}etwork (namely NFGNN). By estimating the node-oriented spectral filter for each node, NFGNN is provided with the capability of precise local node positioning via the generalized translated operator, thus discriminating the variations of local homophily patterns adaptively. Meanwhile, the utilization of re-parameterization brings a good trade-off between global consistency and local sensibility for learning the node-oriented spectral filters. Furthermore, we theoretically analyze the localization property of NFGNN, demonstrating that the signal after adaptive filtering is still positioned around the corresponding node. Extensive experimental results demonstrate that the proposed NFGNN achieves more favorable performance.
\end{abstract}

}

\maketitle

\IEEEdisplaynontitleabstractindextext

%
\IEEEpeerreviewmaketitle

\IEEEraisesectionheading{\section{Introduction}\label{sec:introduction}}

%
%
%
%

\label{sect::introduction}
\IEEEPARstart{A}{s} a potent tool for analyzing graph data, GNNs are attracting considerable attention from both academia and industry. Meanwhile, GNNs have also demonstrated remarkable capabilities in a number of graph-related applications, including but not limited to recommendation system~\cite{lightGCN, ngcf}, disease prediction~\cite{popGCN, InceptionGCN}, drug discovery~\cite{drug1, drug2}, and action recognition~\cite{STGCN, TSAGCN}. 
In the field of graph machine learning, homophily has always remained to be a common assumption~\cite{Homophily,tutorial}, i.e., nodes that connect with each other tend to be within the same class.
However, behind the impressive success of the previous efforts, such  an assumption as a critical limitation doesn't hold in many graph-related scenarios, which severely inhibits the further extension of GNNs to more general graph data. In practice, it is hard to argue that homophily is an inherent characteristic of graph data~\cite{H2GCN} and there are also a considerable number of non-homophilic graphs in the real world, where the links usually exist between nodes from different classes. A typical example is the protein structure network which is widely known as a strong heterophilic network because connections between different types of amino acids are easier to form~\cite{protein_structure}.

As a matter of fact, the relevance of the downstream task to the graph construction always determines whether a graph is homophilic or heterophilic. It means that for a fixed topology structure, when combined with the different label distributions from different downstream tasks, the identification of its graph property will be very different. \textcolor{revise}{Taking the citation network as an example, a paper is more likely to cite papers that study the same or similar topics over multiple time periods. Therefore, the citation graph is likely to be heterophilic or random if we use the published year of the paper as the label, while it may be homophilic if we use the topic of the paper as the label.} In consideration of the practical applications, a simple yet effective GNN  model that can be adaptively applied to graphs with different structural properties should be preferred. To address the issue of non-homophily,  Fig.~\ref{fig::methods} shows a brief overview of the existing GNNs to deal with the non-homophilic graph.

\begin{figure}[t]
	\centering
	\includegraphics[width=3.5in]{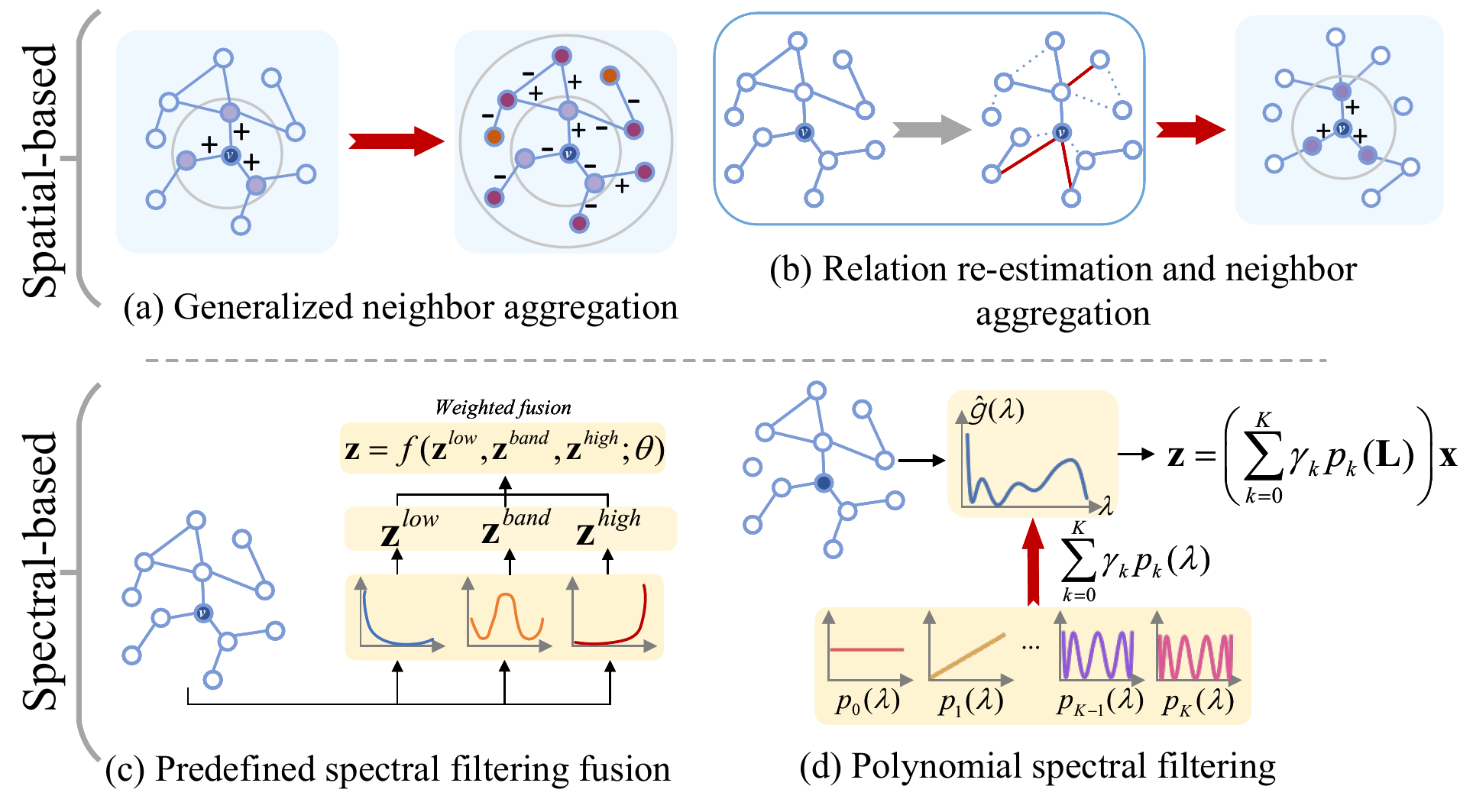}
	\vspace{-3mm}
	\caption{A summary of the existing GNNs for non-homophily. They are broadly classified into two categories: spatial-based and spectral-based, with spatial-based methods further classified into (a) and (b), and spectral-based methods further classified into (c) and (d).}
	\label{fig::methods}
	\vspace{-5mm}
\end{figure}
\begin{figure*}[t]
    \centering
   \subfigure[Cora]{
   \includegraphics[width=35mm]{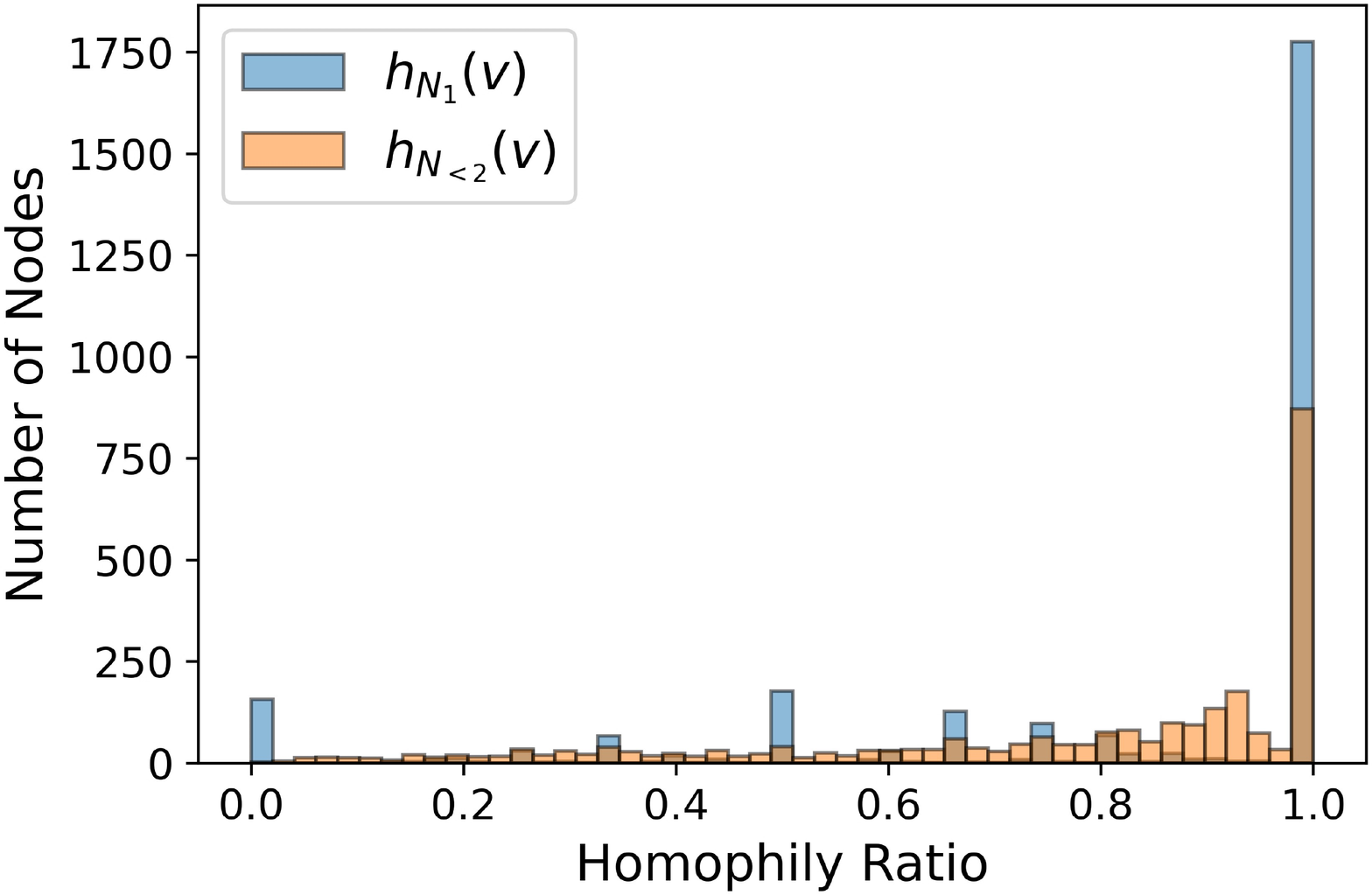}
   }
   \hspace{-2mm}
   \subfigure[Citeseer]{
   \includegraphics[width=35mm]{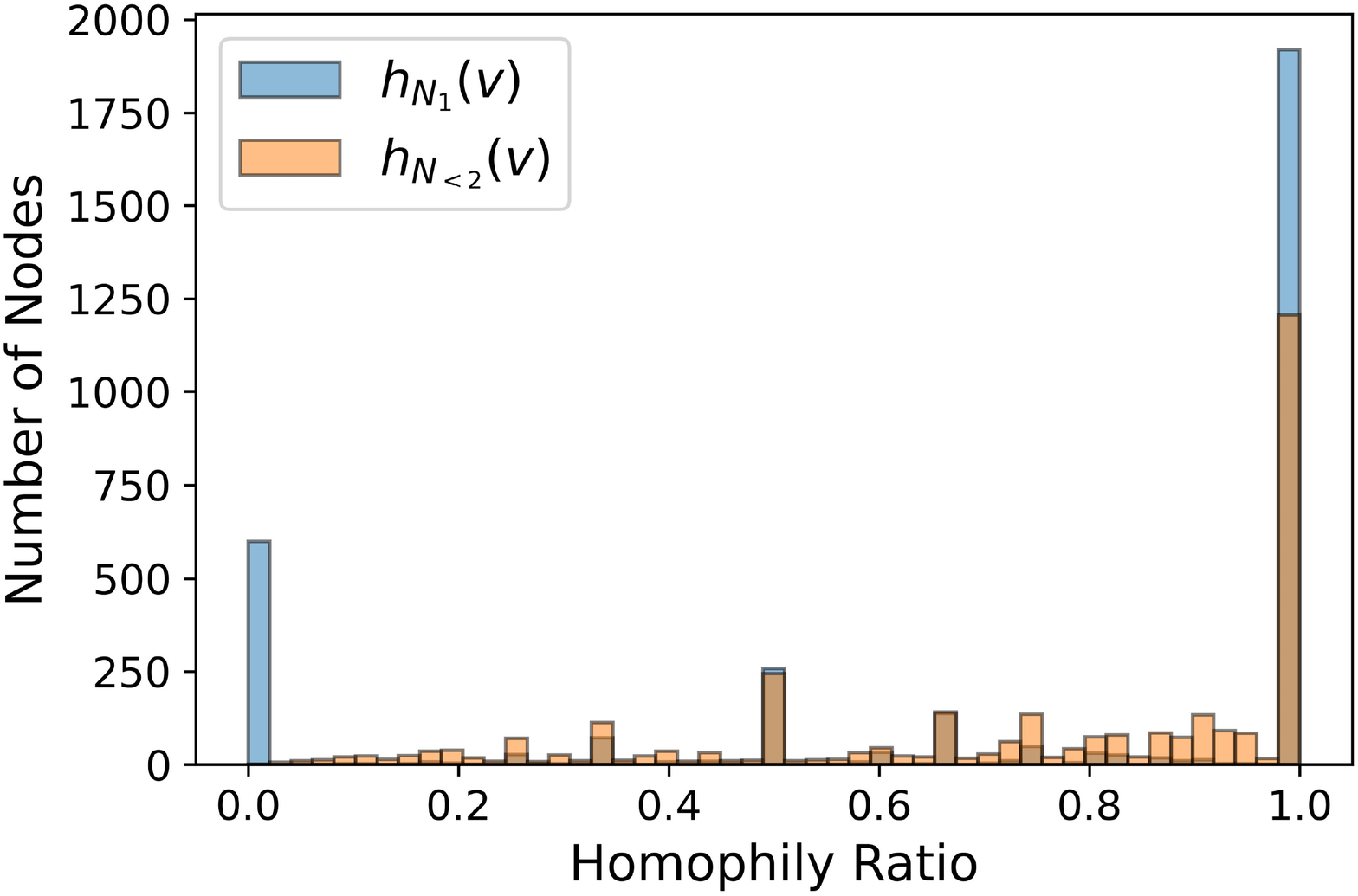}
   }
   \hspace{-2mm}
   \subfigure[Texas]{
   \includegraphics[width=35mm]{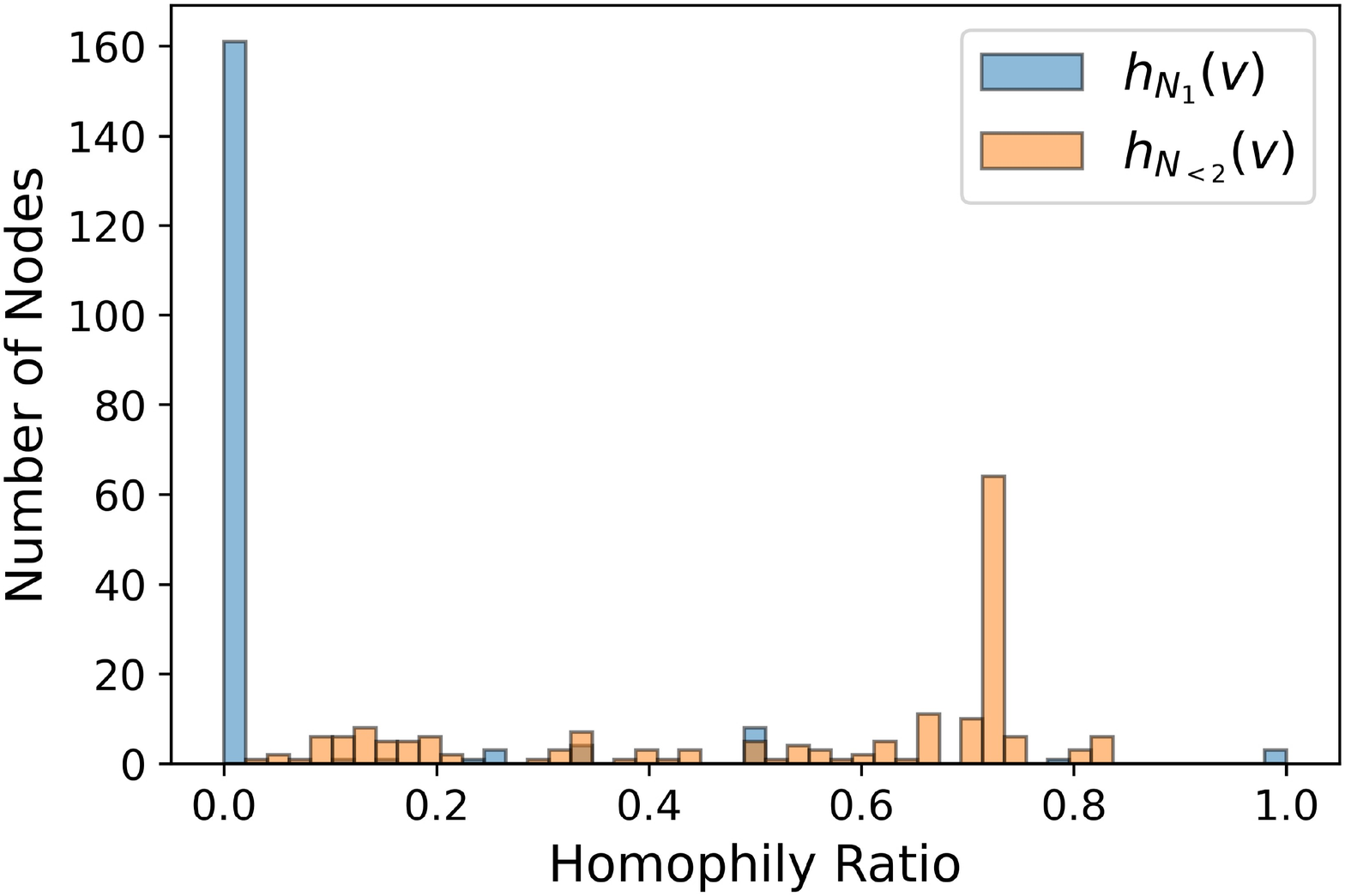}
   }
   \hspace{-3mm}
   \subfigure[Actor]{
   \includegraphics[width=35mm]{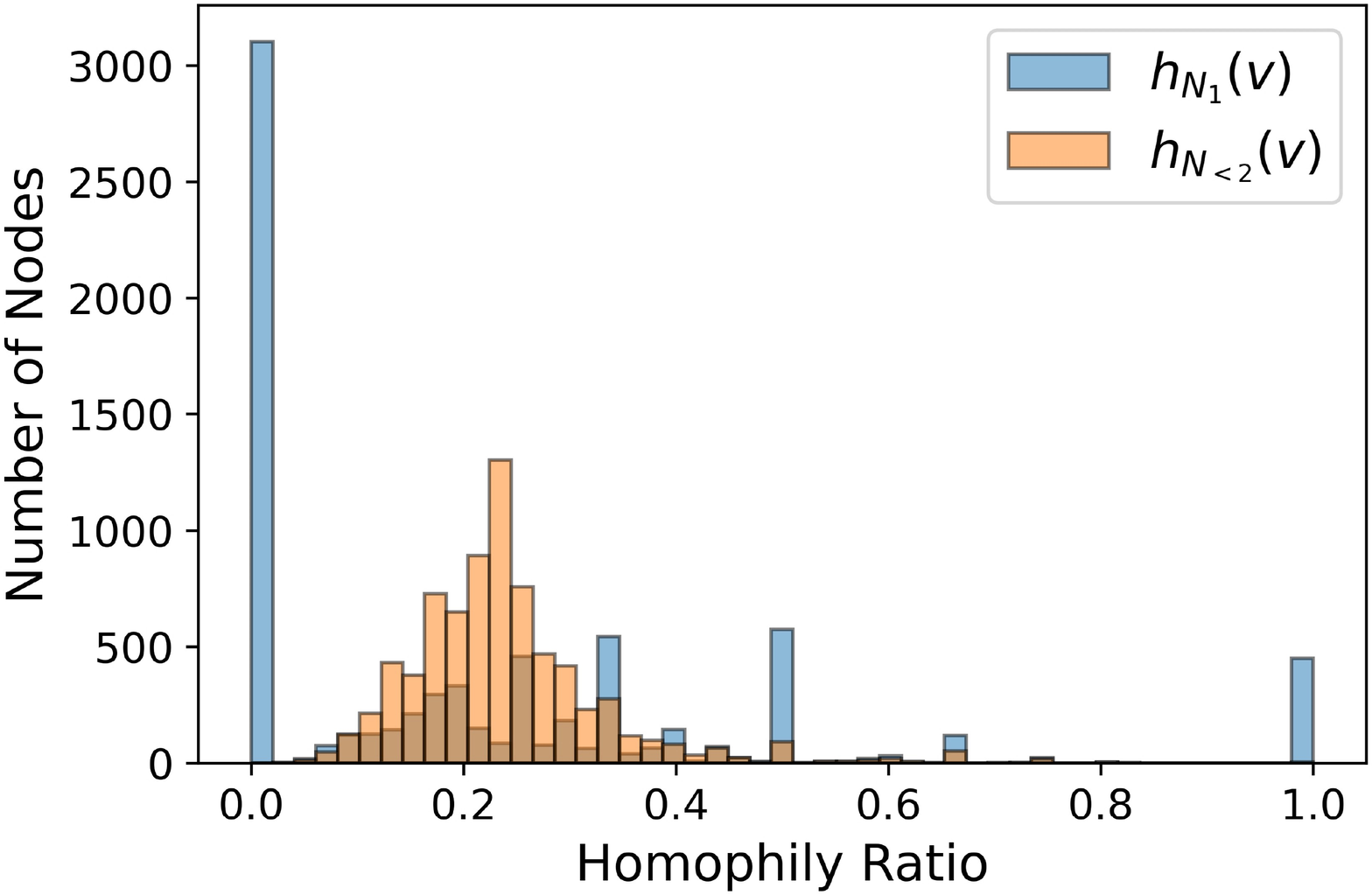}
   }
   \vspace{-3mm}
    \caption{The statistical histogram of $h_{N_1}(v)$ and $h_{N_{<2}}(v)$ of four real-world graphs, where Cora and Citeseer are known as the graphs with strong homophily, Cornell and Actor are known as the graphs with strong heterophily.}
    \vspace{-6mm}
    \label{fig::homo}
\end{figure*}
Furthermore, due to the unpredictability of the correlation between downstream tasks and graph structure, a follow-on and natural question is: (\textbf{Q1}) \textit{whether the homophily property is consistent across different local subgraphs in the whole graph?} Intuitively, the assumption that there are always diverse subgraph patterns among different regions might be more realistic~\cite{breaking}. 
A concrete example is that in an air traffic network, the establishment of air routes is more for commercial reasons and has little to do with the activities of airports~\cite{struc2vec}. In this case, it is not so easy to know whether the local homophily property of different subgraphs in the whole graph is consistent. 
In addition, as shown in Fig.~\ref{fig::homo}(a), a strong homophilic graph will also have a small amount of random or heterophilic local patterns, i.e., the neighbor label distribution of some nodes may also be random or heterophilic.
It poses a deeper challenge for GNN to take into account the possible existence of mixed patterns in graphs.

As far as the existing GNNs are concerned, most of them usually adopt message-passing architecture in the spatial domain to aggregate the node feature from neighbors over the given topology structure~\cite{GAT, Graphsage,LGCL,SGC} under the homophily assumption. 
Such near-neighbor aggregation that all neighbor nodes are considered to contribute positively to the target node has been shown to be remarkably applicable to homophilic graphs~\cite{analyzing}. Nevertheless, it will inevitably lack feasibility for non-homophilic graphs according to the experiment performance~\cite{GeomGCN,H2GCN}, and furthermore, it is also certainly unsuitable for graphs with mixed local patterns. 
Although some works have attempted to give interpretations from the spectral perspective that the essence of the above methods can be seen as a low-pass filter, it is quite contrary to the non-low-pass properties of non-homophilic graphs~\cite{SGC,analyzing}. However, considering that most of the GNNs are designed on the basis of message-passing architecture in the spatial domain, another interesting and worthwhile question to explore is: (\textbf{Q2}) \textit{how to understand why the near-neighbor aggregation performs poorly on the non-hospitability graph under message-passing architecture}.


In order to break the homophily limitation of previous GNNs, several recent studies have made some exploratory efforts beyond the near-neighbor aggregation to solve the generic GNN modeling problem for graphs with various homophily  properties and mixing local patterns from spatial and spectral perspectives. As shown in Fig.~\ref{fig::methods}, the spatial-based methods mainly focus on two aspects, various aggregation schemes~\cite{H2GCN, CPGNN, mixhop, GCNII} and relation re-estimation (also known as graph rewiring)~\cite{non_local, breaking, UGCN, GeomGCN, BMGCN}, to enhance the mastery of non-homophily graphs. Even though these methods show some progress to a certain extent, some additional calculations are always inevitably introduced, especially for the methods based on relation re-estimation. It is indispensable for them to define measures of node similarity to perform potential neighborhood discovery, whereas the design of similarity measure and the high complexity that comes with it making them less concise and flexible~\cite{survey}. Besides, the spectral-based methods, which have better interpretability compared to spatial-based methods, can also be broadly identified into two categories: predefined spectral filtering combination~\cite{acm, Interpreting_and_unifying} and adaptive spectral filter learning~\cite{GPRGNN, BernNet}. The starting point of the former is easy to understand, i.e., to deal with the heterogeneous map problem by expanding the filter set. However, the choice of filter sets obviously limits the generalizability of such methods. The latter is a straightforward way to learn a suitable global filter by higher-order polynomial approximation, which nicely extends the applicability of GNNs for graphs with different homophily properties. In addition, the global filter modeling of existing methods is incapable of capturing the changes in local structural patterns, which may still be suboptimal for graphs with a mixture of more complex local structural patterns.

 In the face of analyzing complex real-world graphs with different homophily properties, the latent mixed local structural patterns in graphs should not be neglected. Therefore, the two  questions, i.e., (\textbf{Q1}) and (\textbf{Q2}) as motioned above, should be well considered on the way to implementing a more generic GNN. For this purpose, we attempt to get deeper insights into them from two points, respectively, \textbf{(A1): Randomness of local patterns}, and \textbf{(A2): Aggregability  of near-neighbors}. In fact, the randomness of local patterns is described by the extent to which the local patterns in the graph are consistent, reflecting the complexity of modeling the graph. Meanwhile, the aggregatability of near-neighbors manifests the ability of GNNs based on neighbor aggregation to model a complex graph. It specifies why the existing near-neighborhood aggregation mechanisms fail to work for the non-homophily graphs from a spatial perspective. Considering the excellent theoretical interpretability and computational efficiency of spectral-based methods compared to spatial-based ones, we endeavor to explore the local adaptive spectral filter learning for the first time, as we know, to tackle the mixed local patterns in the graph. It breaks through the dilemma of globally shared filtering of the original adaptive spectral filter learning methods, while the computational superiority can be well maintained through the proposed reparameterization strategy. 

In summary, the main contributions of this paper can be highlighted as follows:
\begin{itemize}
\item[-]{In order to gain thorough insights into the higher-order mixing patterns in the real graphs and the adaptability of GNNs to them, both empirical and theoretical analyses on \textit{randomness of the graph} and \textit{aggregability of near-neighbors} are presented.}
\item[-]{Inspired by the generalized translation operator, we propose a node-oriented spectral filtering GNN using the polynomial-parameterized spectral convolution, namely NFGNN. It takes fully into account the specific effect of the node where the filter is positioned.}
\item[-]{To alleviate the heavy burden on learning node-oriented local filtering coefficients, we present an elegant low-rank approximation-based reparameterization to decompose the filter weight matrix, which not only simplifies the parameter complexity but also brings a trade-off between global and local filtering.}
\item[-]{ An extensive group of experiments on various real-world datasets up to 14 real-world graphs, including 6 classical homophilic graphs and 8 non-homophilic graphs, verifies that the proposed NFGNN achieves more favorable performance. }
\end{itemize}

The rest of this paper is organized as follows. The notations and terminology are introduced in Sect.~\ref{sect::Preliminaries}. In Sect.~\ref{sect::motivation}, we investigate the local mixing patterns in the graph. Sect.~\ref{sect::method} presents the methodology of the proposed NFGNN. The experiments are shown in Section~\ref{sect::experiment}. We also discuss the connection of some existing GNNs with NFGNN and the scalability of NFGNN in Sect.~\ref{sect::GNNs_connection} and  Sect.~\ref{sect::scalability}, respectively. Finally, we give the conclusion in Sect.~\ref{sect::conclusion}.

\section{Preliminaries}
\label{sect::Preliminaries}
\subsection{Notations}
An undirected graph is denoted as $\mathcal{G} = (\mathcal{V}, \mathcal{E})$, where $\mathcal{V}=\{v_i\}_{i=1}^{|\mathcal{V}|}$ denotes the set of nodes with $|\mathcal{V}|=n$, and $\mathcal{E}$ is the set of edges among nodes. 
The topology structure of graph $\mathcal{G}$ could be described by the adjacency matrix $\mathbf{A} \in \mathbb{R}^{n \times n}$ with $\mathbf{A}_{i,j}=1$ if $(i,j)\in \mathcal{E}$ or 0 otherwise. $\mathbf{D}$ is the diagonal degree matrix with its $i$-th diagonal entry $\mathbf{D}_{ii}=\sum_{j}A_{ij}$. Besides, we use $\mathbf{L}=\mathbf{I}-\mathbf{D}^{-1/2}\mathbf{A}\mathbf{D}^{-1/2}$ to denote the symmetric normalized Laplacian matrix of $\mathcal{G}$, and $\mathbf{I}$ is the identity matrix. 

For each node $v \in \mathcal{V}$, we denote its neighborhood by using $N(v)$, and further, the $i$-hop neighbors $N_i(v)$ and the neighbors within $i$-hops $N_{<i}(v)$ of node $v$ by $N_i(v) = \{m: m \in \mathcal{V} \wedge d_{\mathcal{G}}(v,m) = i\}$ and $N_{<i}(v) = \{m: m \in \mathcal{V} \wedge d_{\mathcal{G}}(v,m) \leq i\}$, respectively, where $d_{\mathcal{G}}(i,j)$ is the shortest path distance between two nodes $i$ and $j$ on graph $\mathcal{G}$.
Besides, let $\mathbf{x}=[x_1, \cdots, x_i, \cdots, x_n]^{\top} \in \mathbb{R}^n$ denote the $n$-dimensional signal defined on the given graph $\mathcal{G}$, where $x_i$ denotes the signal response (feature) at the $i$-th node $v_i$. Generally, when each node receives $f$ channels of signals, we will have a feature matrix $\mathbf{X}=[\mathbf{X}_1,\cdots, \mathbf{X}_i,\cdots, \mathbf{X}_n ]^{\top} \in \mathbb{R}^{n \times f}$ with each column of it  being  a graph signal $\mathbf{x}$ and its $i$-th row $\mathbf{X}_i\in \mathbb{R}^{f}$ representing $f$- dimensional feature vector associated with node $v_i$\footnote{Unless otherwise stated, only $\mathbf{x}\in \mathbb{R}^n$ is considered  as the input of GNNs for convenience of presentation, the following discussions of this work can also be  easily extendable to $\mathbf{X}\in \mathbb{R}^{n\times f}$  with $f$ channels of signals (i.e., $f$-dimensional features). }.

Furthermore, for the node classification task, each node $v \in \mathcal{V}$ has a class label $y_v \in \mathcal{Y}=\{1,\cdots,C\}$, where $\mathcal{Y}$ is the set of class labels with $|\mathcal{Y}|=C$, and $C$ is the number of classes. In addition, we use $\mathbf{y}_v\in\mathbb{R}^C$ to denote the one-hot vector corresponding to $y_v$.
\subsection{Graph Fourier Transform}
According to the graph signal processing theory, the graph Laplacian provides an effective way of spectral analysis on graphs. Given the Laplacian matrix $\mathbf{L}$, it can be eigendecomposed into $\mathbf{U}\Lambda\mathbf{U}^{\top}$, 
where $\mathbf{U}=[\mathbf{u}_1, \cdots, \mathbf{u}_l, \cdots, \mathbf{u}_{n}] \in \mathbb{R}^{n \times n}$ is the graph Fourier basis formed by $n$ orthonormal eigenvectors $\{\mathbf{u}_l\}_{l=1}^{n}$, 
and $\Lambda = \mathrm{diag}[\lambda_1, \cdots, \lambda_l, \cdots, \lambda_{n}]\in \mathbb{R}^{n \times n}$ is the diagonal matrix of the ordered eigenvalues $\{ \lambda_l\}_{l=1}^{n}$, $\lambda_l \in [0,2]$. Notice that $\{ \lambda_l\}_{l=1}^{n}$ is also identified as the frequencies of the graph. Thus, the graph Fourier transform of the signal $\mathbf{x}$ is defined as $\hat{\mathbf{x}}=\mathbf{U}^{\top}\mathbf{x}$, and $\hat{x}(\lambda_l) = \mathbf{u}^{\top}_l\mathbf{x}$ indicates the response of $\mathbf{x}$ at the frequency $\lambda_l$. The inverse graph Fourier transform is defined as $\mathbf{x}=\mathbf{U}\hat{\mathbf{x}}$~\cite{shuman2013emerging}. Thus, on the basis of the graph Fourier transform, the signal $\mathbf{x}$  filtered by $\hat{\mathbf{g}}$ can be given as follows:
\begin{equation}
\label{eq::graph_conv}
    \mathbf{z} = \sum_{l=1}^{n}\hat{g}(\lambda_l)\mathbf{u}_{l}\mathbf{u}_{l}^{\top}\mathbf{x}
    = \mathbf{U}\hat{\mathbf{g}} \mathbf{U}^{\top}\mathbf{x}
\end{equation}
where $\hat{\mathbf{g}}=\hat{g}(\Lambda) = \mathrm{diag}[\hat{g}(\lambda_1), \cdots, \hat{g}(\lambda_l), \cdots, \hat{g}(\lambda_{n})]$ is the spectral filter and we have $\mathbf{g}=\mathbf{U}\hat{\mathbf{g}}$. Since the spectral filtering is equivalent to convolution in the spatial domain~\cite{ortega2018graph}, Eq.~(\ref{eq::graph_conv}) could also be defined as the spectral graph convolution $\mathbf{z} = \mathbf{x} \ast_{\mathcal{G}} \mathbf{g}$, where $\ast_{\mathcal{G}}$ denotes the graph convolution operator.
\subsection{Polynomial-based spectral filtering}
Different from the spatial-based GNNs, spectral-based GNNs aim to learn a specific spectral filter for a given graph structure and node labels, thus preserving the appropriate frequency components for the downstream  tasks.
Based on Eq.~(\ref{eq::graph_conv}), early spectral GNNs~\cite{bruna2013spectral,henaff2015deep} directly eigendecompose the normalized Laplacian matrix $\mathbf{L}$ to obtain the Fourier basis $\mathbf{U}$ and treat the $\hat{\mathbf{g}}$ as the trainable parameters. However, the expensive eigenvalue decomposition restricts the availability of these methods greatly. To circumvent the eigendecomposition, $K$-order polynomial approximation $\hat{P}_K(\cdot)$ is adopted to parameterize the spectral filter:
\begin{equation}
    \hat{g}(\lambda_l) \approx \hat{P}_K(\lambda_l) = \sum^{K}_{k=0}\gamma_k
    \hat{p}_k(\lambda_l)
    \label{eq::Poly_Appro}
\end{equation}
where $\hat{p}_k(\cdot)$ is the $k$-th polynomial base and $\gamma_k$ denotes the learnable fitting coefficient corresponding to $\hat{p}_k(\cdot)$. By plugging Eq.~(\ref{eq::Poly_Appro}) into Eq.~(\ref{eq::graph_conv}), the spectral filtering can  be rewritten as:
\begin{equation}
    \mathbf{U}\hat{\mathbf{g}} \mathbf{U}^{\top}\mathbf{x} 
    \approx \mathbf{U}\left(\sum^{K}_{k=0}\gamma_k \hat{p}_k(\Lambda)\right )\mathbf{U}^{\top}\mathbf{x}
    =\left(\sum^{K}_{k=0}\gamma_k \hat{p}_k(\mathbf{L})\right )\mathbf{x}
\end{equation}
Due to the high efficiency, various polynomial bases are used for spectral filter parameterization, such as Chebyshev basis~\cite{chebyNet}, Bernstein basis~\cite{BernNet}, and Jacobi basis~\cite{Jacobi}. 

Except for the reduced complexity, another advantage of polynomial-parameterized filter learning is the localization property. When the filter $\hat{\mathbf{g}}$ centers at node $v_i$,  the value at node $v_j$ after filtering by $\hat{\mathbf{g}}$ is equal to $\sum^{K}_{k=0}\gamma_k \hat{p}_k(\mathbf{L})_{i,j}$. Since the bases of $\hat{p}_K(\mathbf{L})$ are $\left\{\mathbf{L}^k\right\}_{k=1}^K$, and $(\mathbf{L}^K)_{i,j}$ will be 0 if $d_{\mathcal{G}}(i,j) > K$~\cite{hammond2011wavelets}, the above facts show that the $K$-order polynomial spectral filter is exactly localized in $N_{<K}(i)$.

\section{Mixed Local Structural Patterns}
\label{sect::motivation}
The label consistency of nodes and their neighbors is the key factor in measuring the homophily of a graph. Beyond the global homophily, in this section, we also analyze and discuss the randomness of local patterns and the aggregability of near-neighbors from the aspect of the label distribution of neighbors.

\subsection{Randomness of the Graph} 
\begin{figure}[t]
	\centering
	\includegraphics[width=3.3in]{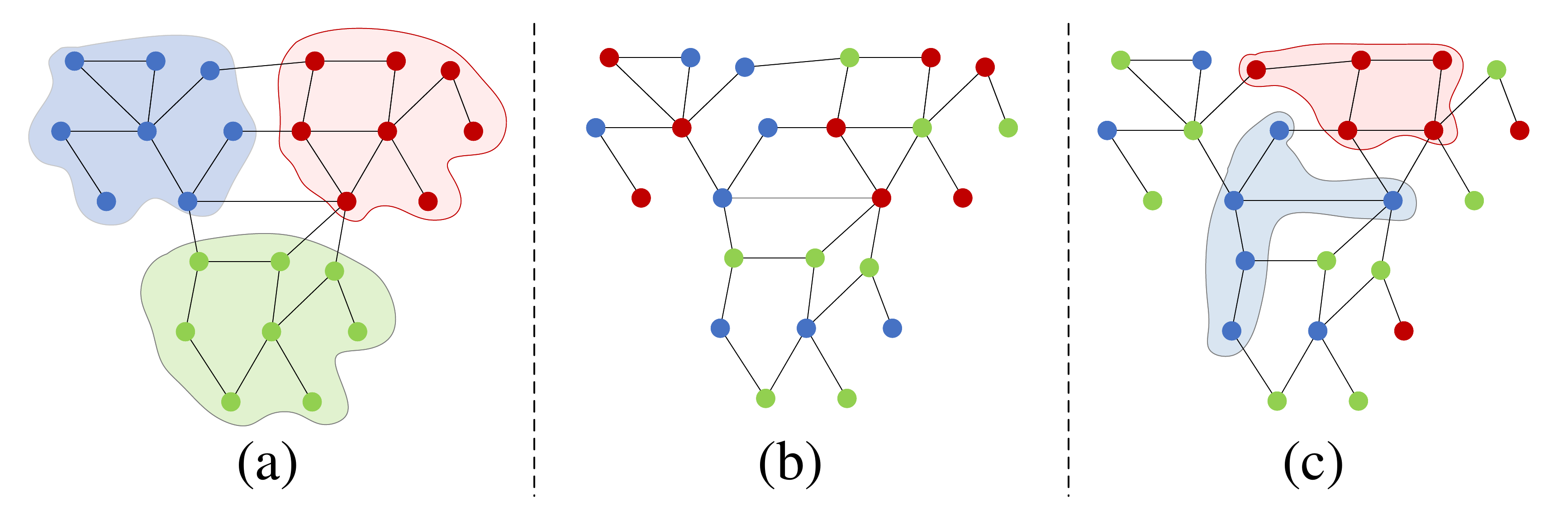}
	\vspace{-3mm}
	\caption{Three graphs with different homophily properties (a), (b), and (c). (a) is a strong homophilic graph with $\mathcal{H}_{\mathcal{G}} \approx 0.81$, (b) is a strong heterophilic graph with $\mathcal{H}_{\mathcal{G}} \approx 0.10$, and  (c) is a graph with mixing local patterns, whose $\mathcal{H}_{\mathcal{G}} \approx 0.26$.}
	\vspace{-5mm}
	\label{example}
\end{figure}
To measure the homophily of a graph, both edge homophily ratio~\cite{H2GCN} and  node homophily ratio~\cite{GeomGCN} are two widely used metrics. \textcolor{black}{In addition, ~\cite{lim1} recently proposed a more comprehensive homophily metric that mitigates homogeneity bias from class imbalance. It is less sensitive to the number of classes and size of each class than the edge homophily ratio and node homophily ratio.}
Since we aim to analyze the local patterns of the graph via the label consistency of the node neighborhoods, the node homophily ratio is adopted in this work. It should be noticed that the edge homophily ratio and the node homophily ratio  have similar qualitative behavior~\cite{benchmark}. 

In particular, the node homophily ratio $\mathcal{H}_{\mathcal{G}}$ of the graph $\mathcal{G}$ is defined as the average of the homophilic 1-hop neighbor ratio of each node $v$ in $\mathcal{G}$ and given by:
\begin{equation}
\label{H_G}
    \mathcal{H}_{\mathcal{G}} = \frac{1}{|\mathcal{V}|}\sum_{v\in\mathcal{V}}h_{N_1}(v) =  \frac{1}{|\mathcal{V}|}\sum_{v\in\mathcal{V}}\frac{|\{m\in N_1(v):y_m=y_v\}|}{|N_1(v)|}
\end{equation}
where $h_{N_1}(v)$ denotes the homophilic $1$-hop neighbor ratio of node $v$. Essentially, $H_{\mathcal{G}}$ provides an overall evaluation criterion for the homophily of a graph. However, it is not sufficient for a comprehensive assessment of a graph. As shown in Fig.~\ref{example}, we can see that the overall homophily is hard to reflect the local patterns of the graph. Therefore, we should also cast more insight into the variation of local structural patterns. 

Particularly, since $h_{N_1}(v)$ can be used to measure the homophily of a neighbor subgraph of a single node, we first give a node-level visualization of the statistical histograms of $h_{N_1}(v)$ and $h_{N_{<2}}(v)$. As shown in Fig.~\ref{fig::homo}, even in Cora and Citeseer networks which are usually considered homophilic graphs, there still exist a small number of completely 1-hop heterophilic subgraphs. Similarly, there are some subgraphs with a high homophily ratio in Cornell and Actor networks. Obviously, these observations mean that a graph with a complex topological structure is definitely a mixture of homophilic and heterophilic local subgraphs. Furthermore, for two heterophily graphs, we can find that the statistical histogram of $h_{N_{<2}}(v)$ is much different from that of $h_{N_{1}}(v)$, 
demonstrating that each node's associated local subgraph patterns varied generally with the change of neighborhood range. 

It should be noted that $h_{N_i}(v)$ only simply conveys the consistency of the neighborhood labels, but ignores the randomness of the neighborhood labels, which is also important for local pattern analysis. Thereby, \textcolor{revise}{inspired by the well-known Shannon entropy in information theory~\cite{shannon1948mathematical}}, we propose a label entropy $S_{N_i}(v)$ to measure the neighbor label distribution of node $v$, which is defined as :
\begin{equation}
\label{eq::entropy}
S_{N_i}(v) = -\sum_{y \in \mathcal{Y}}(\frac{|N^{(y)}_i(v)|}{|N_i(v)|}+\varepsilon)\mathrm{log}(\frac{|N^{(y)}_i(v)|}{|N_i(v)|}+\varepsilon)
\end{equation}
where $N^{(y)}_i(v) = \{m: m\in N_i(v)\wedge y_m = y \}$ and $\varepsilon=$1e-10 is a constant to avoid overflow. \textcolor{revise}{Similar to Shannon entropy, which measures the uncertainty or randomness of the possible outcomes of a variable, the label entropy serves as a node-level indicator that scalarizes the neighborhood label distribution of the given node and indicates the randomness of the subgraph centered at the node. Clearly, the label entropy tends to be maximal when the label distribution of neighbor nodes is uniform. On the contrary, if the neighborhood labels of a given node all belong to the same class, the label entropy will be minimal.}

As shown in Fig.~\ref{fig::entropy}, most nodes in the homophily graphs have low $S_{N_1}(v)$ and most nodes in the heterophily graphs have high $S_{N_1}(v)$. Besides, for all four graphs, the statistical histogram of $S_{N_{<2}}(v)$ is shifted to the right overall compared to $S_{N_{1}}(v)$. These observations suggest that the neighbor label distribution of each node tends to be uniform as the neighborhood range increases. More important, from Fig.~\ref{fig::entropy}(c) and (d), some explicit clustering phenomena can be easily found, indicating several types of important local patterns that may be existed in the graphs. 
Based on the above analysis of the variation of local structural patterns, we are prompted to consider that conducting adaptive modeling for the graph nodes with different degrees of homophily is a necessity, and further, we should also improve the effectiveness of GNNs for the nodes with various local structural patterns.

\subsection{Aggregability of Near-neighbors}
To facilitate the discussion of the aggregability of near-neighbors, we first give a definition of the neighborhood, i.e., \textit{homophily-preferred}:

\begin{definition}(homophily-preferred)
\label{def:homophily-preferred}
For a node $v$ with label $y_v$, $N_t(v)$ is expected to be heterophily-preferred if 
$\frac{P(y_m=y_v|y_v)}{P(y_m\neq y_v|y_v)} \le 1$, $\forall m \in N(v)$. Otherwise, $N_t(v)$ is expected to be homophily-preferred.
\end{definition}
Intuitively, the near-neighbor aggregation is effective when the near-neighbor subgraph is completely homophilic, while it may not capture adequate homophilic information when the neighborhood is expectedly heterophily-preferred. According to Definition~\ref{def:homophily-preferred}, it can be inferred that the aggregation of expectedly homophily-preferred neighborhoods is also beneficial to the node representation. Classical GNNs~\cite{GCN,GAT,Graphsage} commonly suffer from the over-smoothing problem~\cite{DAGNN} and are thus limited to shallow networks, which means that each node just aggregates the information about its neighbors within 2 or 3-hops. Hence, whether the near-neighborhood is homophily-preferred or heterophily-preferred will be of great importance to them. Moreover, Fig.\ref{fig::entropy}(c) and (d) also show that the label randomness of near-neighborhood in heterophily graphs is relatively high. Combined with the definition~\ref{def:homophily-preferred}, we can conclude that the near-neighbor-based aggregation is not the optimal solution for heterophily graphs.
\begin{figure*}[t]
    \centering
   \subfigure[Cora]{
   \includegraphics[width=35mm]{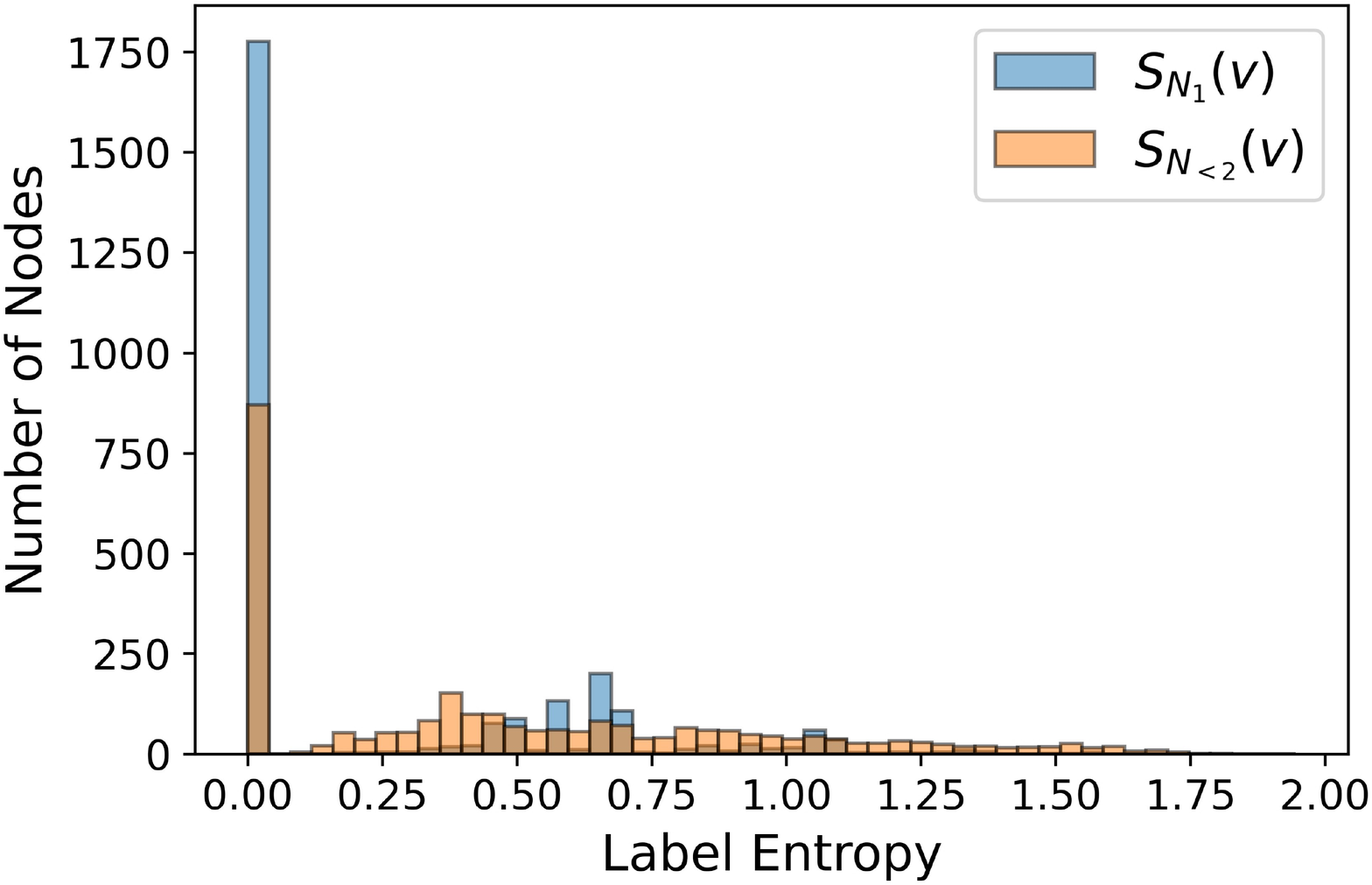}
   }
   \hspace{-2mm}
   \subfigure[Citeseer]{
   \includegraphics[width=35mm]{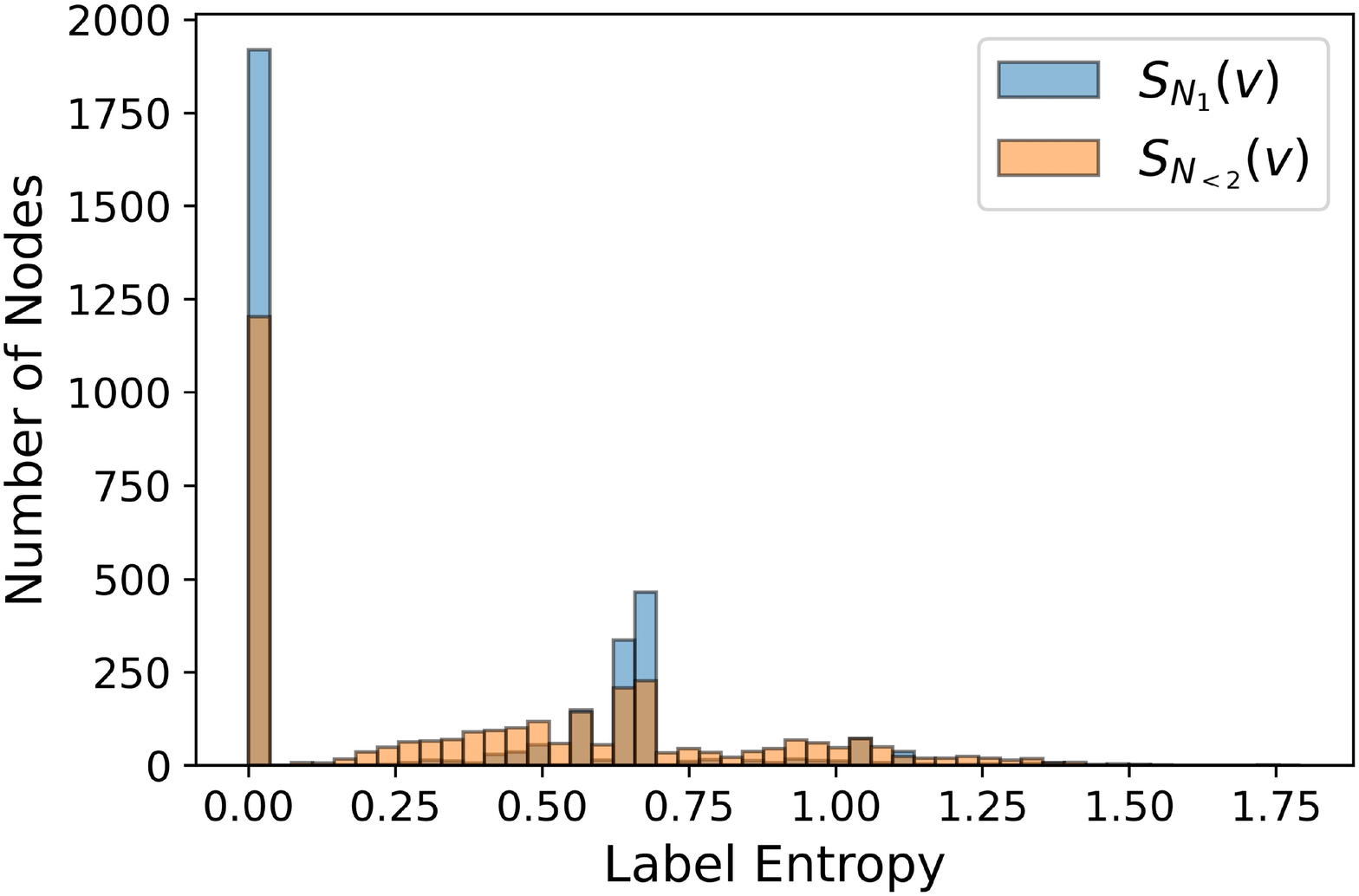}
   }
   \hspace{-2mm}
   \subfigure[Texas]{
   \includegraphics[width=35mm]{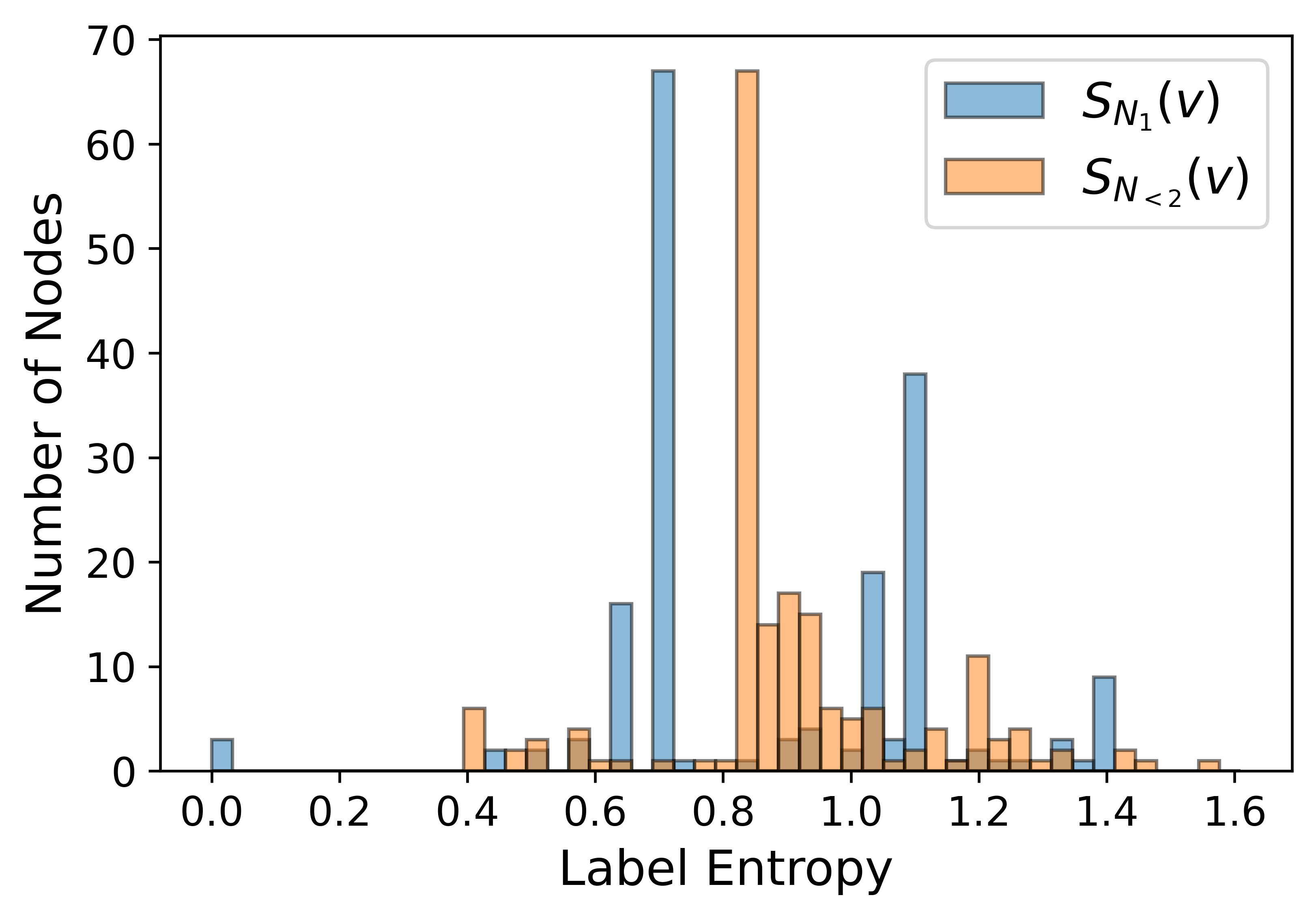}
   }
   \hspace{-2mm}
   \subfigure[Actor]{
   \includegraphics[width=35mm]{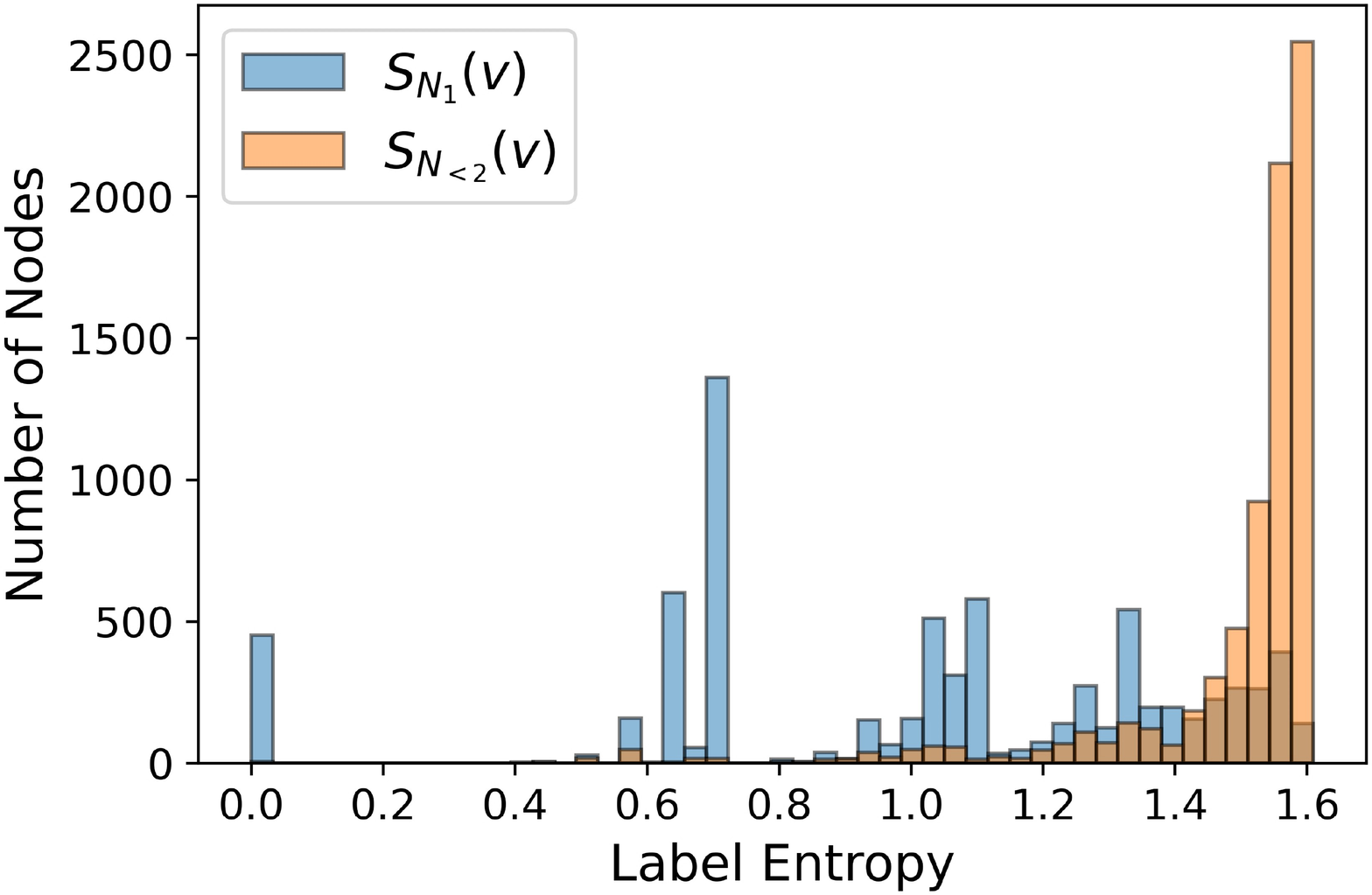}
   }
   \vspace{-3mm}
    \caption{The statistical histogram of $S_{N_1}(v)$ and $S_{N_{<2}}(v)$ of four real-world graphs. $S_{N_1}(v)$ and $S_{N_{<2}}(v)$ of homophilic graphs are generally low, indicating that the consistency of the neighborhood label distribution is high. While, $S_{N_1}(v)$ of heterophilic graphs show a multi-cluster phenomenon, which suggests that there may exist some potential local patterns. 
    }
    \vspace{-5mm}
    \label{fig::entropy}
\end{figure*}

\textcolor{R3}{Due to the fact that the homophily property of a graph is predominantly influenced by structure and labels instead of features, 
we theoretically explore the preference of the 2-hop neighborhood $N_2(v)$ by considering the structure with node labels for multi-class node classification, and have the following proposition:}   
\begin{proposition}
\label{thm::homophily-preferred}
For each node $v$ in a graph $\mathcal{G}$, let's assume the class labels of its t-hop neighbors $\{y_m: m \in N_t(v)\}$ are conditionally independent when given $y_v$, and $P(y_m=y_v|y_v) = \alpha$, $P(y_m=y|y_v) = \frac{1-\alpha}{|\mathcal{Y}|-1}, \forall y \neq y_v$. Then, the $\hat{t}$-hop neighborhood $N_{\hat{t}}(v)$ of a node $v$ will always be expectedly heterophily-preferred if $\alpha \le \frac{2}{|\mathcal{Y}|}$, where $\hat{t}=2t$ denotes even-order hops.
\end{proposition}

\begin{proof}[\textbf{Proof}]
For node $v$, following the assumption that its neighbors' class labels $\{y_m: m \in N_t(v)\}$ are conditionally independent when $y_v$ is given, and $P(y_m=y_v|y_v) = \alpha$, $P(y_m=y|y_v) = \frac{1-\alpha}{|\mathcal{Y}|-1}, \forall y \neq y_v$, we have $P(y_m=i|y_v=i) = \alpha$ and $P(y_m=j|y_v=i) = \frac{1-\alpha}{|\mathcal{Y}|-1}, \forall i, j \in \mathcal{Y}$ and $j \neq i$, where $m \in N_t(v)$. 
The $\hat{t}$-hop neighborhood $N_{\hat{t}}(v)$ of a node $v$ will be expectedly heterophily-preferred when the following inequality holds:
\begin{equation}
P(y_o=i|y_v=i) - P(y_o\neq i|y_v=i) \le 0, \; \forall o \in N_{\hat{t}}(v)
\label{inequality}
\end{equation}
Consider node $o \in N_{\hat{t}}(v)$, we have: 
\begin{equation}
P(y_o=k|y_v=i) =\sum_{j \in \mathcal{Y}} P(y_o=k|y_m=j)P(y_m=j|y_v=i)
\label{ki}
\end{equation}
Let $\tau = \frac{(1-\alpha)^2}{|\mathcal{Y}|-1} $ and $ (1-\alpha)^2 = (|\mathcal{Y}|-1)\tau $. From Eq.~(\ref{ki}), we have: 
\begin{equation}
P(y_o=i|y_v=i) = \alpha^2 + \tau
\label{ii}
\end{equation}
and for $j \in \mathcal{Y}$,
\begin{equation}
    \begin{split}
        \sum_{j \neq i}P(y_o=j|y_v=i) &= P(y_o\neq i|y_v=i) \\ &= 2\alpha(1-\alpha) + (|\mathcal{Y}|-2)\tau
    \end{split}
\label{ji}
\end{equation}
Hence, defining the difference between $P(y_o=i|y_v=i)$ and $P(y_o\neq i|y_v=i)$ as $\epsilon$, we have:
\begin{equation}
\begin{split}
\epsilon &= P(y_o=i|y_v=i) - P(y_o\neq i|y_v=i)  \\ & = \alpha^2 + \tau - 2\alpha(1-\alpha) - (|\mathcal{Y}|-2)\tau 
\end{split}
\label{ii-ji}
\end{equation}
Eq.~(\ref{ii-ji}) can be simplified as :
\begin{equation}
  \epsilon =  (2\alpha^2-1) + 2\tau
\label{temp1}
\end{equation}
Applying $\frac{(1-\alpha)^2}{|\mathcal{Y}|-1} = \tau$, we have
\begin{equation}
  \epsilon =  \frac{2\alpha(\alpha|\mathcal{Y}|-2) + (3-|\mathcal{Y}|)}{|\mathcal{Y}|-1}
\label{temp2}
\end{equation}
Notice that $|\mathcal{Y}| \ge 3$ for the multi-class node classification task, it can be easily obtained:
\begin{equation}
  \epsilon \le \frac{2\alpha(\alpha|\mathcal{Y}|-2)}{|\mathcal{Y}|-1} \le 0,\; \mathrm{if}\; \alpha \le \frac{2}{|\mathcal{Y}|}
\label{temp3}
\end{equation}
Clearly, $\alpha \le \frac{2}{|\mathcal{Y}|}$ is a sufficient condition for $P(y_o=i|y_v=i) - P(y_o\neq i|y_v=i) \le 0$. Thus, the 2t-hop neighborhood $N_{2t}(v)$ of a node $v$ will always be expectedly heterophily-preferred, i.e., $\frac{P(y_m=y_v|y_v)}{P(y_m\neq y_v|y_v)} \le 1$, if $\alpha \le \frac{2}{|\mathcal{Y}|}$. 

This completes the proof.
\label{proof::heter-perferred}
\end{proof}

According to Eq.~(\ref{temp1}) and Eq.~(\ref{temp2}), we also have remarks as follows:
\begin{remark} For a graph $\mathcal{G}$ with a class label set $\mathcal{Y}$, consider the class labels of the neighbors of node $v$ in $\mathcal{G}$, $\{y_m: m \in N(v)\}$, are conditionally independent given $y_v$, and $P(y_m=y_v|y_v) = h$, $P(y_m=y|y_v) = \frac{1-h}{|\mathcal{Y}|-1}, \forall y \neq y_v$, we have: 
\textbf{1)} the 2-hop neighborhood $N_2(v)$ of a node $v$ will always be expectedly homophily-preferred if $h \ge \sqrt{\frac{1}{2}}$.
\textbf{2)} the 2-hop neighborhood $N_2(v)$ of a node $v$ will always be expectedly homophily-preferred if $\mathcal{Y}$ is a binary class label set.
\end{remark}


When $t=1$, proposition~\ref{thm::homophily-preferred} shows that the aggregation of $N_2(v)$ is hard to be beneficial for GNNs when the neighborhood label distribution tends to be uniform and it holds for all even-hop neighborhoods. This finding is also consistent with our empirical analysis mentioned above.

Further, we discuss the trend of homophily probability $P(y_m=y_v|y_v)$ in neighborhoods of different hops.
\begin{proposition}
\label{thm::homophily-trend}
For each node $v$ in a graph $\mathcal{G}$, assume the class labels of its neighbors $\{y_m: m \in N_t(v)\}$ are conditionally independent when given $y_v$, and $P(y_m=y_v|y_v) = \alpha$, $P(y_m=y|y_v) = \frac{1-\alpha}{|\mathcal{Y}|-1}, \forall y \neq y_v$. Then, $P(y_o=y_v|y_v) > P(y_m=y_v|y_v), o \in N_{\hat{t}}(v)$ if $\alpha < \frac{1}{|\mathcal{Y}|}$, where $\hat{t}=2t$ denotes even-order hop.
\end{proposition}

\begin{proof}
Following the same assumption in proposition \ref{thm::homophily-preferred} and considering the node $o \in N_{\hat{t}}(v)$, we have $P(y_o=i|y_v=i) = \alpha^2 + \frac{(1-\alpha)^2}{|\mathcal{Y}|-1}$ as in Eq.~(\ref{ii}). Reuse $\epsilon$ to define as the difference between $P(y_o=i|y_v=i)$ and $P(y_m=i|y_v=i)$, we have:
\begin{equation}
    \begin{split}
        \epsilon &= P(y_o=i|y_v=i) - P(y_m=i|y_v=i)  \\ & = \alpha^2 + \frac{(1-\alpha)^2}{|\mathcal{Y}|-1} - \alpha
    \end{split}
    \label{eq::P(y_o=i|y_v=i)-P(y_m=i|y_v=i)}
\end{equation}
Clearly, Eq.~(\ref{eq::P(y_o=i|y_v=i)-P(y_m=i|y_v=i)}) can be simplified as:
\begin{equation}
    \epsilon = \frac{(\alpha-1)(\alpha|\mathcal{Y}|-1)}{|\mathcal{Y}|-1}
    \label{eq::P(y_o=i|y_v=i)-P(y_m=i|y_v=i)_2}
\end{equation}

Since $\alpha \in [0,1]$, it can be easily derived from Eq.~(\ref{eq::P(y_o=i|y_v=i)-P(y_m=i|y_v=i)_2}) that $\epsilon < 0$ if $\alpha > \frac{1}{|\mathcal{Y}|}$, and $\epsilon > 0$ if $\alpha < \frac{1}{|\mathcal{Y}|}$.

This completes the proof.
\end{proof}
Proposition~\ref{thm::homophily-trend} shows that $P(y_m=y_v|y_v)$ of higher-order-hop neighbors is higher than low-order neighbors in heterophily graphs. It suggests that higher-order neighborhood aggregation, instead of near-neighbor aggregation, is more beneficial for handling heterophilic graphs, which is consistent with the empirical conclusions in previous works~\cite{mixhop,H2GCN}.
It is also important to note that, the gains from high-order neighborhood aggregation may still be modest since it can be seen from Proposition~\ref{thm::homophily-trend} that $P(y_m=y_v|y_v)$ will also decrease if $\alpha > \frac{1}{|\mathcal{Y}|}$. 

Based on the above observations and justifications, as the real-world graph is often made up of a mixture of various local patterns, the near-neighbor aggregation mechanism does not handle the heterophilic local patterns well. To tackle these issues, various elaborately designed aggregation mechanisms in the spatial domain have been proposed. Different from these spatial-based methods, polynomial-based spectral filtering served as graph convolution aims to learn a specific spectral filter for a given graph structure and node labels, thus preserving the appropriate frequency components for the downstream tasks. It not only covers a large neighborhood, but also has a low computational complexity compared to spatial aggregation, making it an alternative development route for GNNs worth exploring.

\section{The Proposed NFGNN}
\label{sect::method}
Spectral graph convolution has been demonstrated that possesses a strong expressive power~\cite{analyzing} and can work well for both heterophily and homophily graphs.
Nevertheless, it is still not flexible enough. 
More precisely, what is applied to all nodes in these methods is a single filter $\hat{\mathbf{g}}$ using fixed frequency coefficients $\{\gamma_k\}_{k=0}^{K}$ trained on the entire graph, and the estimated filter makes no specific discrimination for each node when performing filtering.
Thus, even though the polynomial-parameterized spectral filter learning they employ is localized, it may still be inappropriate to effectively model the complex local structural patterns. Intuitively, instead of learning a globally shared filter $\hat{g}(\lambda_l)$, learning an appropriate node-specific filter $\hat{g}_{i}(\lambda_l)$ for node $i$ to fit the local pattern where it is located seems to be a better choice. In this section, we rethink the globally consistent spectral graph convolutions, and propose a localized spectral filter learning method to break the limitation.

\subsection{Translated Filter for Node-oriented Filtering}
Inspired by the generalized translation operator in GSP, we develop an adaptive localized spectral filtering on graph $\mathcal{G}$ using the polynomial-parameterized spectral convolution. It can take fully into account the specific effect of the node where the filter is positioned.

\begin{definition}(~\textbf{Generalized translation operator})~\cite{shuman2013emerging}
\label{def::generalized_T}
For any signal $\mathbf{g} \in \mathbb{R}^n$ defined on a given graph $\mathcal{G}$ and any $i \in \{0,1,\cdots,n-1\}$, we define a generalized translation operator $\mathbf{T}_i:\mathbb{R}^n \to \mathbb{R}^n$ via generalized convolution with a Kronecker delta function $\delta_i$ centered at the $i$-th node $v_i$:
\begin{equation}
\mathbf{T}_i(\mathbf{g}):=\sqrt{N}(\mathbf{g} \ast  \delta_i)
=\sqrt{N}\sum_{l=1}^{n} \mathbf{u}_l u_l^{\top}(i)\hat{g}(\lambda_l)
\end{equation}
where $u_l^{\top}(i)$ denotes the $i$-th element of $\mathbf{u}^{\top}_l$. 
\end{definition}
Definition~\ref{def::generalized_T} shows that a signal could be centered at a specific node through a kernelized operator acting on $\hat{\mathbf{g}}$~\cite{shuman2013emerging}. 
Therefore, for the purpose of adaptive local filtering, the filter signal $\mathbf{g}$ can be firstly centered at the target node $v_i$ through $\mathbf{T}_i$, and then performed the spectral convolution with $\mathbf{x}$ as: 
\begin{equation}
\label{eq::Y}
\mathbf{x} \ast_{\mathcal{G}} \mathbf{T}_i(\mathbf{g})
           = \sqrt{N} \sum_{l=1}^{n} \mathbf{u}_l \hat{x}(\lambda_l)u_l^{\top}(i)\hat{g}(\lambda_l)
\end{equation}
Let $\hat{g}_i(\lambda_l) = \sqrt{N} u_l^{\top}(i) \hat{g}(\lambda_l)$, Eq.~(\ref{eq::Y}) becomes:
\begin{equation}
\mathbf{x} \ast_{\mathcal{G}} \mathbf{T}_i(\mathbf{g}) = \sum_{l=1}^{n}\hat{g}_i(\lambda_l)\mathbf{u}_{l}\mathbf{u}_{l}^{\top}\mathbf{x}
= \mathbf{U}\hat{\mathbf{g}}_{i} \mathbf{U}^{\top}\mathbf{x}
\end{equation}
It can be noted that $\hat{\mathbf{g}}_{i}$ is a function related to $\mathbf{U}$ according to the definition of $\hat{g}_i(\lambda_l)$. Therefore, we can further derive the connection between $\hat{\mathbf{g}}_{i}$ and $\mathbf{x}$ by approximation of $\mathbf{U}$.

Recall the inverse graph Fourier transform, it can be further derived that $x_i = \mathbf{U}_{i:}\hat{\mathbf{x}}$, where $\mathbf{U}_{i:}$ indicates the $i$-th row of $\mathbf{U}$. Hence, $\mathbf{U}_{i:}$ could also approximated from $x_i$ according to $\mathbf{U}_{i:} \approx x_i\mathbf{q}$, where $\mathbf{q} = \mathrm{pinv}(\hat{\mathbf{x}}) \in \mathbb{R}^{n}$ is the pseudoinverse of $\hat{\mathbf{x}}$. 
Meanwhile, note that $u_l(i) = u_l^{\top}(i)$, the approximate form of $\hat{g}_i(\lambda_l)$ can be derived as
\begin{equation}
\label{eq::hat_g}
\hat{g}_i(\lambda_l) \approx \tilde{g}_i(\lambda) = \sqrt{N} (x_i q_{l}) \hat{g}(\lambda_l)
\end{equation}
where $q_{l}$ denotes the $l$-th element of $\mathbf{q}$. Therefore, $\hat{\mathbf{g}}_i$ can be approximated by a specific filter $\tilde{\mathbf{g}}_i$ corresponding to node $v_i$, which also should consider the impact from the feature $x_i$ associated with node $v_i$ in estimating the filter coefficients.

Without loss of generality, the $K$-order polynomial approximation $\hat{P}_K(\cdot)$ can be used to directly parameterize $\tilde{\mathbf{g}}_{i}$,
that is, $\tilde{\mathbf{g}}_{i} = \sum^{K}_{k=0}\eta_{i,k} \hat{p}_k(\Lambda)$, and $\eta_{i,k}$ denotes the coefficient of $\hat{p}_k(\Lambda)$ for the filter $\tilde{\mathbf{g}}_{i}$.
For each node, we focus only on the convolution result of the filter positioned at that node, i.e., $\mathbf{z}_i = \delta_i(\mathbf{U}\tilde{\mathbf{g}}_i \mathbf{U}^{\top}\mathbf{x})$, here we reuse $\delta_i$ as an indicator vector $\delta_i =[0,\cdots,\underset{i}{\underline{1}},\cdots,0]\in \mathbb{R}^{n}$, which denotes a row vector with only the $i$-th element being 1 and the remains being zeros. Thus, the node-oriented localized filtering can be formulated as:
\begin{equation}
\begin{split}
\label{eq::final}
\mathbf{z}_i = \delta_i(\mathbf{U}\tilde{\mathbf{g}}_i \mathbf{U}^{\top}\mathbf{x})&=
\delta_i\mathbf{U}\left(\sum^{K}_{k=0}\eta_{i,k} \hat{p}_k(\Lambda)\right )\mathbf{U}^{\top}\mathbf{x} \\ & =   \delta_i\sum^{K}_{k=0}\Psi_{ik}\hat{p}_k(\mathbf{L})\mathbf{x}
\end{split}
\end{equation}
where $\Psi=[\eta_{i,k}] \in \mathbb{R}^{n \times (K+1)}$ is the trainable coefficient matrix. In practice, It can still be applicable to the feature matrix $\mathbf{X}$ as $\mathbf{Z}_i = \delta_i\sum^{K}_{k=0}\Psi_{ik}\hat{p}_k(\mathbf{L})\mathbf{X}$. 

Similar to the above discussion of localization property, we have the following Proposition~\ref{prop::T_ig} to claim that the adaptively filtered signal $\mathbf{z}$ is also approximately positioned around the node $i$.
\begin{proposition}
\label{prop::T_ig}
Given a signal $\mathbf{x}$ defined on a graph $\mathcal{G}$ and a filter $\mathbf{T}_i(\mathbf{g})$ that translated to a given center node $v_i$ in $\mathcal{G}$, the filtered signal $\mathbf{z} = \mathbf{x} \ast  \mathbf{T}_i(\mathbf{g})$ is approximately localized around the node $i$.
\end{proposition}
\begin{proof}
To prove the proposition~\ref{prop::T_ig}, let's first introduce the following Lemma~\ref{lemma:localized}.
\begin{lemma}[\cite{shuman2016vertex}]
\label{lemma:localized}
Let $\hat{P}_K$ be the polynomial approximation with degree K to the spectrum of a graph signal $\varphi$, i.e., $\hat{\varphi}(\lambda_l)\approx\hat{P}_K(\lambda_l) = \sum_{k=0}^{K}\gamma_k\lambda_l^k$. If $d_{\mathcal{G}}(i,n) > K$, then $\mathbf{T}_i(\varphi)_n\approx \mathbf{T}_i (P_K)_n=0$, where $p_K$ denotes the signal corresponding to $\hat{p}_K$. 
\end{lemma}

According to the  Definition~\ref{def::generalized_T} and the properties of convolution, we notice that the filtered result $\mathbf{z}$ can be rewritten as:
    \begin{equation}
        \mathbf{z} = \mathbf{x} \ast \sqrt{N}(\mathbf{g} \ast \delta_i) 
        = \sqrt{N}((\mathbf{x} \ast \mathbf{g}) \ast \delta_i)
        = \mathbf{T}_i(\mathbf{x} \ast \mathbf{g})
    \end{equation}
Furthermore, let $\varphi = \mathbf{x} \ast \mathbf{g}$ and $\hat{P}_K$ be the polynomial approximation with degree $K$ to $\hat{\varphi}$. From Lemma~\ref{lemma:localized}, $\mathbf{T}_i(\varphi)_n \approx \mathbf{T}_{i}(p_{K})_n=0$ will hold if $d_{\mathcal{G}}(i,n) > K$. Then, we have $\mathbf{z}_n = \mathbf{T}_i(\varphi)_n \approx 0$ if $d_{\mathcal{G}}(i,n) > K$. 

This completes the proof.
\end{proof}

\subsection{Filter Learning via Low-rank Reparameterization}
\label{sect::reparameterization}
For the node-oriented localized filtering in Eq.~(\ref{eq::final}), we could treat all the coefficient matrix $\Psi$ as learnable parameters or hyperparameters.
But there are several concerns about this scheme that are worth noting:

\noindent\textbf{Parameter complexity}. It is obvious that the scale of $\Psi$ is positively proportional to the number of nodes $\mathcal{V}$. With the increase of $n$, this will inevitably involve learning a large number of parameters. 

\noindent\textbf{Optimization}. 
We can see from Eq.~(\ref{eq::final}) that only the gradients from $\mathbf{z}_i$ can update the parameters of the localized filter corresponding to node $v_i$, i.e., the $i$-th row of $\Psi$, if we learn $\Psi$ directly. It means that $\Psi$ is hard to be efficiently optimized.

\noindent\textbf{Connection}. We have derived the localized filter associated with the corresponding node and graph signal $\mathbf{x}$ in Eq.~(\ref{eq::hat_g}). However, directly treating $\Psi$ as learnable parameters leaves no explicit use of this correlation.


Considering the above issues, we introduce the thought of low-rank matrix factorization and propose a separable reparameterization strategy to learn $\Psi$ indirectly instead of directly. It allows us to cleverly circumvent these problems. 
\textcolor{revise}{Specifically, the proposed low-rank approximation strategy is based on a sensible assumption that 
the filter parameters of nodes with consistent local neighborhood patterns should also be similar. Meanwhile, considering that the number of the types for those local heterophilic patterns is also limited, it means that the filter parameter matrix $\Psi$ could be supposed to be low-rank and its rank should be positively correlated with the number of types of local neighborhood patterns.
Therefore, it is feasible to assume that} $\Psi$ can be factorized into and approximated by two trainable parameter matrices $\Psi = \mathbf{H}\mathbf{\Gamma^{\top}}$,  
where $\mathbf{H} \in \mathbb{R}^{n \times d}$ and $\mathbf{\Gamma}\in \mathbb{R}^{(K+1) \times d}$.
As can be easily observed, $\Psi_{i:} = \mathbf{H}_{i:}\mathbf{\Gamma}^{\top}=\sum_{j=1}^d \mathbf{H}_{i,j} (\mathbf{\Gamma}_{:j})^{\top}$. 
It means each column of $\mathbf{\Gamma}$ can be viewed as $\{\mathbf{\gamma}_k\}_{k=0}^K$, which parameterizes a globally shared filter. Therefore, $\mathbf{\Gamma}$ is equivalent to a set of base filters, while $\mathbf{H}_{i:}$ is corresponding to the filter weights of node $v_i$. According to $\mathbf{H}_{i:}$, the filter dedicated to $v_i$ can be obtained by a weighted combination of the base filters in $\mathbf{\Gamma}$. Therefore, $\mathbf{H}$ is a node-dependent trainable matrix, which establishes a close connection between $\Psi_{i:}$ and  $\mathbf{X}$.

As $\mathbf{\Gamma}$ is seen as node-agnostic, it can be directly trained as general parameters, which is very similar to the learning of the polynomial coefficients in previous works~\cite{GPRGNN}. The rank $d$ then determines the number of base filters.
But for $\mathbf{H}$, since we treat it as node dependent, a simple yet effective nonlinear transformation $F(\cdot)$ is applied:
    \begin{equation}
    \label{eq::H}
        \mathbf{H}=F(\mathbf{X}; \mathbf{W}) := \sigma(\mathbf{X}\mathbf{W})
    \end{equation}
where $\mathbf{W}\in \mathbb{R}^{f \times d}$ and $\sigma(\cdot)$ are the learnable weight matrix and the activation function, respectively. Therefore, the low-rank reparameterization of $\Psi$ can be formulated as $\Psi=\sigma(\mathbf{X}\mathbf{W})\mathbf{\Gamma}^{\top}$.
In practice, we use $\mathtt{sigmoid}(\cdot)$ as $\sigma(\cdot)$. 
\textcolor{revise}{Through the low-rank reparameterization strategy, the filter learning achieved in the form of $\Psi = \mathbf{H}\mathbf{\Gamma}^{\top} = \sigma(\mathbf{X}\mathbf{W})\mathbf{\Gamma}^{\top}$ actually performs two matrix multiplications, where $\mathbf{W} \in \mathbf{R}^{f \times d}$, $\mathbf{\Gamma} \in \mathbb{R}^{(K+1) \times d}$, and $\mathbf{X} \in \mathbb{R}^{n \times f}$. Therefore, the theoretical computational complexity of the low-rank reparameterization is $\mathcal{O}(n \times d \times (f+K+1))$. }

The advantages of the low-rank approximation-based re-parameterization are obvious.
Firstly, the parameter complexity of $\Psi$ is reduced from $\mathcal{O}(n \times (K+1))$ to $\mathcal{O}((K+1)\times d+f\times d)$. Generally speaking, $d$ is much smaller than $n$, so the parameter complexity of $\Psi$ can be heavily reduced.
Besides, it provides a trade-off way to adjust the model capacity by changing $d$, which gives us the opportunity to avoid potential underfitting problems in ~\cite{GPRGNN}.
Secondly, instead of only $x_i$ will participate in the optimization of $\Psi_{i:}$, the reparameterization strategy elegantly solves the optimization issue by learning a transformation function $F(\cdot;\mathbf{W})$ to adaptively estimate the node-dependent matrix $\mathbf{H}$. 
\textcolor{revise}{In the inference phase, for the test set with feature matrix $\mathbf{X}^{test}$, the corresponding filter parameters $\Psi^{test}$ can be obtained as $\Psi^{test} = \mathbf{H}^{test}\mathbf{\Gamma}^{\top} = \sigma(\mathbf{X}^{test}\mathbf{W})\mathbf{\Gamma}^{\top}$ with the learned factor matrix $\mathbf{\Gamma}$ and the projection matrix $\mathbf{W}$.}
Finally, it creates a bridge between the node-oriented localized filtering and global-shared filtering. In addition, it is also convenient to clarify the connections and differences with previous methods~\cite{GPRGNN,BernNet}. To be specific, if we only use the node-agnostic matrix $\mathbf{\Gamma}$, then the node-oriented localized filtering will simply become a global-shared filtering that eventually degenerates to ~\cite{GPRGNN,BernNet}. 

\begin{table*}[h]
\centering
\caption{ Statistics of the used datasets.}
\vspace{-3mm}
\label{datasets}
\resizebox{6.5in}{!}{
\begin{tabular}{@{}l|ccccc|cccccccc@{}}
\toprule
& \multicolumn{5}{c|}{Homophilic graphs} & \multicolumn{8}{c}{Heterophilic graphs}       \\ \midrule
         & Cora  & Cite.  & PubMed  & Comp.   & Photo   & Cham.  & Squi.   & Actor  & Texas & Corn. & Penn94    & genius  & pokec\\ \midrule
Nodes    & 2,708 & 3,327  & 19,717  & 13,752  & 7,650      & 2,277  & 5,201   & 7,600  & 183   & 183   & 41,554     & 421,961  & 1,632,803  \\
Edges    & 5,278 & 4,552  & 44,324  & 245,861 & 119,081  & 31,371 & 198,353 & 26,659 & 279   & 277   & 1,362,229   & 984,979  & 30,622,564  \\
Features & 1,433 & 3,703  & 500     & 767     & 745            & 2325   & 2089    & 932    & 1703  & 1703  & 5         & 12      & 65  \\
Classes  & 7     & 6      & 3       & 10      & 8               & 5      & 5       & 5      & 5     & 5     & 2         & 2       & 2  \\ 
$\mathcal{H}_{\mathcal{G}}$  
         & 0.82  & 0.71   & 0.79    & 0.80    & 0.84          & 0.24 & 0.21  & 0.21 & 0.05 & 0.30 & 0.47    & 0.61   & 0.44 \\ \bottomrule
\end{tabular}
}
\end{table*}

\subsection{The Implementation of NFGNN}

According to the adaptive localized filtering in Eq.~(\ref{eq::final}) and the reparameterization trick, we can formalize the architecture of the proposed NFGNN. As pointed out in~\cite{SGC,APPNP}, the entanglement of feature transformation and filtering may be harmful to the performance and robustness of the GNN model. Hence, we adopt an APPNP-like manner~\cite{APPNP} to decouple the feature transformation and filtering operations in the implementation of NFGNN.

\noindent\textcolor{R2}{\textbf{Feature transformation.} Firstly, we apply an MLP to perform the non-linear transformation for the raw feature matrix $\mathbf{X}$ to obtain $\mathbf{X}^{(0)} = $ MLP$(\mathbf{X};\Theta)$, thus achieving the purpose of feature dimension scaling and increasing the model capacity. Then, the spectral filtering operation can be further performed.} 

\noindent\textcolor{R2}{\textbf{Feature filtering.}} As a traditionally used approximate kernel in GSP~\cite{chebyNet}, the Chebyshev polynomial $T_k(\cdot)$ is adopted to play the role of $\hat{p}_k(\cdot)$ in Eq.~(\ref{eq::final}). It can be computed in a recursive way due to the stable recurrence relation:
\begin{equation}
T_k(\tilde{\mathbf{L}}) = 2\tilde{\mathbf{L}}T_{k-1}(\tilde{\mathbf{L}}) - T_{k-2}(\tilde{\mathbf{L}})
\end{equation}
where $\tilde{\mathbf{L}} = 2\mathbf{L}/\lambda_{max} - \mathbf{I}_n$ is the scaled Laplacian. 
In particular, $T_0(\tilde{\mathbf{L}}) = 1$ and $T_1(\tilde{\mathbf{L}}) = \tilde{\mathbf{L}}$.
Accordingly, given the input $\mathbf{X}$, we will have $\mathbf{X}^{(k)} = T_k(\tilde{\mathbf{L}})\mathbf{X}^{(0)}$. Finally, as shown in Fig.~\ref{fig::pipeline}, $\mathbf{Z}_i$ can be computed according to Eq.~(\ref{eq::final}). The pipeline of the proposed node-oriented filtering is summarized in Alg.\ref{Alg:NSGNN}.

\textcolor{R2}{It can be noticed that the process of computing $\mathbf{X}^{(k)}$ requires the involvement of the Laplacian derived from the given graph, but is parameterization-free.}
Besides, since the proposed filtering is a general form that is not restricted by a particular polynomial basis, the various polynomial bases in~\cite{GPRGNN,BernNet} can also be adopted in our NFGNN. 

\renewcommand{\algorithmicrequire}{\textbf{Input:}} 
\renewcommand{\algorithmicensure}{\textbf{Output:}}
\begin{algorithm}[t]
\scriptsize
\caption{Node-oriented Spectral Filtering for GNNs in one epoch for training.}
\label{Alg:NSGNN} 
\begin{algorithmic}
 \REQUIRE $\mathbf{X} \in \mathbb{R}^{n \times f}$, $\tilde{\mathbf{L}}\in \mathbb{R}^{n \times n}$, $K$ \\
 \ENSURE $\mathbf{Z}$ \\ 
 \renewcommand{\algorithmicensure}{\textbf{Learnable Parameters:}}
 \ENSURE 
 $\mathbf{W}$: parameters for learning the node-dependent matrix
 
 $\mathbf{\Gamma}$: the learnable node-agnostic matrix
 
 $\Theta$: parameters of the MLP for feature transformation
 
 \STATE $\mathbf{X}^{(0)} \leftarrow $ MLP$(\mathbf{X};\Theta)$ ~~~~~~~/* \textit{Feature Transformation} */
 
        $\mathbf{X}^{(1)} \leftarrow \tilde{\mathbf{L}}\mathbf{X}^{(0)}$ , 
        $\mathbf{Z} \leftarrow \mathbf{0}$
    \STATE \textbf{for} $k =0 $ to $K+1$ \textbf{do}\\
   ~~~~~~$\mathbf{X}^{(k)} \leftarrow 2\tilde{\mathbf{L}}\mathbf{X}^{(k-1)} - \mathbf{X}^{(k-2)}$ if $k>1$  
   \\ 
    ~~~~~~\textbf{for} $i = 1$ to $n$ \textbf{do} \\
         ~~~~~~ ~~~~~~$\mathbf{H}_{i:} \leftarrow \sigma(\mathbf{X}^{(k)}_i\mathbf{W})$\\
         ~~~~~~ ~~~~~~$\eta_{i,k} \leftarrow \mathbf{H}_{i:}(\mathbf{\Gamma}_{k:})^\top$\\
         ~~~~~~ ~~~~~~$\mathbf{Z}_i \leftarrow \eta_{i,k} \mathbf{X}^{(k)}_i + \mathbf{Z}_i$ \\
    ~~~~~~\textbf{end} \textbf{for}\\
        \textbf{end} \textbf{for}
    \STATE Update the learnable parameters $\mathbf{W}$, $\mathbf{\Gamma}$, $\Theta$.
\end{algorithmic}
\end{algorithm}

\textcolor{revise}{In the inference phase, for the test set with $n_{test}$ nodes and the feature matrix $\mathbf{X}^{test} \in \mathbb{R}^{n_{test} \times f}$, it is indeed a transductive setting if $\mathbf{X}^{test} \subset \mathbf{X}$, and we can get the corresponding embedding $\mathbf{Z}^{test}$ directly by Alg.\ref{Alg:NSGNN}. }
\textcolor{R2}{In addition, we can find from Fig.~\ref{fig::pipeline} that the node-oriented filtering does not require that the test node must be available during the training process. It means the proposed NFGNN is also applicable to the inductive learning setting.
Therefore, when the test nodes are unseen in $\mathbf{X}$, the corresponding $\mathbf{Z}^{test}$ can also be obtained in an inductive way.}

\renewcommand{\algorithmicrequire}{\textbf{Input:}} 
\renewcommand{\algorithmicensure}{\textbf{Output:}}
\begin{algorithm}[t]
\scriptsize
\caption{\textcolor{revise}{Scalable Node-oriented Spectral Filtering for GNNs in one epoch}}
\label{Alg::scalable NFGNN} 
\begin{algorithmic}
 \REQUIRE \textcolor{revise}{$\mathbf{X} \in \mathbb{R}^{n \times f}$, $\tilde{\mathbf{L}}\in \mathbb{R}^{n \times n}$, $K$}
 \ENSURE \textcolor{revise}{$\mathbf{Z}$} \\ 
 \renewcommand{\algorithmicensure}{\textbf{Learnable Parameters:}}
 \ENSURE 
 \textcolor{revise}{$\mathbf{W}$},
 \textcolor{revise}{$\mathbf{\Gamma}$},
 \textcolor{revise}{$\Theta$}.
 
 \STATE 
  \textcolor{revise}{$\mathbf{X}^{(0)} \leftarrow \mathbf{X}$,
  $\mathbf{X}^{(1)} \leftarrow \tilde{\mathbf{L}}\mathbf{X}^{(0)}$}  
  ~~~~~~~
   
    \STATE \textcolor{revise}{\textbf{for} $k =0 $ to $K+1$ \textbf{do}~~~~~/* \textit{\textbf{Steps that can be precomputed}} */}
    \\
   ~~~~~~\textcolor{revise}{$\mathbf{X}^{(k)} \leftarrow 2\tilde{\mathbf{L}}\mathbf{X}^{(k-1)} - \mathbf{X}^{(k-2)}$ if $k>1$ }
   
   \textcolor{revise}{\textbf{end} \textbf{for}}\\ 
   \textcolor{revise}{$\tilde{\mathbf{Z}} \leftarrow \mathbf{0}$ }\\
   \STATE \textcolor{revise}{\textbf{for} $k =0 $ to $K+1$ \textbf{do}} \\
   ~~~~~~\textcolor{revise}{$\mathbf{H}=\sigma(\mathbf{X}^{(k)}\mathbf{W})$} \\
    ~~~~~~\textcolor{revise}{\textbf{for} $i = 1$ to $n$ \textbf{do}} \\
         ~~~~~~ ~~~~~~\textcolor{revise}{$\eta_{i,k} \leftarrow \mathbf{H}_{i:}(\mathbf{\Gamma}_{k:})^\top$}\\
         ~~~~~~ ~~~~~~\textcolor{revise}{$\tilde{\mathbf{Z}}_i \leftarrow \eta_{i,k} \mathbf{X}^{(k)}_i + \tilde{\mathbf{Z}}_i$} \\
    ~~~~~~\textcolor{revise}{\textbf{end} \textbf{for}}\\
        \textcolor{revise}{\textbf{end} \textbf{for}}
        
        \textcolor{revise}{$\mathbf{Z}\leftarrow $ MLP$(\tilde{\mathbf{Z}};\Theta)$ ~~~~~~~/* \textit{Feature Transformation} */}
    \STATE \textcolor{revise}{Update the learnable parameters $\mathbf{W}$, $\mathbf{\Gamma}$, $\Theta$.}
\end{algorithmic}
\end{algorithm}
\begin{figure}[t]
	\centering
	\includegraphics[width=2.75in]{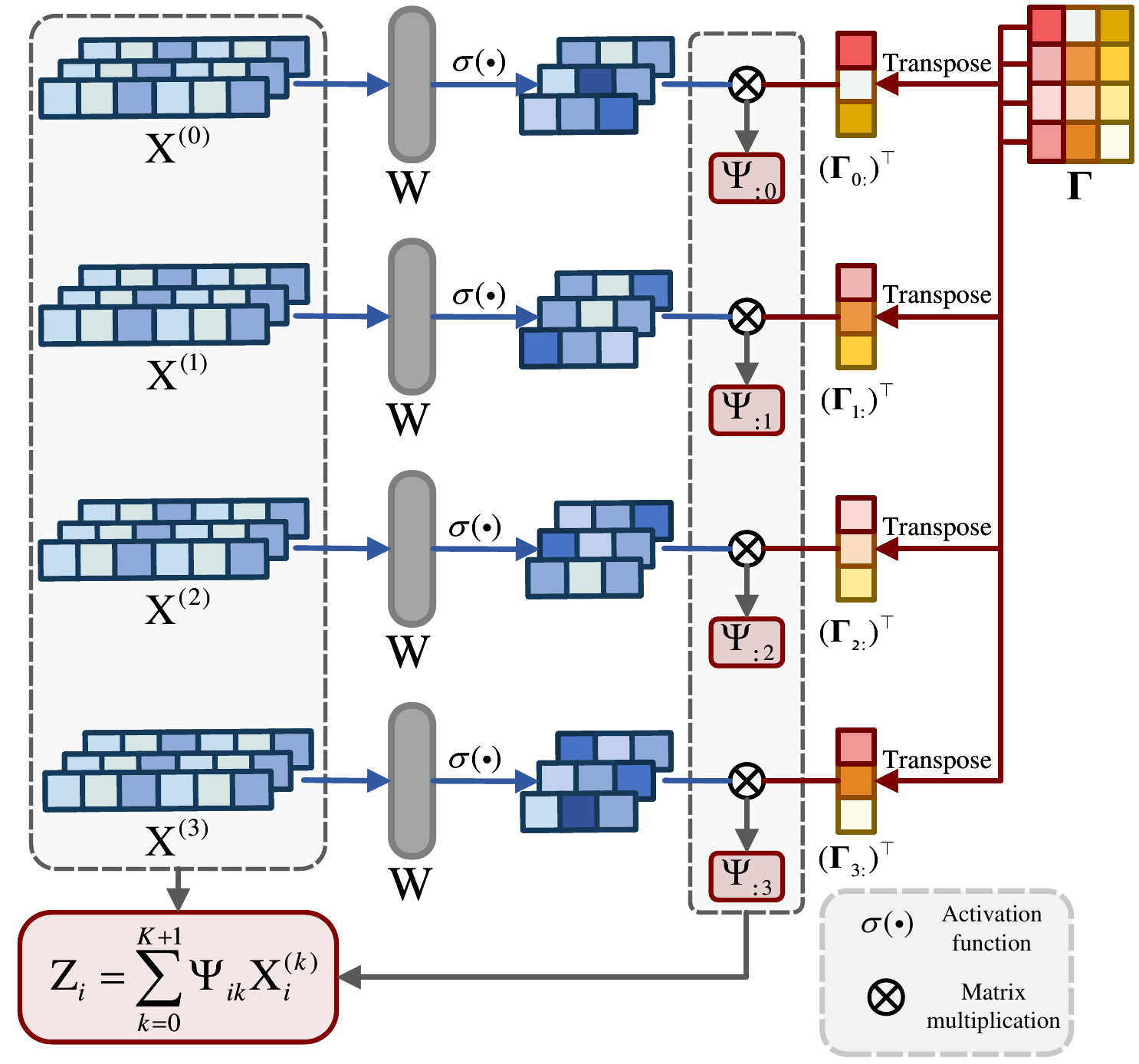}
	\vspace{-3mm}
	\caption{\textcolor{R2}{The pipeline of the proposed node-oriented spectral filtering when $K$ = 2 and $d$ = 3. $\mathbf{W}$ and $\mathbf{\Gamma}$ are the trainable parameter matrices. During the training phase, $\mathbf{W}$ and $\mathbf{\Gamma}$ will be updated through nodes from the training set. While in the test phase, we can use the trained $\mathbf{W}$ and $\mathbf{\Gamma}$ to obtain the embedding of the specific node from the test set.}}
	\label{fig::pipeline}
	\vspace{-5mm}
\end{figure}

\noindent\textbf{\textcolor{R2}{The scalable variant of NFGNN.}}
\label{sect::scalability}
\noindent\textcolor{revise}
{
Except for the implementation of NFGNN introduced above in an APPNP-like manner~\cite{APPNP}, 
NFGNN can also be implemented by adopting an SGC-like manner~\cite{SGC} to scale up to very large graphs like ogbn-paper100M with more than a billion nodes, i.e., performing feature filtering first and feature transformation last. In such a pipeline, we can precompute and save $\{\hat{p}_k(\mathbf{L})\mathbf{x}\}_{k=0}^{K}$ of Eq.~\eqref{eq::final} without training since there are no trainable parameters to be involved in the calculation of $\{\hat{p}_k(\mathbf{L})\mathbf{x}\}_{k=0}^{K}$. Then, the parameters of NFGNN can be trained in a mini-batch manner which is suitable for graphs with arbitrary scales without the limitation of GPU memory. 
The scalable variant of the proposed node-oriented filtering is summarized in Alg.~\ref{Alg::scalable NFGNN}}.

\section{Experimental Results and Analysis}
\label{sect::experiment}
\subsection{Experimental Settings}
\textbf{Datasets and Baselines.}

To provide a comprehensive evaluation of our method, several graphs from various domains with different homophily ratios and scales are used, which range from 183 to 1,632,803 nodes and from 277 to 1,362,229 edges. Specifically, there are 6 homophilic graphs: citation graphs Cora, CiteSeer, PubMed~\cite{sen2008collective}, co-purchase graphs Computers and Photo~\cite{pitfalls}, \textcolor{R2}{and three large-scale datasets from OGB~\cite{OGB}: ogbn-arxiv, ogbn-products, and ogbn-papers100M}; 5 heterophilic graphs with small scale: Chameleon and Squirrel~\cite{rozemberczki2021multi}, Actor, Texas, and Cornell~\cite{GeomGCN}; and 3 heterophilic graphs with large scale from~\cite{benchmark}: Penn94, genius, and pokec. The statistics of datasets except for the datasets from OGB are summarized in Table~\ref{datasets}. 
\begin{table*}[t]
\centering
\setlength{\tabcolsep}{2.5pt}
\caption{Results under the semi-supervised learning setting: mean accuracy ($\%$) $\pm$ $95\%$ confidence interval. Boldface letters mark the best result,
while underlined letters denote the second-best result.}
\vspace{-2mm}
\label{tab::semi}
\small
\resizebox{7in}{!}{
\begin{tabular}{@{}ccccccccccc@{}}
\toprule
 &
  Cora &
  Citeseer &
  PubMed &
  Computers &
  Photo &
  Chameleon &
  Actor &
  Squirrel &
  Texas &
  Cornell \\ \midrule

  \textcolor{revise}{P-value}
  & - 
  & - 
  & - 
  & \textcolor{revise}{$2.47 \times 10^{-33}$} 
  & \textcolor{revise}{$1.30 \times 10^{-6}$} 
  & - 
  & \textcolor{revise}{$5.03 \times 10^{-7}$} 
  & \textcolor{revise}{$7.78 \times 10^{-5}$} 
  & \textcolor{revise}{$1.44 \times 10^{-15}$} 
  & \textcolor{revise}{$4.18 \times 10^{-18}$} \\ \midrule

\textbf{NFGNN} &
  77.69\tiny{$\pm$0.91} &
  {\ul 67.74\tiny{$\pm$0.52}} &
  \textbf{85.07\tiny{$\pm$0.13}} &
  \textbf{84.18\tiny{$\pm$0.40}} &
  \textbf{92.16\tiny{$\pm$0.82}} &
  {\ul 40.10\tiny{$\pm$1.34}} &
  \textbf{30.33\tiny{$\pm$0.82}} &
  \textbf{30.91\tiny{$\pm$0.62}} &
  \textbf{50.37\tiny{$\pm$6.54}} &
  \textbf{48.37\tiny{$\pm$7.34}} \\  

BernNet &
  76.37\tiny{$\pm$0.36} &
  65.83\tiny{$\pm$0.61} &
  82.57\tiny{$\pm$0.17} &
  79.57\tiny{$\pm$0.28} &
  91.60\tiny{$\pm$0.35} &
  27.33\tiny{$\pm$1.14} &
  28.61\tiny{$\pm$0.55} &
  26.42\tiny{$\pm$0.39} &
  48.21\tiny{$\pm$3.17} &
  {\ul 46.50\tiny{$\pm$5.57}} \\
  
GPRGNN &
  \textbf{79.51\tiny{$\pm$0.36}} &
  67.63\tiny{$\pm$0.38} &
  \textbf{85.07\tiny{$\pm$0.09}} &
  {\ul 82.90\tiny{$\pm$0.37}} &
  {\ul 91.93\tiny{$\pm$0.26}} &
  31.46\tiny{$\pm$0.89} &
  28.32\tiny{$\pm$0.76} &
  26.14\tiny{$\pm$0.51} &
  44.76\tiny{$\pm$4.33} &
  43.45\tiny{$\pm$4.65} \\
APPNP &  
  {\ul 79.41\tiny{$\pm$0.38}} &
  \textbf{68.59\tiny{$\pm$0.30}} &
  {\ul 85.02\tiny{$\pm$0.09}} &
  81.99\tiny{$\pm$0.26} &
  91.11\tiny{$\pm$0.26} &
  30.12\tiny{$\pm$0.92} &
  {\ul 29.18\tiny{$\pm$0.62}} &
  26.01\tiny{$\pm$0.48} &
  46.44\tiny{$\pm$3.10} &
  44.37\tiny{$\pm$4.90} \\
ChebNet &
  71.39\tiny{$\pm$0.51} &
  65.67\tiny{$\pm$0.38} &
  83.83\tiny{$\pm$0.12} &
  82.41\tiny{$\pm$0.28} &
  90.09\tiny{$\pm$0.28} &
  32.36\tiny{$\pm$1.61} &
  26.52\tiny{$\pm$3.01} &
  27.76\tiny{$\pm$0.93} &
  28.51\tiny{$\pm$4.03} &
  26.21\tiny{$\pm$3.12} \\
GCN &
  75.21\tiny{$\pm$0.38} &
  67.30\tiny{$\pm$0.35} &
  84.27\tiny{$\pm$0.01} &
  82.52\tiny{$\pm$0.32} &
  90.54\tiny{$\pm$0.21} &
  35.73\tiny{$\pm$0.87} &
  22.43\tiny{$\pm$1.01} &
  28.37\tiny{$\pm$0.77} &
  32.40\tiny{$\pm$2.31} &
  35.90\tiny{$\pm$3.64} \\ \midrule
  

BMGCN  &
  74.07\tiny{$\pm$0.25} &
  64.34\tiny{$\pm$0.92} &
  84.71\tiny{$\pm$0.34} &
  NA &
  NA &
  \textbf{41.22\tiny{$\pm$1.11}} &
  NA &
  {\ul 30.28\tiny{$\pm$0.96}} &
  48.21\tiny{$\pm$4.39} &
  NA \\
  
  FAGCN &
  78.10\tiny{$\pm$0.21} &
  66.77\tiny{$\pm$0.18} &
  84.09\tiny{$\pm$0.02} &
  82.11\tiny{$\pm$1.55} &
  90.39\tiny{$\pm$1.34} &
  37.24\tiny{$\pm$3.54} &
  28.02\tiny{$\pm$2.21}&
  26.91\tiny{$\pm$3.08} &
{\ul  48.44\tiny{$\pm$1.78}} &
  46.38\tiny{$\pm$1.82}  \\


GAT &
  76.70\tiny{$\pm$0.42} &
  67.20\tiny{$\pm$0.46} &
  83.28\tiny{$\pm$0.12} &
  81.95\tiny{$\pm$0.38} &
  90.09\tiny{$\pm$0.27} &
  36.17\tiny{$\pm$0.47} &
  23.71\tiny{$\pm$0.91} &
  27.55\tiny{$\pm$1.19} &
  35.11\tiny{$\pm$3.32} &
  37.37\tiny{$\pm$3.74} \\
  
  MLP &
  50.34\tiny{$\pm$0.48} &
  52.88\tiny{$\pm$0.51} &
  80.57\tiny{$\pm$0.12} &
  70.48\tiny{$\pm$0.28} &
  78.69\tiny{$\pm$0.30} &
  22.47\tiny{$\pm$2.05} &
  28.77\tiny{$\pm$0.50} &
  24.10\tiny{$\pm$0.93} &
  45.15\tiny{$\pm$2.46} &
  46.28\tiny{$\pm$4.73} \\

  LINK &
  42.94\tiny{$\pm$2.02} &
  25.52\tiny{$\pm$1.98} &
  54.78\tiny{$\pm$0.96} &
  70.05\tiny{$\pm$1.31} &
  78.84\tiny{$\pm$1.45} &
  38.97\tiny{$\pm$1.92} &
  21.07\tiny{$\pm$0.71} &
  29.53\tiny{$\pm$0.58} &
  44.89\tiny{$\pm$7.08} &
  27.48\tiny{$\pm$5.34} \\ \bottomrule
\end{tabular}}
\normalsize
\vspace{-3mm}
\end{table*}

\begin{table*}[t]
\centering
\setlength{\tabcolsep}{2.5pt}
\caption{Results under the full-supervised learning setting: mean accuracy ($\%$) $\pm$ $95\%$ confidence interval.}
\vspace{-2mm}
\label{tab::full}
\small
\resizebox{7in}{!}{
\begin{tabular}{@{}ccccccccccc@{}}
\toprule
 &
  Cora &
  Citeseer &
  PubMed &
  Computers &
  Photo &
  Chameleon &
  Actor &
  Squirrel &
  Texas &
  Cornell \\ \midrule

  \textcolor{revise}{P-value}
  & \textcolor{revise}{$1.37 \times 10^{-10}$}
  & - 
  & \textcolor{revise}{$2.32 \times 10^{-14}$}  
  & \textcolor{revise}{$9.43 \times 10^{-11}$} 
  & \textcolor{revise}{$5.66 \times 10^{-28}$} 
  & \textcolor{revise}{$2.62 \times 10^{-4}$ }
  & \textcolor{revise}{$1.15 \times 10^{-2}$ }
  & - 
  & \textcolor{revise}{$8.42 \times 10^{-13}$} 
  & - \\ \midrule

\textbf{NFGNN} &
  \textbf{89.82\tiny{$\pm$0.43}} &
  {\ul 80.56\tiny{$\pm$0.55}} &
  \textbf{89.89\tiny{$\pm$0.68}} &
  \textbf{90.31\tiny{$\pm$0.42}} &
  \textbf{94.85\tiny{$\pm$0.24}} &
  \textbf{72.52\tiny{$\pm$0.59}} &
  \textbf{40.62\tiny{$\pm$0.38}} &
  {\ul58.90\tiny{$\pm$0.35}} &
  \textbf{94.03\tiny{$\pm$0.82}} &
  {\ul 91.90\tiny{$\pm$0.91}} \\  

BernNet &
  88.06\tiny{$\pm$0.91} &
  80.17\tiny{$\pm$0.78} &
  88.79\tiny{$\pm$0.25} &
  88.61\tiny{$\pm$0.41} &
  93.32\tiny{$\pm$0.40} &
  68.73\tiny{$\pm$0.57} &
  {\ul 40.01\tiny{$\pm$0.42}} &
  50.75\tiny{$\pm$0.67} &
  92.30\tiny{$\pm$1.23} &
  \textbf{91.96\tiny{$\pm$1.07}} \\
  
GPRGNN &
  {\ul 88.57\tiny{$\pm$0.69}} &
  80.13\tiny{$\pm$0.67} &
  {\ul 88.92\tiny{$\pm$0.68}} &
  87.20\tiny{$\pm$0.47} &
  {\ul 93.87\tiny{$\pm$0.42}} &
  67.48\tiny{$\pm$0.40} &
  39.30\tiny{$\pm$0.27} &
  49.93\tiny{$\pm$0.53} &
  92.92\tiny{$\pm$0.61} &
  91.36\tiny{$\pm$0.70} \\
APPNP &  
  88.20\tiny{$\pm$0.62} &
  80.22\tiny{$\pm$0.63} &
  87.98\tiny{$\pm$0.67} &
  88.04\tiny{$\pm$0.51} &
  90.32\tiny{$\pm$0.17} &
  51.91\tiny{$\pm$0.56} &
  38.86\tiny{$\pm$0.24} &
  34.77\tiny{$\pm$0.34} &
  91.18\tiny{$\pm$0.70} &
  91.80\tiny{$\pm$0.63} \\
ChebNet &
  86.36\tiny{$\pm$0.46} &
  79.32\tiny{$\pm$0.39} &
  88.12\tiny{$\pm$0.43} &
  87.82\tiny{$\pm$0.72} &
  93.58\tiny{$\pm$0.19} &
  59.96\tiny{$\pm$0.51} &
  38.02\tiny{$\pm$0.23} &
  40.67\tiny{$\pm$0.31} &
  86.08\tiny{$\pm$0.96} &
  85.33\tiny{$\pm$1.04} \\
GCN &
  86.99\tiny{$\pm$1.23} &
  79.67\tiny{$\pm$0.86} &
  86.65\tiny{$\pm$0.70} &
  86.64\tiny{$\pm$1.04} &
  92.49\tiny{$\pm$0.55} &
  60.96\tiny{$\pm$0.78} &
  30.59\tiny{$\pm$0.23} &
  45.66\tiny{$\pm$0.39} &
  75.16\tiny{$\pm$0.96} &
  66.72\tiny{$\pm$1.37} \\ \midrule
  

BMGCN  &
  87.49\tiny{$\pm$1.18} &
  80.54\tiny{$\pm$0.35} &
  88.38\tiny{$\pm$0.27} &
  NA &
  NA &
  69.69\tiny{$\pm$1.21} &
  NA &
  53.16\tiny{$\pm$0.74} &
  {\ul 93.00\tiny{$\pm$0.57}} &
  NA \\
  
  FAGCN &
  {\ul88.57\tiny{$\pm$0.92}} &
  \textbf{83.51\tiny{$\pm$0.43}} &
  86.08\tiny{$\pm$0.33} &
  {\ul 89.68\tiny{$\pm$0.97}} &
  93.67\tiny{$\pm$0.50} &
  61.59\tiny{$\pm$1.98} &
  39.08\tiny{$\pm$0.65}&
  44.41\tiny{$\pm$0.62} &
  89.61\tiny{$\pm$1.52} &
  88.52\tiny{$\pm$1.33}  \\


GAT &
  88.13\tiny{$\pm$0.72} &
  80.33\tiny{$\pm$0.60} &
  87.25\tiny{$\pm$0.24} &
  85.83\tiny{$\pm$0.49} &
  92.94\tiny{$\pm$0.66} &
  63.90\tiny{$\pm$0.46} &
  35.98\tiny{$\pm$0.23} &
  42.72\tiny{$\pm$0.33} &
  78.87\tiny{$\pm$0.86} &
  76.00\tiny{$\pm$1.01} \\
  MLP &
  77.37\tiny{$\pm$0.51} &
  76.26\tiny{$\pm$0.49} &
  85.82\tiny{$\pm$0.30} &
  81.78\tiny{$\pm$0.76} &
  88.17\tiny{$\pm$0.62} &
  46.72\tiny{$\pm$0.46} &
  38.58\tiny{$\pm$0.25} &
  31.28\tiny{$\pm$0.27} &
  92.26\tiny{$\pm$0.71} &
  91.36\tiny{$\pm$0.70} \\
  LINK &
  80.88\tiny{$\pm$0.58} &
  60.76\tiny{$\pm$1.35} &
  81.14\tiny{$\pm$0.26} &
  83.71\tiny{$\pm$0.32} &
  89.28\tiny{$\pm$0.31} &
  {\ul 71.09\tiny{$\pm$1.16}} &
  26.25\tiny{$\pm$1.43} &
  \textbf{59.77\tiny{$\pm$1.27}} &
  89.61\tiny{$\pm$1.52} &
  44.91\tiny{$\pm$2.19} \\ \bottomrule
\end{tabular}}
\normalsize
\vspace{-3mm}
\end{table*}
Several baselines have been selected for comparison, including 8 GNN baselines from both spatial and spectral perspectives, GCN~\cite{GCN}, ChebNet~\cite{chebyNet},  APPNP~\cite{APPNP},  GPRGNN~\cite{GPRGNN}, BernNet~\cite{BernNet}, GAT~\cite{GAT}, BMGCN~\cite{Graphsage}, and FAGCN~\cite{FAGCN}. \textcolor{black}{Besides, 2 classical non-GNN baselines are also chosen: MLP and LINK~\cite{LINK}}. 
Specifically, we use 2 GCN layers with 64 hidden units for GCN implementation. In addition, 2 GAT layers are adopted for GAT implementation, where the attention heads are (8, 1), and the number of hidden units of each head is (8, 64), respectively. For ChebyNet, each layer is set to 2 propagation steps with 32 hidden units. For APPNP, we use a 2-layer MLP with 64 hidden units for feature transformation and 10 steps for feature propagation. For SGC, we use the default $K$ = 2. The MLP in baselines is the same as the 2-layer MLP of APPNP.

\begin{table}[t]
\centering
\caption{Results on large-scale graphs: Mean accuracy (\%) $\pm$ standard deviation. "OOM" represents out of memory and '*' denotes the results of the scalable version of the model.}
\vspace{-3mm}
\label{tab::large}
\resizebox{3.5in}{!}{
\setlength{\tabcolsep}{2.5pt}
\begin{tabular}{lcccccc}
\toprule
                
& Penn94         
& genius         
& pokec         
& \textcolor{revise}{\begin{tabular}[c]{@{}c@{}}ogbn\\ arxiv\end{tabular}}
& \textcolor{R2}{\begin{tabular}[c]{@{}c@{}}ogbn\\ products\end{tabular}}
& \textcolor{R2}{\begin{tabular}[c]{@{}c@{}}ogbn\\ papers100M\end{tabular}}\\ \midrule
\textbf{NFGNN } 
& 84.10\tiny{$\pm$0.40} 
& 90.94\tiny{$\pm$0.43} 
& 81.37\tiny{$\pm$0.11} 
& \textcolor{revise}{72.13\tiny{$\pm$0.33}}
& \textcolor{R2}{81.02\tiny{$\pm$0.31}}*
& \textcolor{R2}{67.11\tiny{$\pm$0.52}}*\\
BernNet         
& 81.41\tiny{$\pm$0.27} 
& 90.13\tiny{$\pm$0.22} 
& 79.36\tiny{$\pm$0.09} 
& \textcolor{revise}{71.40\tiny{$\pm$0.37}}
& \textcolor{R2}{80.01\tiny{$\pm$0.27}}*
& \textcolor{R2}{65.12\tiny{$\pm$0.35}}*\\
GPRGNN          
& 81.38\tiny{$\pm$0.16} 
& 90.05\tiny{$\pm$0.31} 
& 78.83\tiny{$\pm$0.05} 
& \textcolor{revise}{71.68\tiny{$\pm$0.44}}
& \textcolor{R2}{79.93\tiny{$\pm$0.44}}*
& \textcolor{R2}{65.76\tiny{$\pm$0.40}}*\\
APPNP           
& 74.33\tiny{$\pm$0.38} 
& 85.36\tiny{$\pm$0.62} 
& 62.58\tiny{$\pm$0.08} 
& \textcolor{revise}{70.97\tiny{$\pm$0.71}} 
& \textcolor{R2}{76.57\tiny{$\pm$0.20}}
& \textcolor{R2}{OOM}\\
GCN             
& 82.47\tiny{$\pm$0.27} 
& 87.42\tiny{$\pm$0.37} 
& 75.45\tiny{$\pm$0.17} 
& \textcolor{revise}{71.03\tiny{$\pm$0.38}}
& \textcolor{R2}{OOM}
& \textcolor{R2}{OOM}\\

\bottomrule
\end{tabular}
}
\vspace{-5mm}
\end{table}

\begin{table*}[t]
\centering
\setlength{\tabcolsep}{2.5pt}
\caption{\textcolor{R2}{Results on the observed and the unobserved test subsets of inductive node classification: mean accuracy ($\%$) $\pm$ $95\%$ confidence interval.}}

\label{tab::inductive}
\small
\resizebox{0.9\textwidth}{!}{
\begin{tabular}{@{}l|l|cccccccccc@{}}
\toprule
&         
& Cora
& Citeseer
& PubMed
& Computers
& Photo
& Chameleon
& Actor
& Squirrel
& Texas
& Cornell \\ \midrule
\multirow{5}{*}{Observed}
  & \textcolor{R2}{P-value}
  & -
  & \textcolor{R2}{$7.53 \times 10^{-7}$} 
  & \textcolor{R2}{$1.41 \times 10^{-4}$}  
  & -
  & \textcolor{R2}{$8.26 \times 10^{-24}$} 
  & \textcolor{R2}{$3.57 \times 10^{-30}$ }
  & \textcolor{R2}{$6.72 \times 10^{-6}$ }
  & \textcolor{R2}{$8.14 \times 10^{-18}$ } 
  & \textcolor{R2}{$1.24 \times 10^{-33}$} 
  & \textcolor{R2}{$2.75 \times 10^{-17}$ } \\ \cmidrule(l){2-12} 

& NFGNN   
& \textcolor{R2}{78.78\tiny{$\pm$1.02} }    
& \textcolor{R2}{\textbf{68.79}\tiny{$\pm$0.57} }        
& \textcolor{R2}{\textbf{85.84}\tiny{$\pm$0.43} }      
& \textcolor{R2}{86.91\tiny{$\pm$0.71}          }
& \textcolor{R2}{\textbf{92.51}\tiny{$\pm$0.32} }     
& \textcolor{R2}{\textbf{42.23} \tiny{$\pm$1.96}}         
& \textcolor{R2}{\textbf{32.80} \tiny{$\pm$0.77}}     
& \textcolor{R2}{\textbf{33.59} \tiny{$\pm$1.07}}        
& \textcolor{R2}{\textbf{54.51} \tiny{$\pm$6.78}}     
& \textcolor{R2}{\textbf{47.97} \tiny{$\pm$9.35}}       
\\
                            
& BernNet 
& \textcolor{R2}{\textbf{78.98} \tiny{$\pm$0.94} }   
& \textcolor{R2}{68.23 \tiny{$\pm$1.60}        }
& \textcolor{R2}{85.26 \tiny{$\pm$0.26}      }
& \textcolor{R2}{\textbf{87.17} \tiny{$\pm$0.79} }        
& \textcolor{R2}{91.74 \tiny{$\pm$1.37}     }
& \textcolor{R2}{36.87 \tiny{$\pm$4.38}         }
& \textcolor{R2}{32.14 \tiny{$\pm$1.01}     }
& \textcolor{R2}{29.75 \tiny{$\pm$0.61}        }
& \textcolor{R2}{50.27 \tiny{$\pm$4.28}     }
& \textcolor{R2}{46.13 \tiny{$\pm$8.39}       }
\\
                            
& GPRGNN  
& \textcolor{R2}{77.43 \tiny{$\pm$1.70}}    
& \textcolor{R2}{66.82 \tiny{$\pm$1.25}}        
& \textcolor{R2}{85.14 \tiny{$\pm$0.31}}      
& \textcolor{R2}{86.09 \tiny{$\pm$0.59}}         
& \textcolor{R2}{91.49 \tiny{$\pm$0.62}}     
& \textcolor{R2}{35.80 \tiny{$\pm$2.59}}         
& \textcolor{R2}{31.24 \tiny{$\pm$1.03}}     
& \textcolor{R2}{28.47 \tiny{$\pm$0.55}}        
& \textcolor{R2}{46.19 \tiny{$\pm$5.81}}     
& \textcolor{R2}{46.32 \tiny{$\pm$8.27}}       
\\
                            
& APPNP   
& \textcolor{R2}{77.28 \tiny{$\pm$1.81} }   
& \textcolor{R2}{67.73 \tiny{$\pm$0.69} }       
& \textcolor{R2}{83.97 \tiny{$\pm$0.32} }     
& \textcolor{R2}{86.30 \tiny{$\pm$0.88} }        
& \textcolor{R2}{91.52 \tiny{$\pm$0.94} }    
& \textcolor{R2}{29.63 \tiny{$\pm$2.03} }        
& \textcolor{R2}{30.82 \tiny{$\pm$0.80} }    
& \textcolor{R2}{26.94 \tiny{$\pm$0.26} }       
& \textcolor{R2}{44.35 \tiny{$\pm$6.42} }    
& \textcolor{R2}{43.96 \tiny{$\pm$10.40}}       
\\ \midrule

\multirow{5}{*}{Unobserved} 
  & \textcolor{R2}{P-value}
  & \textcolor{R2}{$3.81 \times 10^{-12}$}
  & \textcolor{R2}{$1.44 \times 10^{-24}$} 
  & \textcolor{R2}{$6.48 \times 10^{-17}$}  
  & \textcolor{R2}{$4.13 \times 10^{-5}$} 
  & \textcolor{R2}{$1.79 \times 10^{-2}$} 
  & \textcolor{R2}{$2.66 \times 10^{-8}$ }
  & \textcolor{R2}{$5.32 \times 10^{-10}$ }
  & \textcolor{R2}{$7.96 \times 10^{-16}$}  
  & \textcolor{R2}{$2.41 \times 10^{-4}$} 
  & \textcolor{R2}{$2.83 \times 10^{-2}$}  \\ \cmidrule(l){2-12}

& NFGNN   
& \textcolor{R2}{\textbf{70.09} \tiny{$\pm$2.78}}    
& \textcolor{R2}{\textbf{63.53} \tiny{$\pm$1.88}}        
& \textcolor{R2}{\textbf{82.90} \tiny{$\pm$0.49}}      
& \textcolor{R2}{\textbf{82.85} \tiny{$\pm$1.02}}         
& \textcolor{R2}{\textbf{89.41} \tiny{$\pm$0.92}}     
& \textcolor{R2}{\textbf{32.89} \tiny{$\pm$2.43}}         
& \textcolor{R2}{\textbf{32.04} \tiny{$\pm$1.98}}     
& \textcolor{R2}{\textbf{27.29} \tiny{$\pm$2.17}}        
& \textcolor{R2}{\textbf{47.35} \tiny{$\pm$7.85}}     
& \textcolor{R2}{\textbf{43.97} \tiny{$\pm$9.55}}       
\\
& BernNet 
& \textcolor{R2}{69.13 \tiny{$\pm$3.47}}    
& \textcolor{R2}{62.72 \tiny{$\pm$1.53}}        
& \textcolor{R2}{81.25 \tiny{$\pm$0.51}}      
& \textcolor{R2}{80.64 \tiny{$\pm$1.09}}         
& \textcolor{R2}{89.25 \tiny{$\pm$1.20}}     
& \textcolor{R2}{31.55 \tiny{$\pm$4.88}}         
& \textcolor{R2}{30.51 \tiny{$\pm$1.39}}     
& \textcolor{R2}{26.12 \tiny{$\pm$1.92}}        
& \textcolor{R2}{47.05 \tiny{$\pm$6.96}}     
& \textcolor{R2}{40.57 \tiny{$\pm$8.51}}       
\\
& GPRGNN  
& \textcolor{R2}{66.36 \tiny{$\pm$2.54}}    
& \textcolor{R2}{61.77 \tiny{$\pm$3.31}}        
& \textcolor{R2}{82.67 \tiny{$\pm$0.56}}      
& \textcolor{R2}{82.60 \tiny{$\pm$0.99}}         
& \textcolor{R2}{89.19 \tiny{$\pm$0.94}}     
& \textcolor{R2}{31.49 \tiny{$\pm$3.34}}         
& \textcolor{R2}{31.72 \tiny{$\pm$1.79}}     
& \textcolor{R2}{26.80 \tiny{$\pm$1.37}}        
& \textcolor{R2}{45.12 \tiny{$\pm$9.41}}     
& \textcolor{R2}{43.81 \tiny{$\pm$7.83}}       
\\
& APPNP   
& \textcolor{R2}{67.48 \tiny{$\pm$1.97}}    
& \textcolor{R2}{61.55 \tiny{$\pm$2.62}}        
& \textcolor{R2}{81.50 \tiny{$\pm$0.72}}      
& \textcolor{R2}{79.45 \tiny{$\pm$1.41}}         
& \textcolor{R2}{86.39 \tiny{$\pm$1.73}}     
& \textcolor{R2}{27.32 \tiny{$\pm$2.37}}         
& \textcolor{R2}{30.95 \tiny{$\pm$1.68}}     
& \textcolor{R2}{23.24 \tiny{$\pm$1.10}}        
& \textcolor{R2}{44.87 \tiny{$\pm$8.78}}     
& \textcolor{R2}{42.65 \tiny{$\pm$9.47}}      
\\ \bottomrule
\end{tabular}}
\normalsize
\vspace{-10pt}
\end{table*}
\noindent\textbf{Experimental Setup.} 
\textcolor{R2}{For \textbf{the node classification task under the transductive setting}}, we set two random split ratios to split the dataset into training/validation/test under two settings. Specifically, the sparse splitting ratio (2.5\%/2.5\%/95\%) is used for the semi-supervised learning setting, and the dense splitting ratio (60\%/20\%/20\%) for the full-supervised learning setting. As exceptions, following the settings in ~\cite{benchmark}, a splitting ratio (50\%/25\%/25\%) is applied for the 3 heterophilic graphs with large scale from~\cite{benchmark}. \textcolor{revise}{For the datasets from OGB~\cite{OGB}, we adopt the same official splitting from OGB}.
We run each experiment 50 times with random initialization, and using random data split seeds. Finally, we report the average results with a 95\% confidence interval.

\textcolor{R2}{For \textbf{the node classification task under the inductive setting}, based on the sparse splitting ratio, we followed the setup in~\cite{GLNN} to further divide the test subset into the observed test subset and the unobserved test subset in a ratio of 8:2. The nodes in the training set, validation set, and the observed test subset are available in the training phase, while the nodes in the unobserved test subset and the corresponding edges can only be available during the test phase. We run each experiment 20 times with random initialization and random data split seeds. Finally, we report the average results with a 95\% confidence interval.}


For GPRGNN, BernNet, and BMGCN, we use the best combination of hyperparameters provided in the original paper to report the results for each dataset. For our NFGNN, a 2-layer MLP is used for feature transformation, whose dropout rate $dp_l$ is set to 0.5 for all datasets, and hidden units are set to ($f_{h}, |\mathcal{Y}|$).  Besides, we use lr$_l$ to denote the learning rate of the MLP. lr$_p$, $dp_p$ are used to denote the learning rate and dropout rate of the node-oriented filtering layer. The Adam~\cite{Adam} with weight decay $L_2$ is employed as the optimizer for the NFGNN training.
For each dataset, we search the optimal lr$_l$ within $\{0.01, 0.05\}$, lr$_p$ within $\{0.001, 0.005, 0.01\}$, $dp_p$ within $\{0, 0.1, 0.2, 0.5, 0.7, 0.8, 0.9\}$, $f_h$ within $\{ 16, 32, 64\}$, and $L_2$ within $\{ 0.0001, 0.0005, 0.001\}$. We set the order of the polynomial $K = 10$ and the rank $d = 1$ for all datasets for the node classification task under transductive setting, since $\mathbf{\Gamma}$ can be directly seen as $\{\mathbf{\gamma}_k\}_{k=0}^K$ when $d$ is set to 1,  which is easy to visualize and analyze. For the visualization analysis, it will be discussed in Sect.~\ref{node_analysis}.


\begin{figure*}[t]
    \centering
   \subfigure[Cora]{
   \includegraphics[width=40mm]{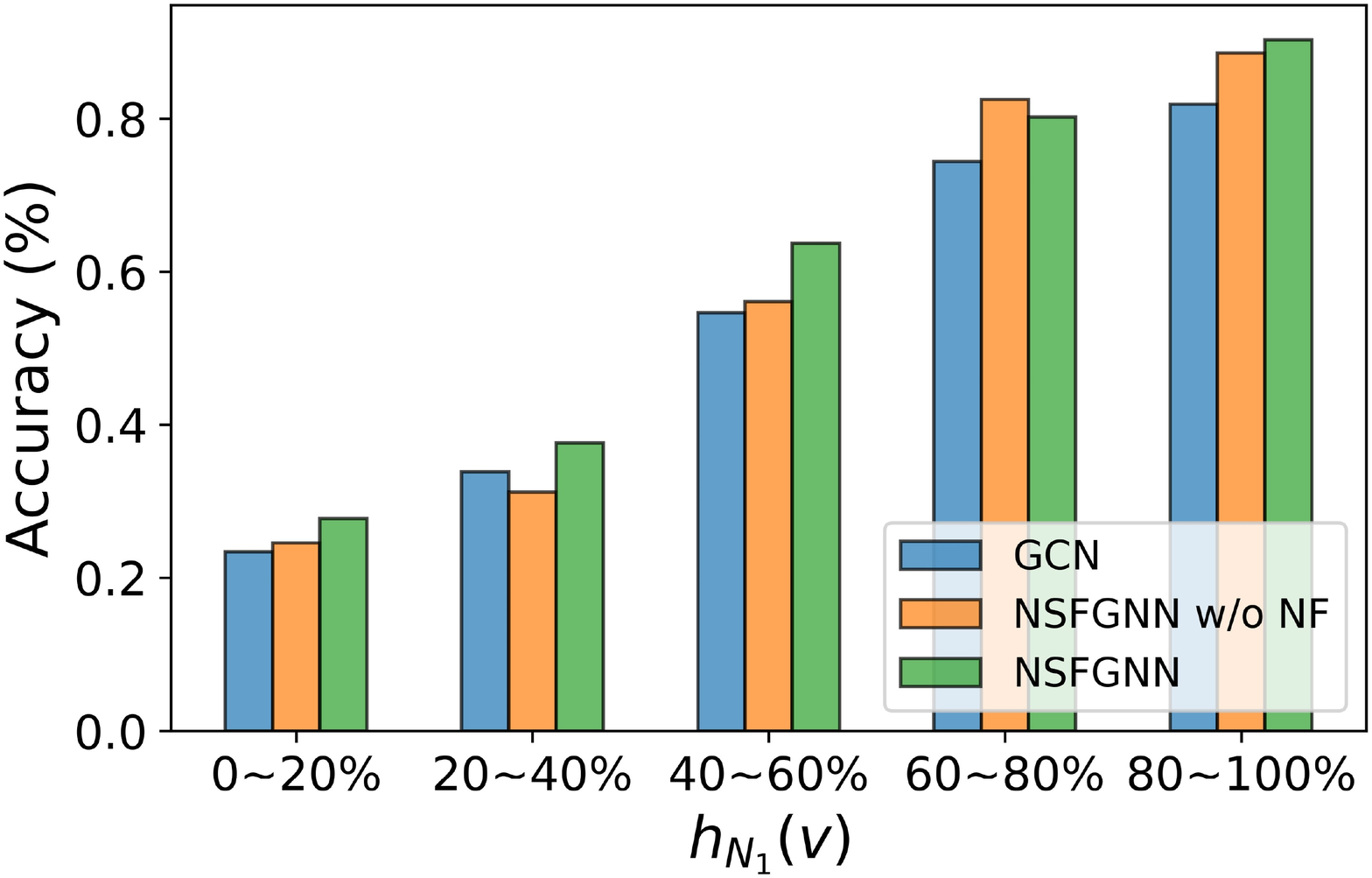}
   }
   \hspace{-2mm}
   \subfigure[Citeseer]{
   \includegraphics[width=40mm]{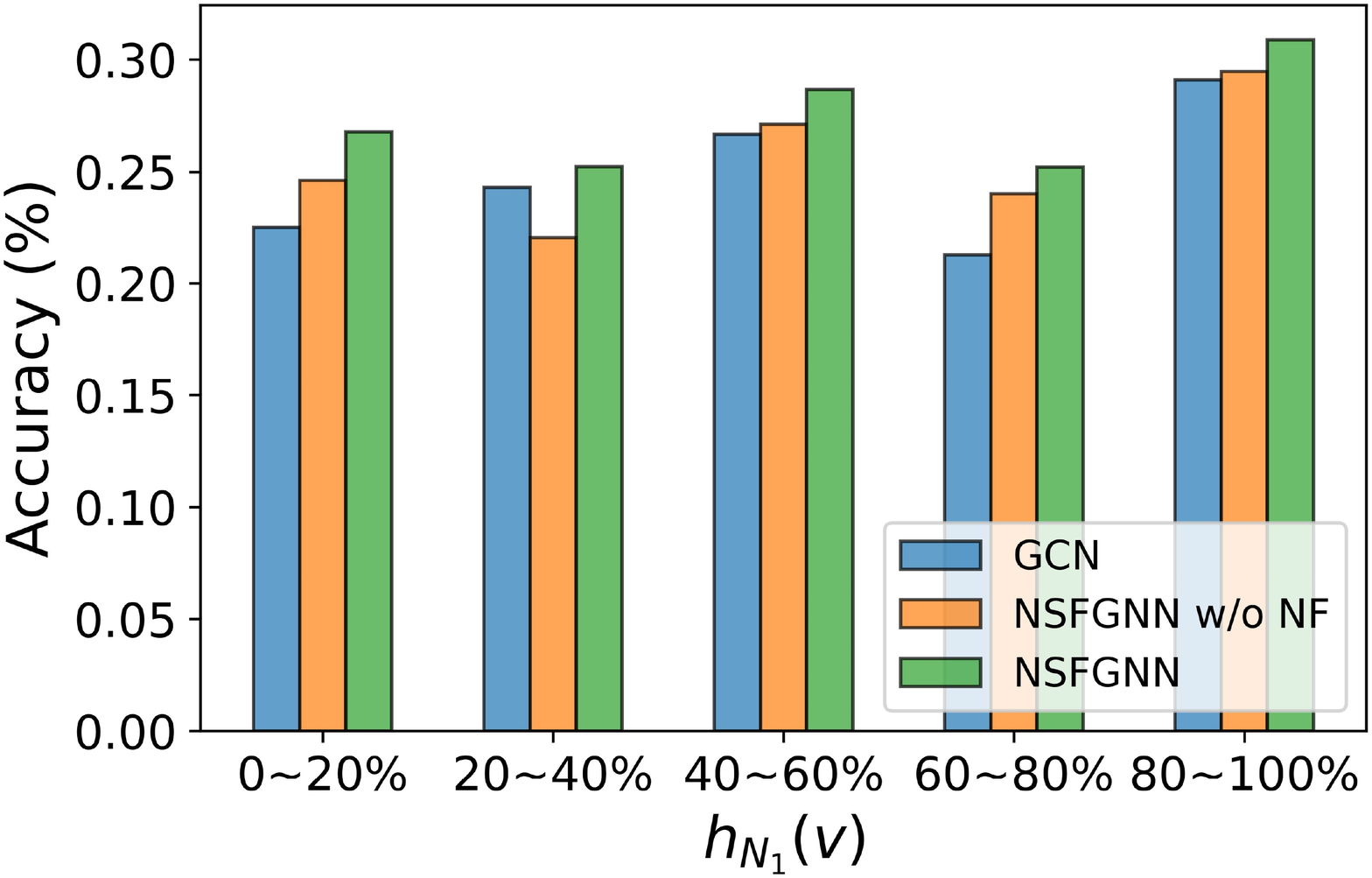}
   }
   \hspace{-2mm}
   \subfigure[Texas]{
   \includegraphics[width=40mm]{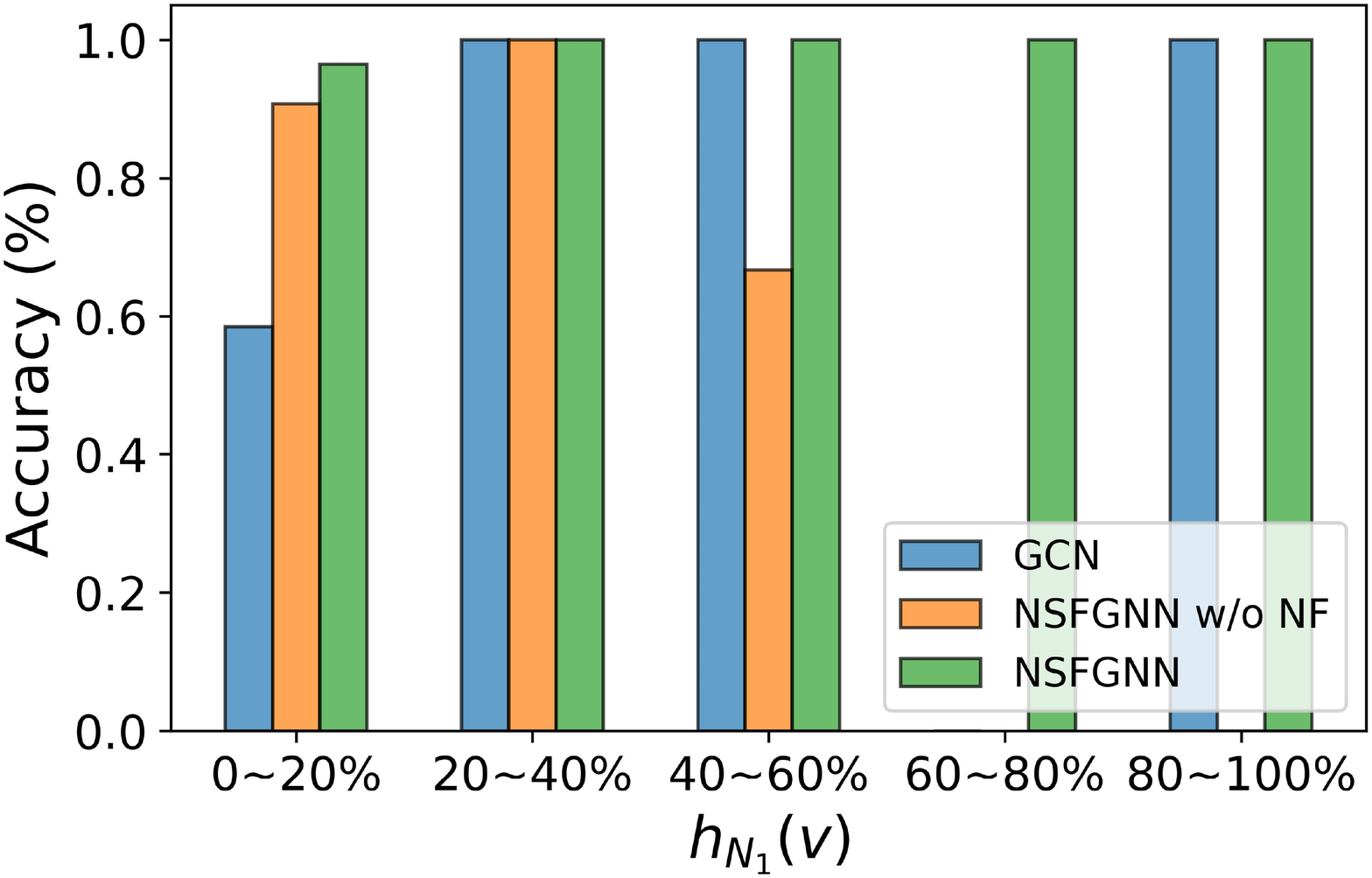}
   }
   \hspace{-2mm}
   \subfigure[Actor]{
   \includegraphics[width=40mm]{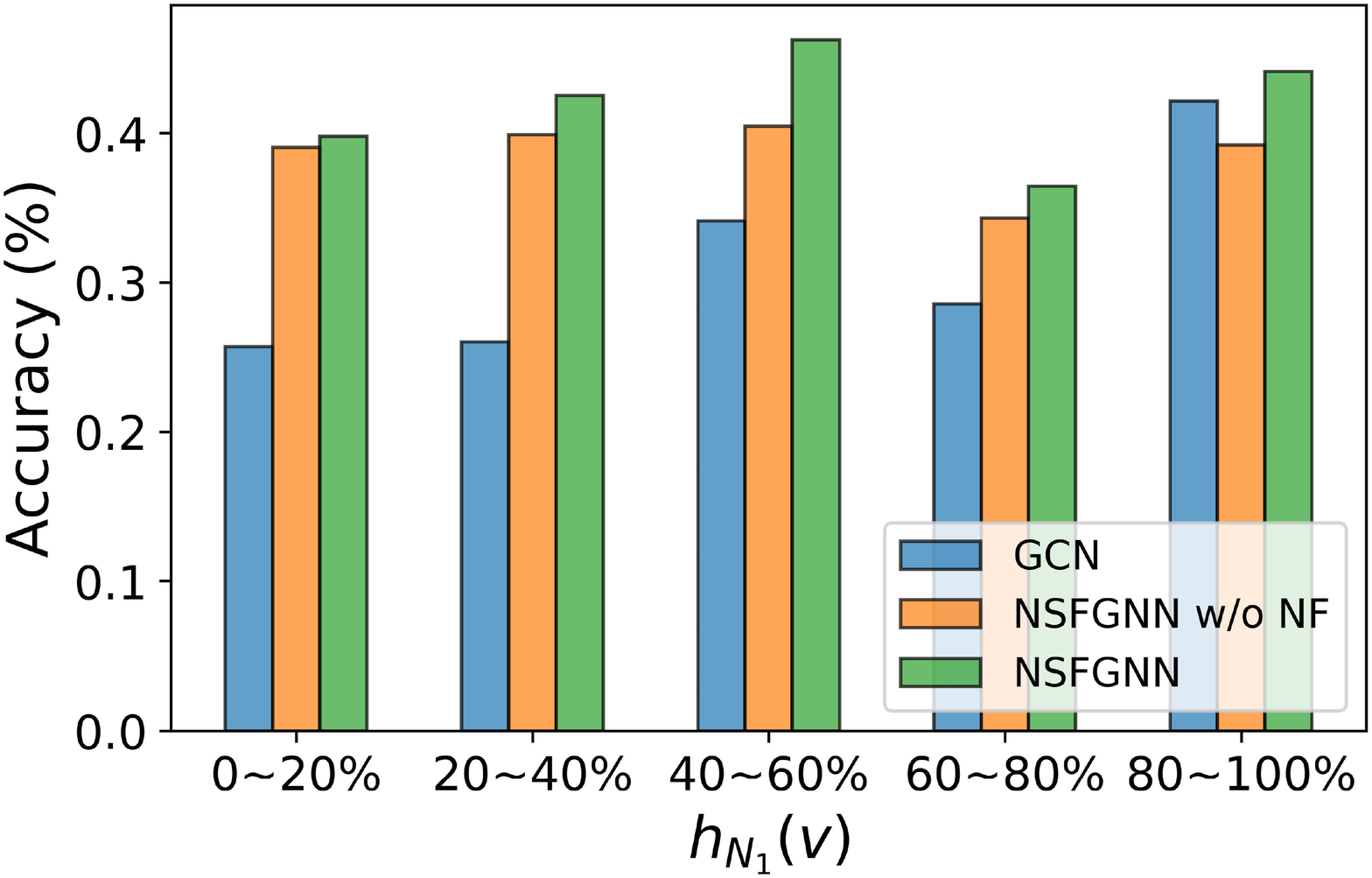}
   }
   \vspace{-3mm}
    \caption{Mean classification accuracy of nodes range by homophily ratio $h_{N_1}(v)$ on four datasets. 
    }
   \vspace{-3mm}
    \label{fig::node_analysis}
\end{figure*}

\begin{figure*}[t]
    \centering
   \subfigure[Cora]{
   \includegraphics[width=40mm]{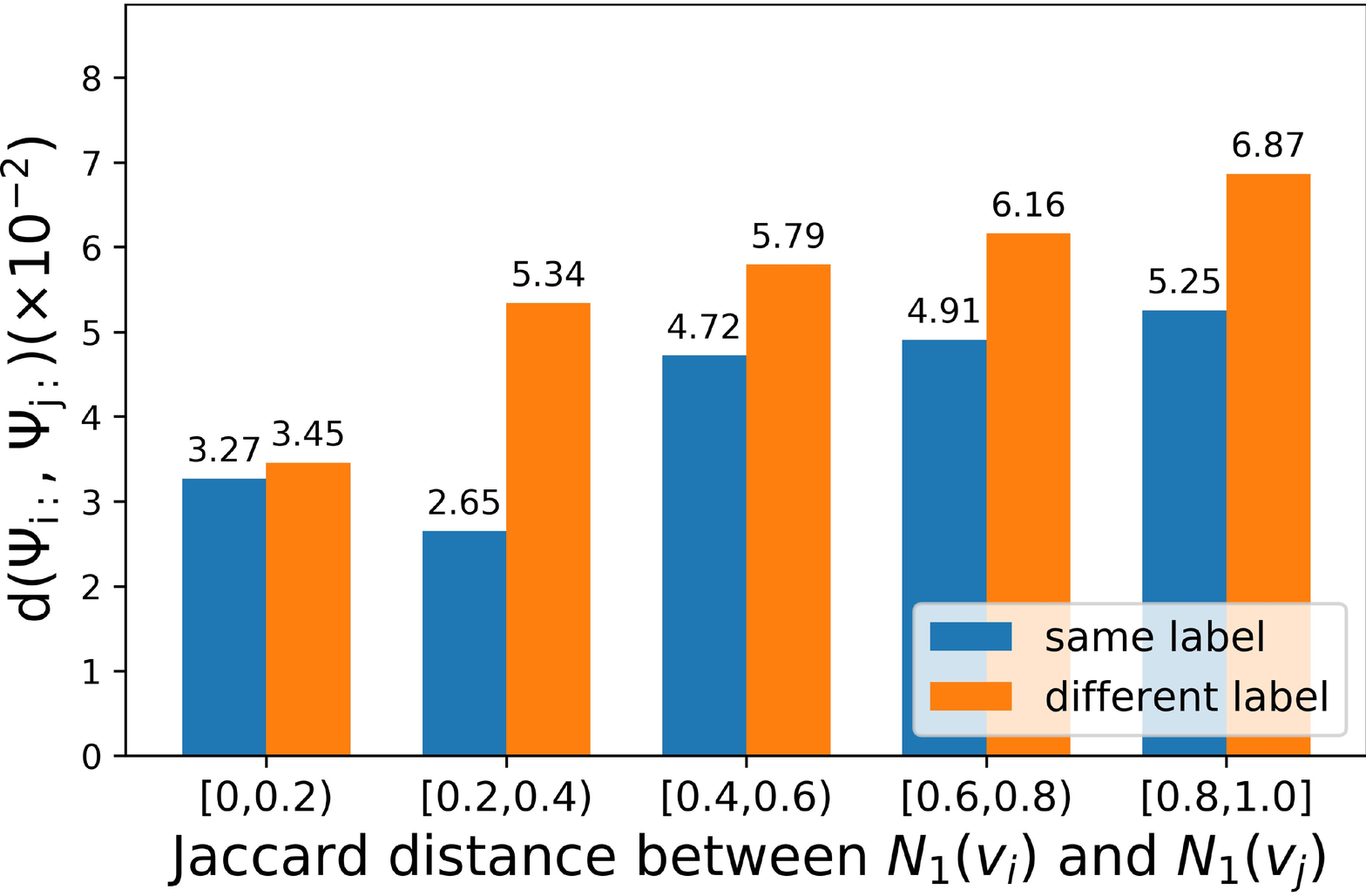}
   }
   \hspace{-2mm}
   \subfigure[Citeseer]{
   \includegraphics[width=40mm]{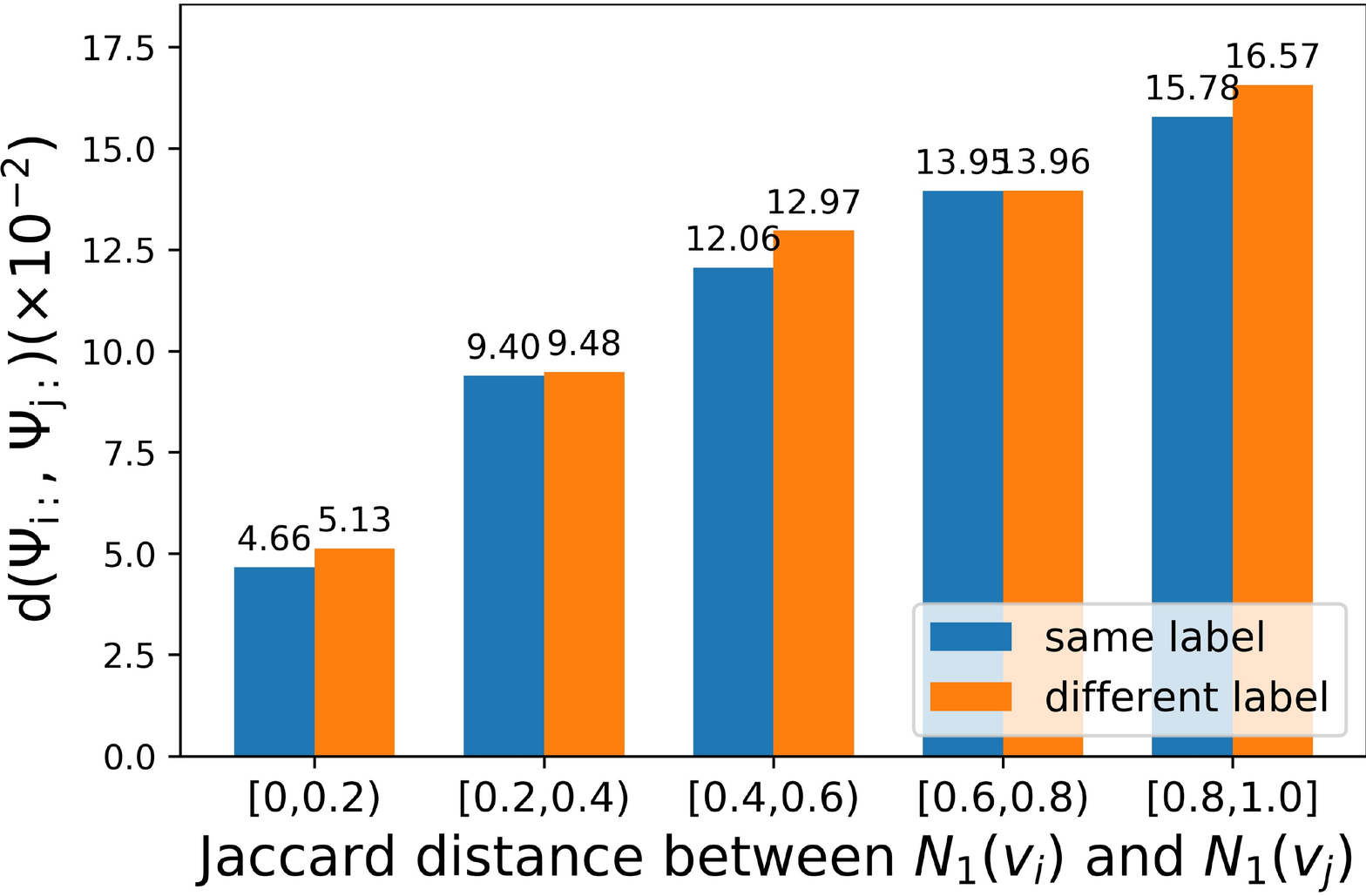}
   }
   \hspace{-2mm}
   \subfigure[Texas]{
   \includegraphics[width=40mm]{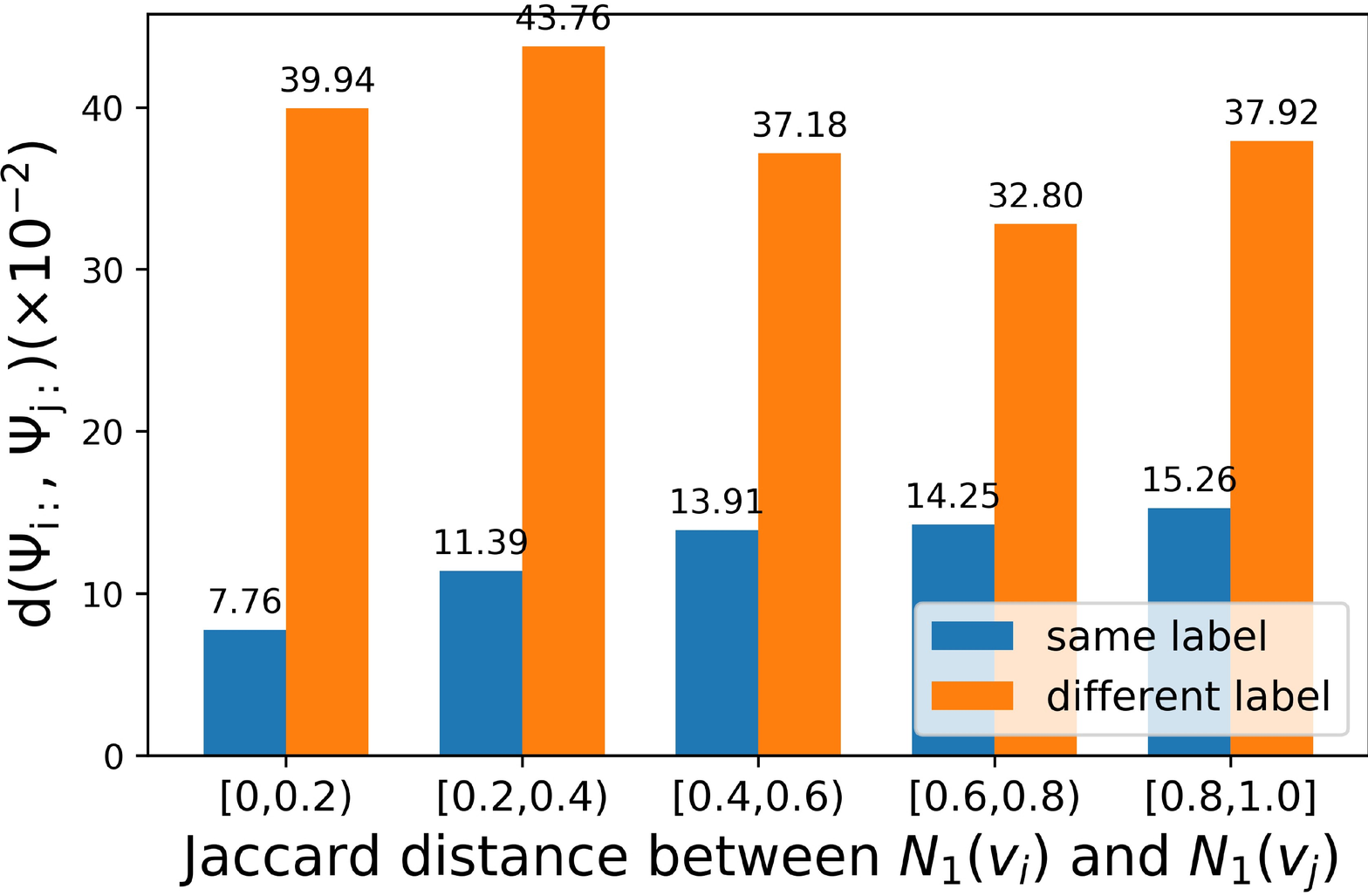}
   }
   \hspace{-2mm}
   \subfigure[Actor]{
   \includegraphics[width=40mm]{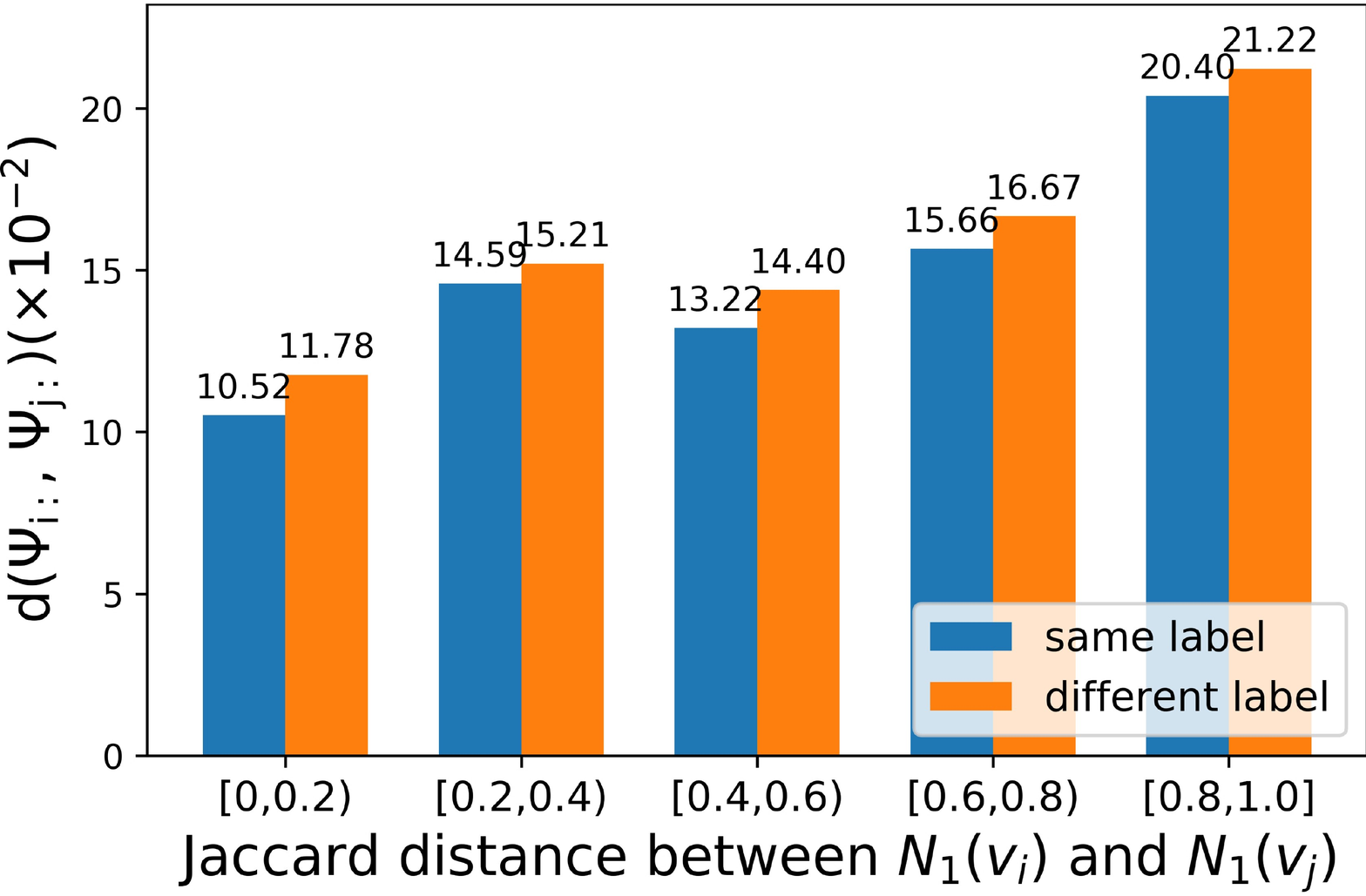}
   }
   \vspace{-10pt}
    \caption{Average coefficient distance $d(\Psi_{i:},\Psi_{j:})$ of node pairs $(v_i, v_j)$ range by Jaccard distance between $N_1(v_i)$ and $N_1(v_j)$.The blue bar denotes $d(\Psi_{i:},\Psi_{j:})$ with the same label, and the orange bar denotes $d(\Psi_{i:},\Psi_{j:})$ with different labels.}
   \vspace{-10pt}
    \label{fig::coefficient_analysis}
\end{figure*}

\subsection{Performance Comparison}
\textcolor{R2}{\textbf{Transductive node classification.}} The results of the node classification task under two split settings are reported in Table~\ref{tab::semi} and Table~\ref{tab::full}, respectively.  
\textcolor{revise}{To verify the improvements gained by NFGNN, we also perform the paired t-test between NFGNN and the second-best baselines on the datasets that NFGNN achieves the best performance, and report the p-values in Table~\ref{tab::semi} and Table~\ref{tab::full}. It can be seen that the results on different datasets are statistically significant with p $<$ 0.05 as tested via the paired t-test.
When the sparse splitting is employed for the semi-supervised learning, NFGNN shows superior performance on 6 datasets and demonstrated comparable results on the remaining 4 datasets in contrast to the baselines. In addition, under the full-supervised learning setting, NFGNN outperforms all the baselines on 7 datasets and achieves comparable results on the other 3 datasets.}
The performance variance between different methods can be relatively large on heterophily graphs under a full-supervised learning setting. Especially on Chameleon and Squirrel graphs, NFGNN outperforms the SOTA method BMGCN by a large margin, i.e., 2.83\% and 5.03\%, demonstrating the superiority of our method. What's more, as shown in Table~\ref{tab::large}, the proposed NFGNN still performs better than other global filter learning methods on graphs with large scales, which further verifies the effectiveness of the adaptive localized filtering.

\noindent\textcolor{R2}{\textbf{Inductive node classification.} Table~\ref{tab::inductive} reveals the impressive performance of NFGNN, not only on the observable subsets of the 8 datasets but also on each unobserved test subset. This can be attributed to the fact that NFGNN is equipped with the capability to adaptively filter each node, rather than relying on a generic global filter. As a result, NFGNN can effectively capture and utilize the unique characteristics and patterns present in different nodes, leading to its robust performance across various subsets compared to the previous spectral-based GNNs.}

Meanwhile, several interesting phenomena can be observed:
\textbf{i}) some GNNs are even inferior to MLP and LINK on some heterophilic graphs.
The performance of MLP shows that the utilization of node features is also very important for GNNs. Besides, the performance of LINK shows that there is still information in the topology that has not been captured by existing GNNs.
\textbf{ii}) The filter-learning-based methods generally have a good performance on both the  homophilic and heterophilic graphs, indicating that adaptive filter learning does have good transferability.



\subsection{Node-level Analysis}
\label{node_analysis}
The motivation of the proposed NFGNN is to solve the mixed local patterns discussed in Sect.~\ref{sect::motivation}. 
Therefore, we divide the test nodes into 5 different intervals according to the homophilic 1-hop neighbor ratio $h_{N_1}(v)$ and report the mean accuracy of each interval. The results of GCN, NFGNN with only $\Gamma$ (marked as NFGNN w/o NF), and NFGNN are shown in Fig.~\ref{fig::node_analysis}. It should be noticed that the NFGNN w/o NF is equivalent to learning a globally consistent filter using the Chebyshev polynomial. 
Different from GCN, NFGNN has a promising and similar performance on all five intervals as shown in Fig.~\ref{fig::node_analysis}(c) and (d). It indicates that NFGNN can effectively capture the various local patterns under the condition as long as the amount of trainable data is sufficient. Besides, both NFGNN and NFGNN w/o NF perform better than GCN on the semi-supervised node classification task, as shown in Fig.~\ref{fig::node_analysis}(a) and (b). It suggests that adaptive learning filters are no less expressive than pre-designed filters, even in the semi-supervised case.

As we discussed before, the local patterns can also be analyzed based on the neighborhood subgraphs of the nodes. In general, if the local patterns of nodes are similar, then the coefficients of the learned filter for that nodes should also be similar. Therefore, in addition to classification performance, we also perform a node-level experiment to analyze the ability of NFGNN to handle the mixing local patterns according to the coefficient matrix $\Psi$. Specifically, we also compute the 1-order neighbor's Jaccard distance $J(N_1(v_i), N_1(v_j))$ for all node pairs $\{(v_i, v_j):v_i, v_j \in \mathcal{V} \wedge i \neq j \}$ to measure the similarity of 1-order local pattern between nodes. Then, for each node pair $(v_i, v_j)$, we compute the average coefficients distance $d(\Psi_{i:}, \Psi_{j:})$ according to the intervals of $J(N_1(v_i), N_1(v_j))$ as shown in Fig.~\ref{fig::coefficient_analysis}. On the whole, as we can see, the larger $J(N_1(v_i), N_1(v_j))$ is, the correspondingly larger $d(\Psi_{i:}, \Psi_{j:})$. More specifically, the $d(\Psi_{i:}, \Psi_{j:})$ of the node pair with the same label is smaller than the node pairs with different labels. The visualization results illustrate that our NFGNN is at least capable of learning a variety of 1-hop local pattern properties.

\begin{table}[t]
\setlength{\tabcolsep}{2.5pt}
\scriptsize
\centering
\caption{Accuracy (\%) improvement of the node-oriented filtering (NF).}
\begin{tabular}{@{}c|c|cccccc@{}}
\toprule
Basis     &          & Cora  & Citeseer & Pubmed & Chameleon & Actor & Texas   \\ \midrule
\multirow{3}{*}{Monomial}  
          & w/o NF    & 78.15 & 66.60    & 82.28  & 61.79     & 38.88 & 91.80   \\
          & w/ NF   & 79.16 & 68.42    & 84.79  & 63.36     & 39.53 & 91.47   \\ 
          & Improv.  & (1.01)& (1.82)   & (2.51) & (1.57)    & (0.66)& (-0.33) \\ \midrule

\multirow{3}{*}{Bernstein} 
          & w/o NF    & 76.32 & 65.61    & 82.10  & 67.82     & 39.31 & 92.29  \\
          & w/ NF   & 78.41 & 66.22    & 83.09  & 69.93     & 40.77 & 93.37  \\ 
          & Improv.  & (2.09)& (0.61)   & (0.99) & (2.11)    & (1.46)& (1.08) \\ \midrule

\multirow{3}{*}{Chebyshev} 
          & w/o NF    & 76.07 & 65.11    & 84.02  & 68.48     & 39.11 & 92.47  \\
          & w/ NF   & 77.69 & 67.74    & 85.07  & 72.52     & 40.62 & 94.03  \\ 
          & Improv.  & (0.62)& (2.63)   & (1.05) & (4.04)    & (1.51)& (1.56) \\ \bottomrule
\end{tabular}
\label{tab::acc_improve}
\vspace{-15pt}
\end{table}

\begin{figure}[t]
    \centering
   \subfigure[Globally shared filter on homophilic graphs]{
   \includegraphics[width=42mm]{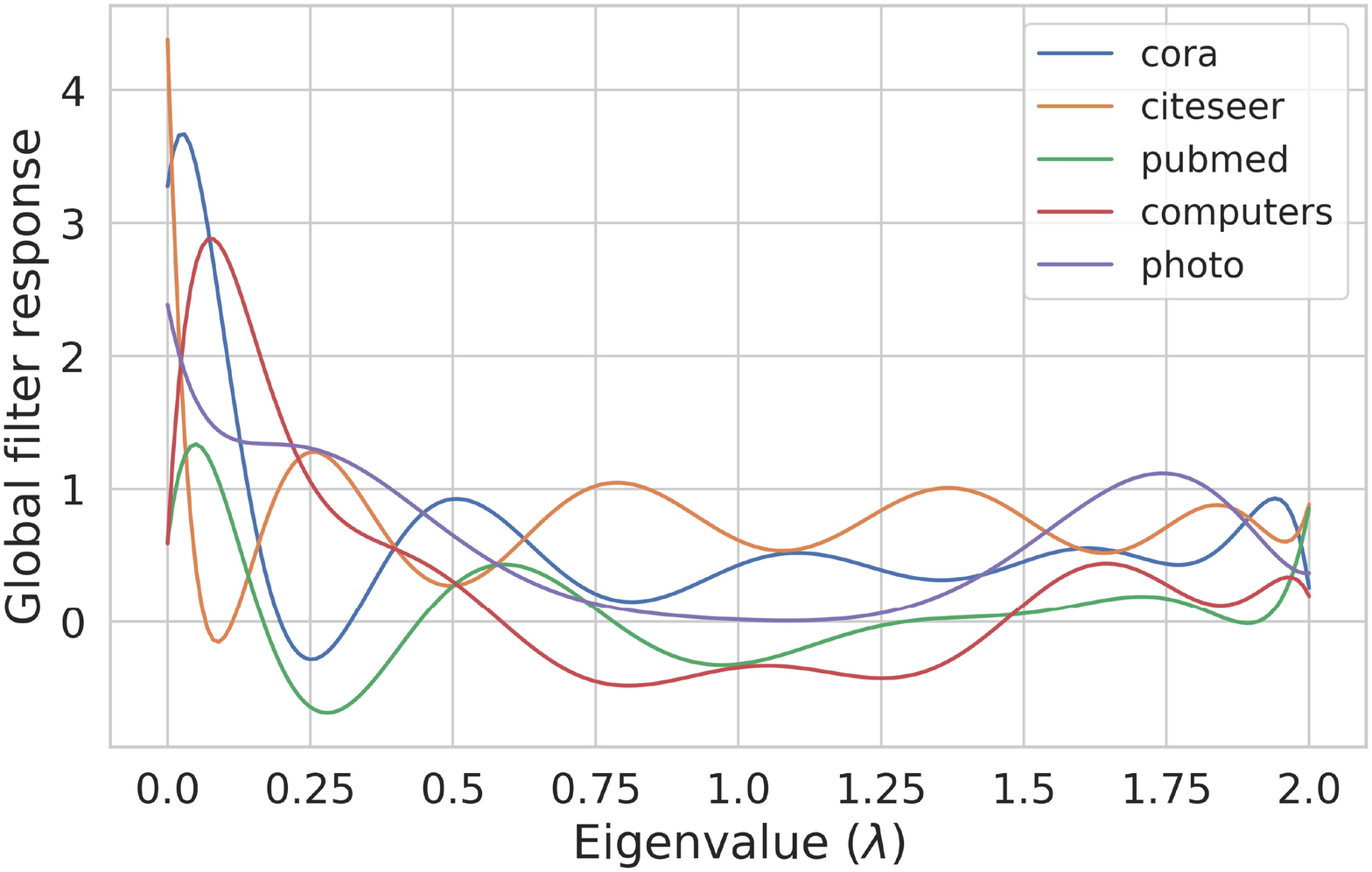}
   }
   \hspace{-1mm}
   \subfigure[Globally shared filter on heterophilic graphs]{
   \includegraphics[width=42mm]{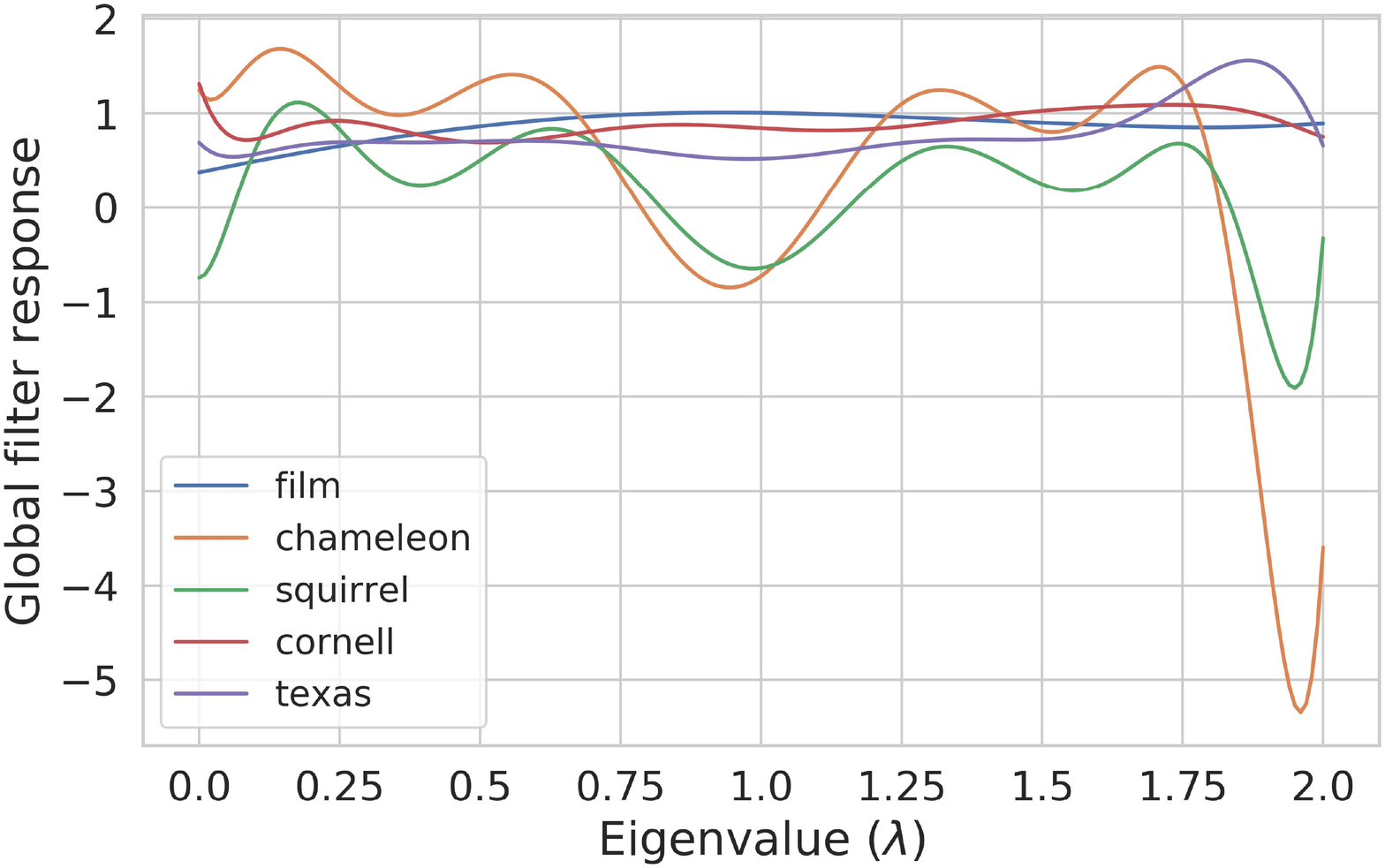}
   }
   \vspace{-7pt}
    \caption{Visualization of the learned globally shared filters on homophilic and heterophilic graphs, respectively. The filters in (a) are all roughly low-pass filters, while the filters in (b) are not identical.}
    \label{fig::visualization}
    \vspace{-7pt}
\end{figure}

\subsection{Effectiveness of the Node-oriented Filtering}
\noindent\textbf{\textcolor{R2}{Ablation Study.}}
To evaluate the effectiveness of the proposed node-oriented filtering more comprehensively, we first compare the performance of NFGNN and NFGNN w/o NF. Further, since the node-oriented filtering is independent of the polynomial basis, the Chebyshev basis is replaced by the Monomial basis and Bernstein basis, respectively, and we check the improvement brought by the node-oriented filtering mechanism for them. For the Bernstein basis, we refer to the implementation form given in~\cite{BernNet}. 
The results on six graphs are summarized in Table~\ref{tab::acc_improve}. Firstly, it can be seen that the globally consistent filters learned using three different bases have leading performance on different datasets, respectively, illustrating the effectiveness of using a polynomial approximation to learn filters. Particularly, the Monomial basis seems more suitable for semi-supervised node classification of homophilic graphs, while the Chebyshev basis performs well on heterophilic graphs. 
Furthermore, except for the Monomial basis on the Texas graph, the node-oriented filtering mechanism has different enhancements for each basis. The improvements not only validate the effectiveness of the proposed node-oriented filtering, but also demonstrate that the polynomial filter and the node-oriented filtering can each other to some extent. 

\noindent\textcolor{R2}{\textbf{Filter Visualization.}}
When we set $d = 1$, $\mathbf{\Gamma}$ can be seen directly as $\{\gamma_k\}_{k=0}^K$ in Eq.~(\ref{eq::Poly_Appro}).
Therefore, we can extract the global shared coefficient matrix $\mathbf{\Gamma}$ and visualize the learned \textbf{global shared filters} according to $\mathbf{\Gamma}$ and the Chebyshev bases. As shown in Fig.~\ref{fig::visualization}, the filters learned on homophilic graphs show strong low-pass properties, while the filters learned on heterophilic graphs have their own focus. Among them, the filter learned on the Chameleon graph shows a strong high-pass property, and the filter learned on the Squirrel graph shows a distinct comb-like property.

\textcolor{R2}{In addition, we also provide visualizations of the local node-oriented filters corresponding to nodes that exhibit different levels of homophily when $d > 1$.
we begin by selecting nodes into three subsets based on their homophily ratio within 2-hops:
$\mathrm{Subset}_{LP} = \{u: u \in \mathcal{V} \wedge h_{N_1}(v) > 0.7\}$, $\mathrm{Subset}_{BP} = \{u: u \in \mathcal{V} \wedge 0.4<h_{N_1}(v)<0.7 \wedge 0.4<h_{N_2}(v)<0.7\}$, and $\mathrm{Subset}_{HP} = \{u: u \in \mathcal{V} \wedge h_{N_1}(v) < 0.2 \wedge h_{N_2}(v) > 0.7 \}$.
Then, we randomly select 3 nodes from each subset and plot the frequency response curves of their corresponding filters as shown in Fig.~\ref{fig::Freq_Response}. It can be noted that the curves of the same color exhibit similar characteristics, whereas the curves of different colors display certain variations among them. The visualization of filters corroborates the effectiveness of NFGNN, i.e., NFGNN can learn filters adaptively based on the local patterns of the nodes.}

\noindent\textbf{\textcolor{R2}{Time Consumption.}} \textcolor{revise}{Except for what is mentioned above, we also report the average training time per epoch and the total training time on several datasets with different scales as numerical complexity analysis. It can be seen from Table~\ref{tab::time_analysis} that NFGNN is only a little slower than GPRGNN and much faster than BernNet, with respect to the average training time per epoch. Besides, in terms of the total training time, there  exists a certain difference between NFGNN and GPRGNN.
It may be due to the slower convergence of NFGNN compared to GPRGNN, and more epochs are needed for model training. Nevertheless, NFGNN is still much faster than BernNet.} 

\begin{table}[t]
\centering
\caption{\textcolor{revise}{Average running time per epoch (ms)/average total training time (s).}}
\label{tab::time_analysis}
\begin{tabular}{@{}ccccc@{}}
\toprule
          & NFGNN      & GPR-GNN    & BernNet    & GCN        \\ \midrule
\textcolor{revise}{Cora}      & \textcolor{revise}{6.1/1.89 }  & \textcolor{revise}{5.9/1.21}   & \textcolor{revise}{14.7/5.03}  & \textcolor{revise}{3.1/1.05}   \\
\textcolor{revise}{PubMed    }& \textcolor{revise}{6.9/2.52  } & \textcolor{revise}{6.4/1.49  } & \textcolor{revise}{15.7/5.92  }& \textcolor{revise}{3.4/1.28  } \\
\textcolor{revise}{Computers }& \textcolor{revise}{6.8/3.17  } & \textcolor{revise}{6.5/2.27  } & \textcolor{revise}{18.4/11.12 }& \textcolor{revise}{4.0/1.88  } \\
\textcolor{revise}{Actor     }& \textcolor{revise}{6.4/1.28  } & \textcolor{revise}{5.6/1.18  } & \textcolor{revise}{13.7/2.76  }& \textcolor{revise}{3.6/0.71  } \\
\textcolor{revise}{Squirrel  }& \textcolor{revise}{5.9/1.19  } & \textcolor{revise}{5.8/1.16  } & \textcolor{revise}{14.2/2.86  }& \textcolor{revise}{4.0/0.81  } \\
\textcolor{revise}{Cornell   }& \textcolor{revise}{5.9/1.16  } & \textcolor{revise}{5.2/1.03  } & \textcolor{revise}{15.5/3.8   }& \textcolor{revise}{3.0/0.63  } \\
\textcolor{revise}{pokec     }& \textcolor{revise}{25.9/12.33} & \textcolor{revise}{23.9/11.16} & \textcolor{revise}{29.8/14.67 }& \textcolor{revise}{41.8/19.92} \\ \bottomrule
\end{tabular}
\end{table}

\begin{figure}[t]
    \centering
   \subfigure[Cora]{
   \includegraphics[height=30mm]{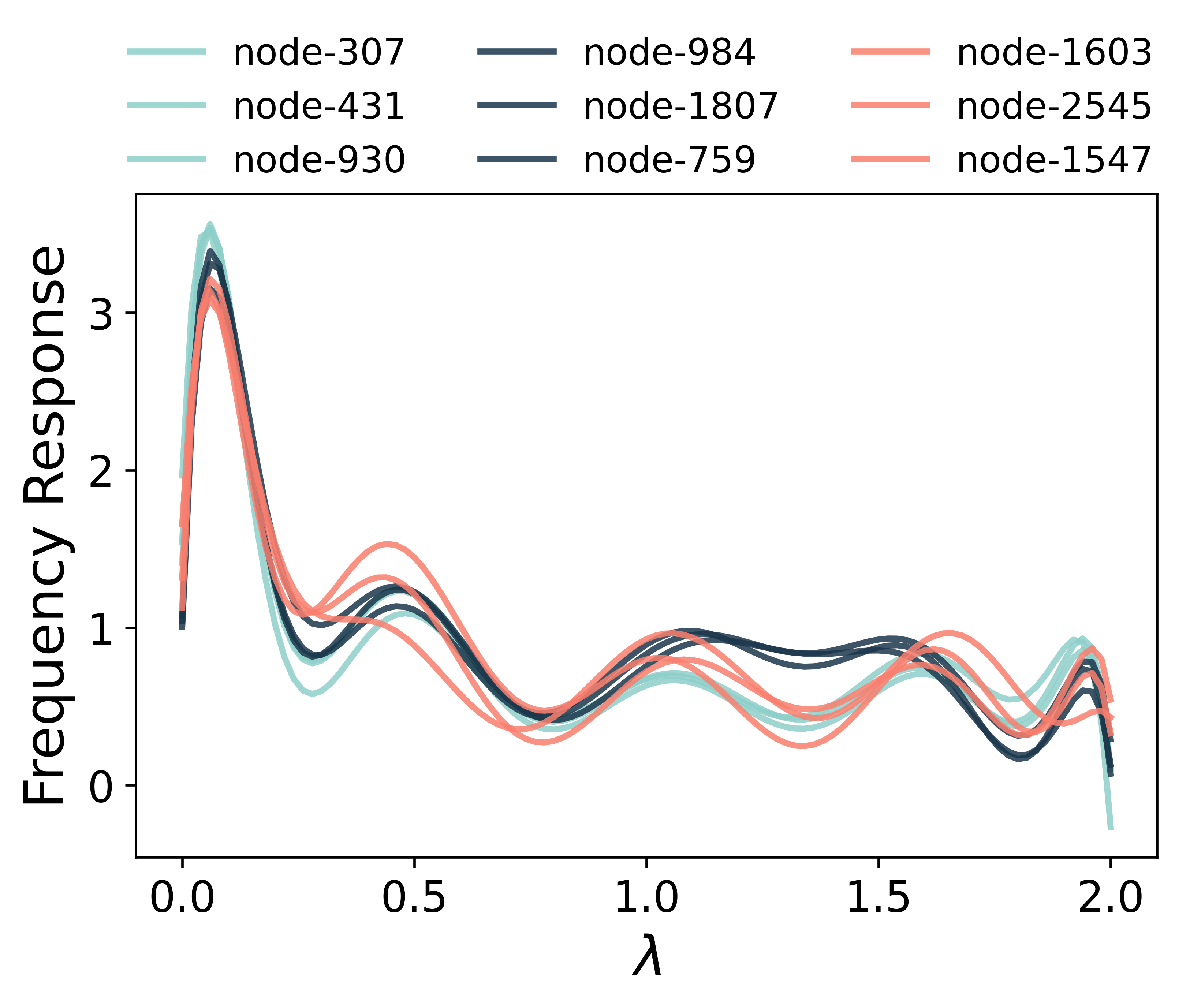}
   }
   \hspace{-2mm}
   \subfigure[Cornell]{
   \includegraphics[height=30mm]{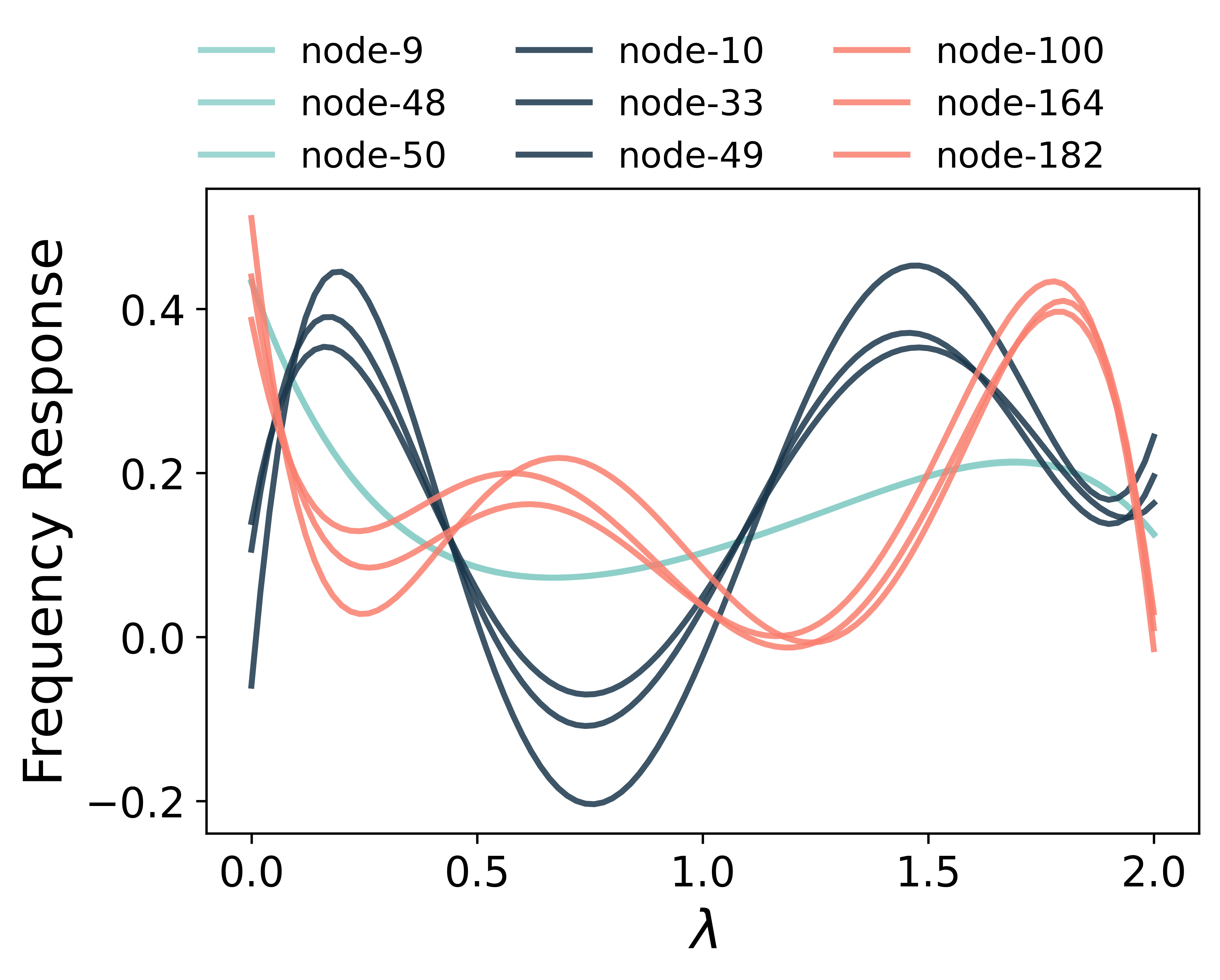}
   }
   \vspace{-10pt}
    \caption{\textcolor{R2}{Visualization of the node-oriented filters of nodes with different local homophilic patterns, where the green lines, the blue lines, and the red lines are the frequency responses of the filters of the nodes sampled from $\mathrm{Subset}_{LP}$, $\mathrm{Subset}_{BP}$, and $\mathrm{Subset}_{HP}$, respectively.}}
    \label{fig::Freq_Response}
\end{figure}

\begin{figure*}[t]
    \centering
   \subfigure[Cora]{
   \includegraphics[width=28mm]{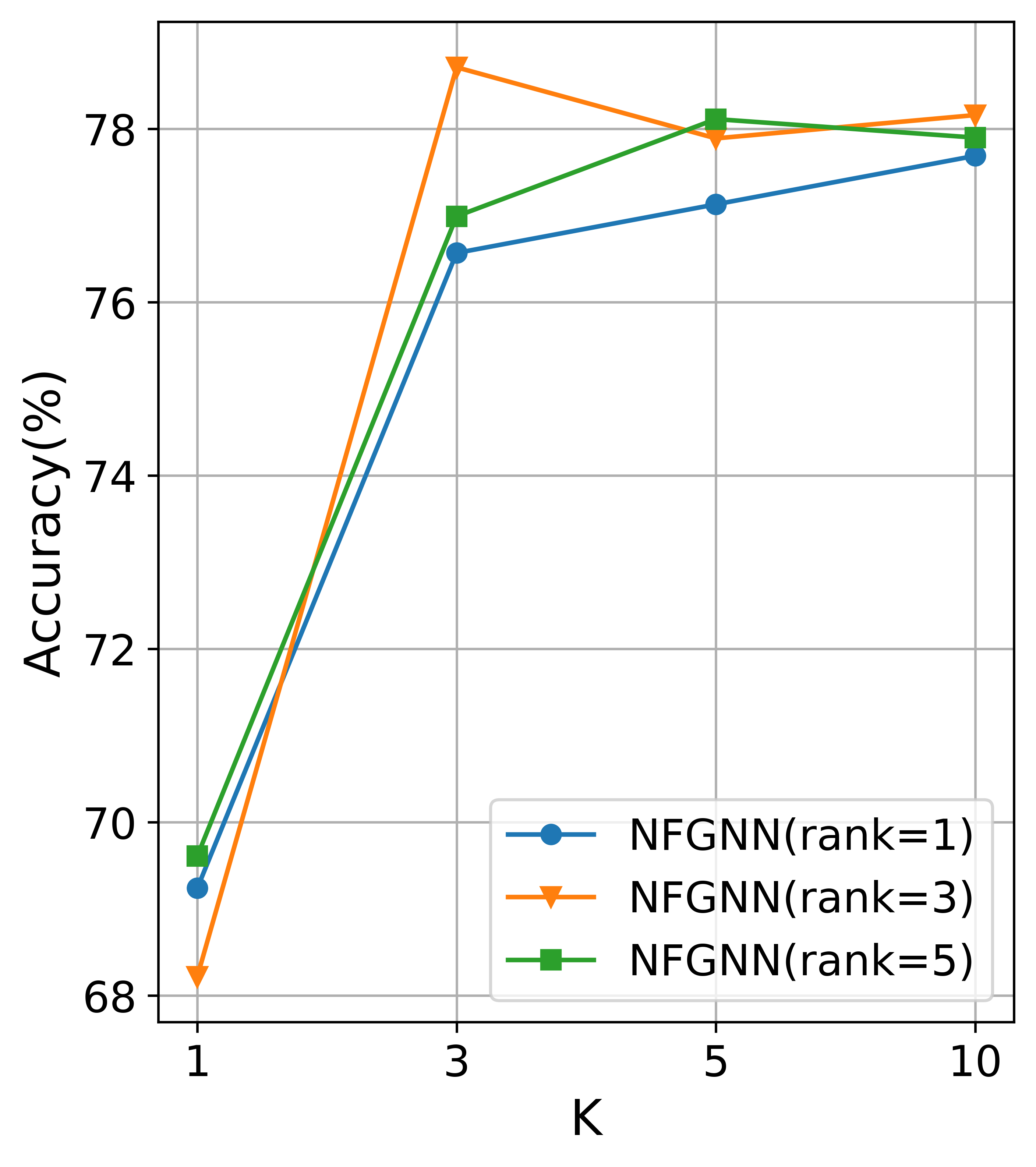}
   }
   \hspace{-2.5mm}
   \subfigure[Citeseer]{
   \includegraphics[width=28mm]{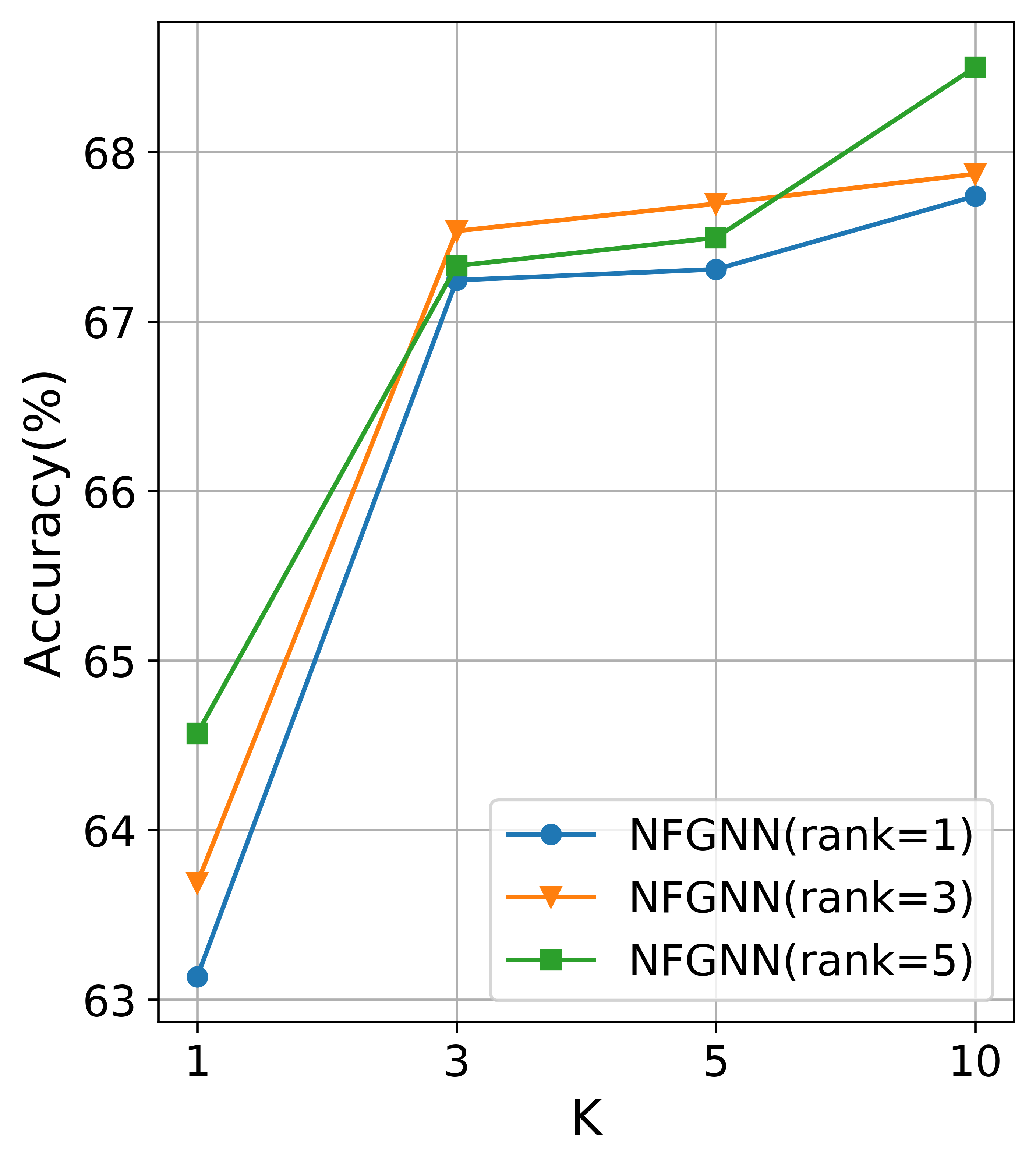}
   }
   \hspace{-2.5mm}
   \subfigure[Photo]{
   \includegraphics[width=29mm]{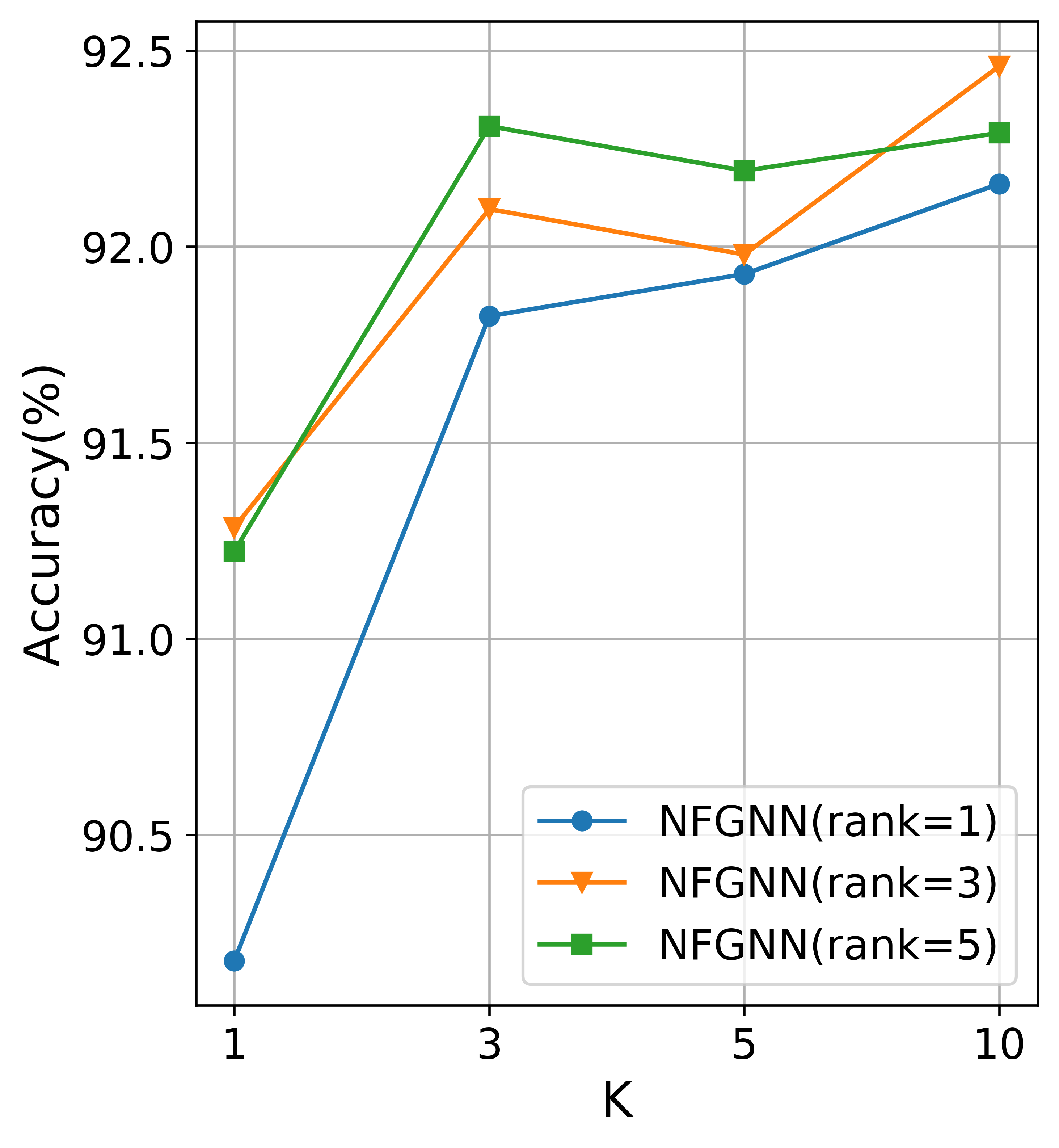}
   }
   \hspace{-2.5mm}
   \subfigure[Texas]{
   \includegraphics[width=29mm]{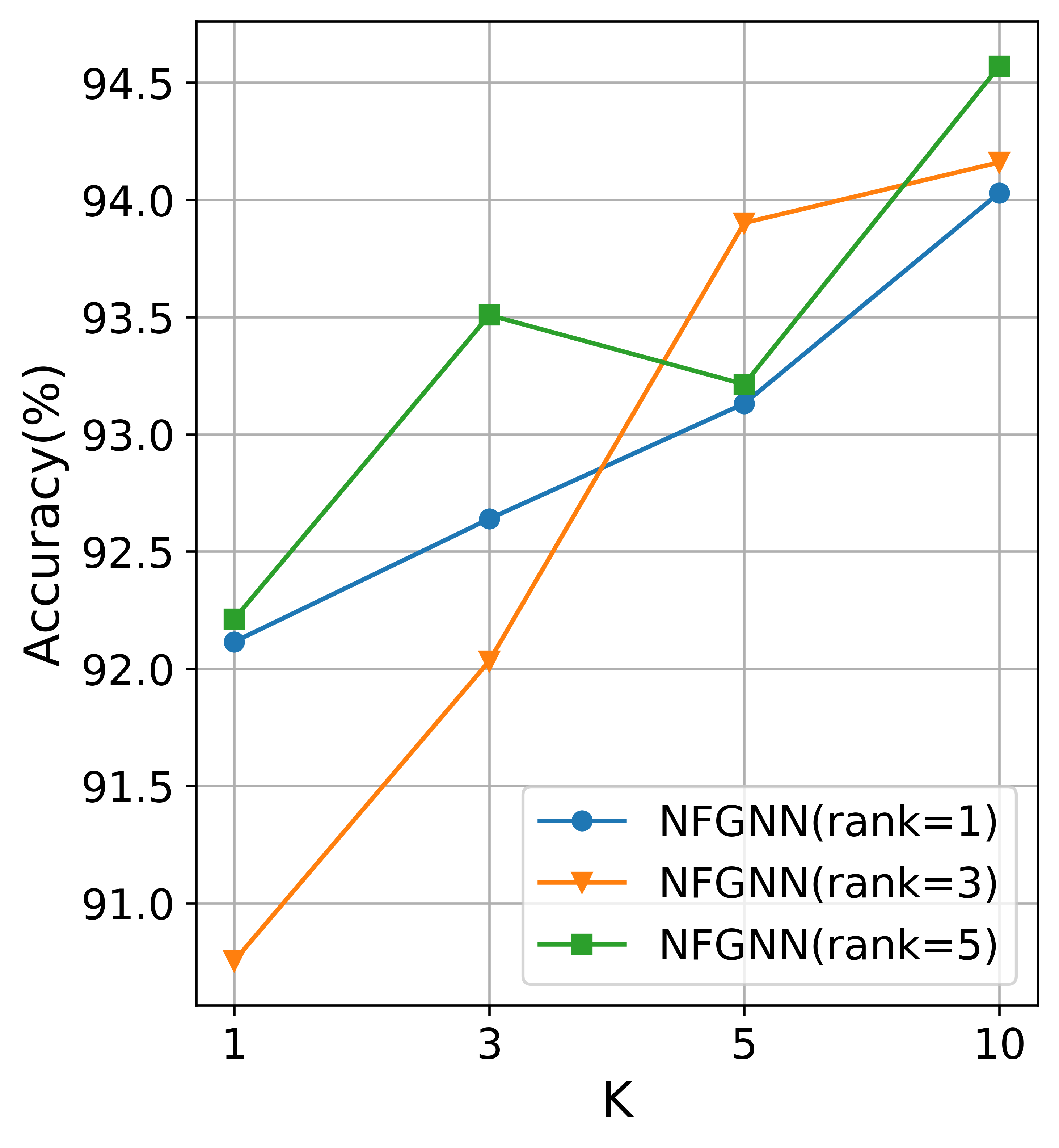}
   }
   \hspace{-2.5mm}
   \subfigure[Chameleon]{
   \includegraphics[width=28mm]{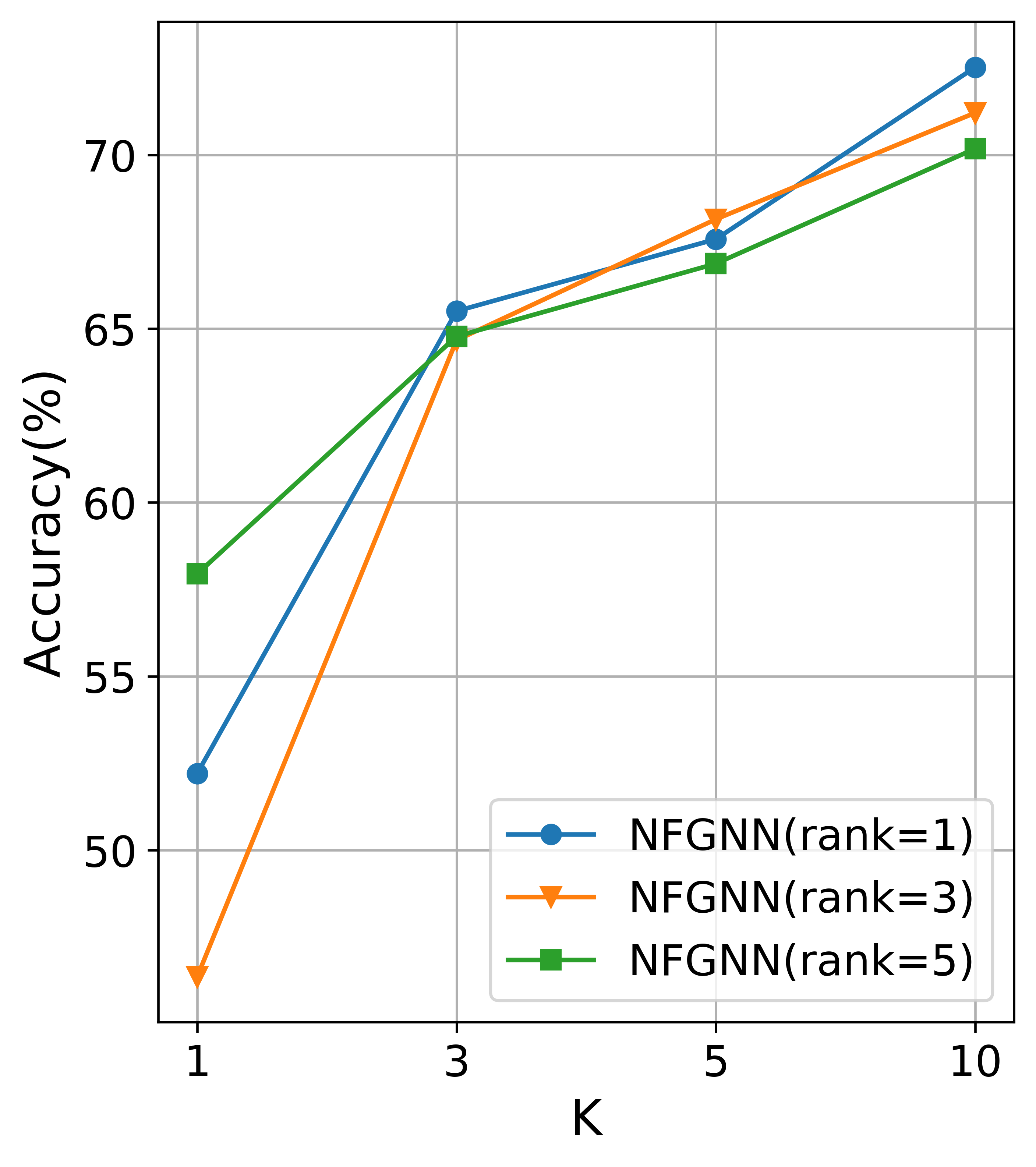}
   }
   \hspace{-2.5mm}
   \subfigure[Actor]{
   \includegraphics[width=29mm]{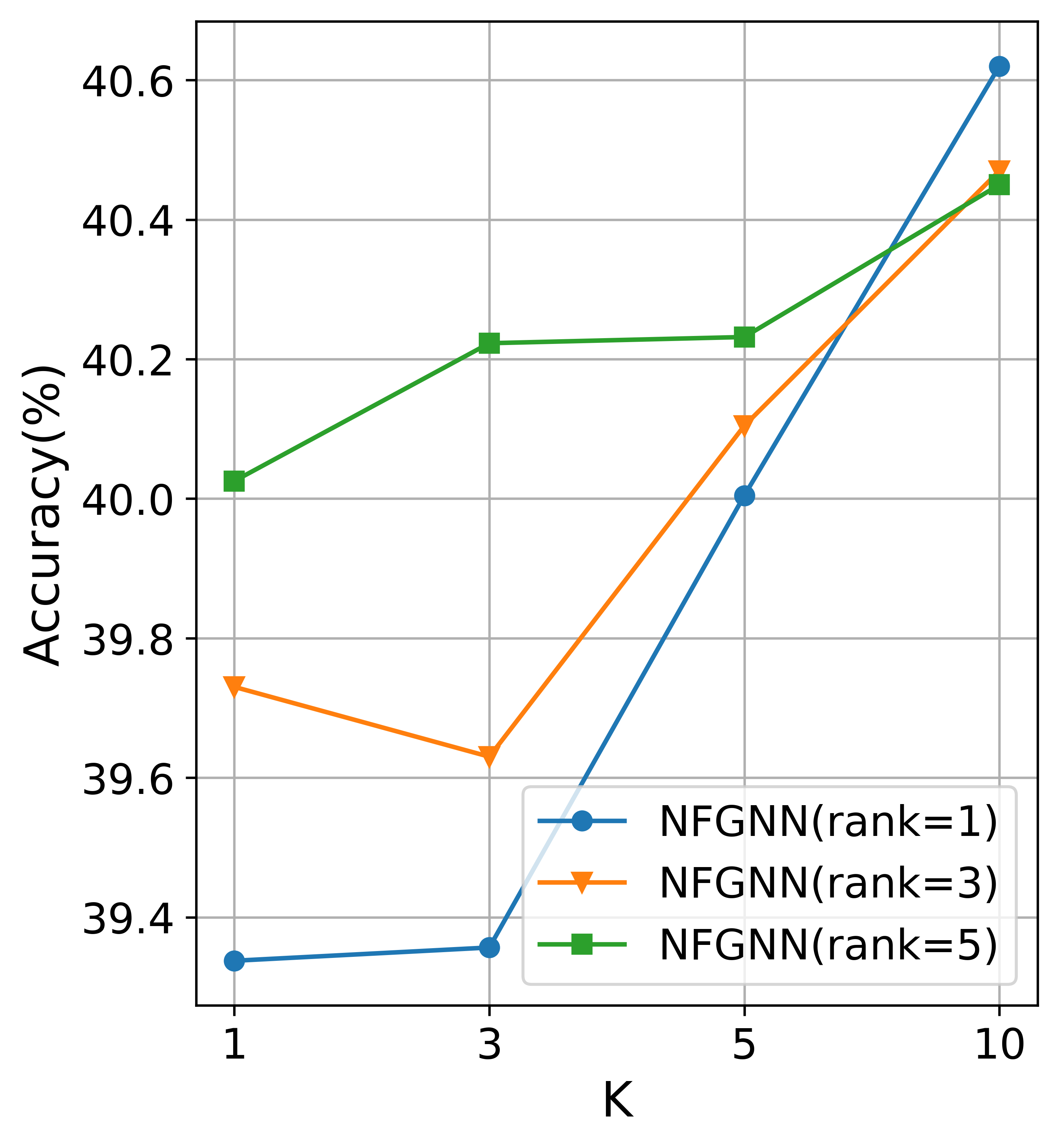}
   }
    \caption{The performance of NFGNN with increasing the rank $d$ from 1 to 5 and increasing $K$ from 1 to 10 .}
    \vspace{-5mm}
    \label{fig::low_rank}
\end{figure*}

\subsection{Analysis of the Low-rank Reparameterization}
To verify the effect of the order $d$ of the low-rank reparameterization and the order $K$ of the polynomial on the performance, we adjust $d$ within $\{1, 3, 5\}$ and $K$ within $\{1, 3, 5, 10\}$ , respectively. The results are reported in Fig.~\ref{fig::low_rank}. In general, performance tends to be improved with $K$ increasing. This is well understood because the polynomial fitting ability increases with the order of the polynomial increasing. Particularly, as we can see from Fig.~\ref{fig::low_rank}(a), (b), (c), and (e), NFGNN gets a large performance gain with $K$ increases from $1$ to $3$, while the gain is relatively marginal with $K$ increases from $3$ to $10$. This phenomenon coincides with our expectations since a strong low-pass filter or a strong high-pass filter can get a good fit by a low-order polynomial. On the contrary, Fig.~\ref{fig::low_rank}(d) and (f) show that some graphs indeed need higher-order polynomials to better fit the complex local patterns. For the order $d$ of the low-rank reparameterization, NFGNN achieves the best performance on Citeseer and Texas graphs when $d=5$, on Chameleon when $d=1$, and  on Cora when $d=3$. Meanwhile, the performance of NFGNN on Photo and Actor graphs does not differ much regardless of the setting of $d$. It shows the setting of $d$ can have some impact on the performance of NFGNN for various graphs.

\section{Discussions}
\subsection{Related GNNs and Their Connections with NFGNN}
\label{sect::GNNs_connection}
In this section, we further discuss some existing GNNs from a polynomial filtering perspective.

\noindent\textbf{ChebNet~\cite{chebyNet}.}  ChebNet first used polynomials to approximate filters, thus avoiding feature decomposition in spectral convolution. The filtering operation based on Chebyshev polynomials $T_k(\cdot)$ is:
\begin{equation}
\label{eq::chebyNet}
    \mathbf{z} = \sum^{K}_{k=0}\theta_k T_k(\tilde{\mathbf{L}})\mathbf{x}
\end{equation}
Theoretically, the Chebyshev polynomials could approximate arbitrary filters if the order $K$ is high enough. 

\noindent\textbf{GCN~\cite{GCN}.} Kipf et al. simplify ChebNet to 1-order polynomial approximation, and set $\theta = \theta_0 = -\theta_1$~\cite{GCN}. Thus, the filtering operation can be seen as:
\begin{equation}
\label{gcn}
    \mathbf{z} = \sum^{1}_{k=0}\theta_k T_k(\tilde{\mathbf{L}})\mathbf{x} = \theta_0(\mathbf{I} + \mathbf{D}^{-1/2}\mathbf{A}\mathbf{D}^{-1/2})\mathbf{x}\approx\theta\tilde{\mathbf{A}}\mathbf{x}
\end{equation}
where $\tilde{\mathbf{A}}$ is the symmetric normalized adjacency matrix with self-loops. Although GCN uses the renormalization trick and non-linear activation function, its filtering capability is still limited since the 1-order polynomial can only approximate a finite number of filters. 
We can derive from Eq.~(\ref{gcn}) that the frequency response of GCN layer is  $\hat{g}(\lambda)=(1-\tilde{\lambda})$. 
Hence, the frequency response of GCN is equivalent to $\hat{g}(\lambda)=(1-\tilde{\lambda})^K$ when stacking $K$ such layers. 
As $K$ tends to infinity, it will approximate an impulse low-pass filter, which means the filtered signal at each node is the same, i.e., over-smoothing. As variants of the GCN, SGC~\cite{SGC} and GCNII~\cite{GCNII} have similar forms.

\noindent\textbf{APPNP~\cite{APPNP}.} An improved propagation scheme is derived from PageRank in \cite{APPNP}, which is also equivalent to a polynomial filtering operation:
\begin{equation}
\label{APPNP}
\mathbf{z} = (1-\alpha)^K\tilde{\mathbf{A}}^K\mathbf{x} + \sum^{K-1}_{k=0}\alpha(1-\alpha)^k\tilde{\mathbf{A}}^k\mathbf{x}
\end{equation}
where $\alpha \in (0,1]$ is a hyper-parameter to control the teleport (restart) probability. The second term is  $K$-order monomial polynomial of $\tilde{\mathbf{A}}$ and the coefficient of the $k$-order is $\alpha(1-\alpha)^k$. It can be seen that the fixed $\alpha$ restricts the expressibility of APPNP, so that only a specific set of filters can be approximated. On the basis of APPNP, GPRGNN~\cite{GPRGNN} and BernNet~\cite{BernNet} improve APPNP via trainable parameters to learn the coefficients of the polynomial.

\noindent\textbf{ARMA~\cite{ARMA}.} A recursive and distributed formulation is used to implement the auto-regressive moving average (ARMA) filter. If we remove the non-linear activation function, the ARMA layer will be:
\begin{equation}
\label{ARMA}
    \mathbf{z} = \frac{1}{B}\sum_{b=1}^{B}\bar{\mathbf{x}}_b^{(K)},  
    \bar{\mathbf{x}}_b^{(K)} = \mathbf{M}^K{\mathbf{x}}\mathbf{W}_b^K + \sum_{k=0}^{K}\mathbf{M}^k{\mathbf{x}}(\mathbf{V}_b\mathbf{W}^k_b)
\end{equation}
where $\mathbf{M} = \mathbf{D}^{-1/}2\mathbf{A}\mathbf{D}^{-1/2}$, $\left\{\mathbf{W}_b\right\}_{b=1}^B$ and $\left\{\mathbf{V}_b\right\}_{b=1}^B$ are the trainable weight matrics. It can be seen that the second term is a $K$-order polynomial of $\mathbf{M}$ and the coefficient of the $k$-th order is integrated into the feature transformation matrix $\mathbf{V}_b\mathbf{W}^k_b$. Thus, ARMA is similar to ChebNet which can approximate arbitrary filters.

Recall the proposed NFGNN in Eq.~(\ref{eq::final}), it can be found that it is still within the framework of polynomial filtering as those models in Eqs.~(\ref{eq::chebyNet}),(\ref{gcn}),(\ref{APPNP}), and (\ref{ARMA}). Nevertheless, unlike their globally consistent filtering coefficients, the associated group of filtering coefficients with each node as in Eq.~(\ref{eq::final}) are explicitly node-oriented, which is the most essential difference between NFGNN and these methods.

\subsection{Property of spectral-based GNNs and Future work}
\textcolor{R2}{
Spectral-based graph neural networks typically estimate a globally consistent filter from the perspective of the whole graph~\cite{BernNet,chebyNet,ARMA} using polynomial-parameterized filter learning. 
Due to the local localization property of polynomial-parameterized filter learning, the global filter is equivalent to a trade-off filtering solution learned on the K-order neighborhood subgraphs centered at each node. It should be noted that there can be misalignment between the eigenvectors and eigenvalues of these subgraphs~\cite{stability}. Although the proposed NFGNN circumvents this potential problem of the global shared filter learning by providing a new form of the global and local perspective tradeoff in the spectral domain. 
This has still sparked thinking regarding the stability and transferability of spectral-based GNNs.}

\textcolor{R2}{To this end, ~\cite{stability} investigates the influence of changes in the underlying topology on the output of GNNs, and shows that spectral-based GNNs exhibit greater stability than linear graph filters when considering subgraphs with the same number of nodes. 
Additionally, the concept of graphon neural networks is introduced by~\cite{transferability} as limit objects of GNNs to prove the transferability of GNNs between different graphs sampled from the same graphon. 
Moreover~\cite{GraphDA} proposes spectral regularization that leverages the theory of optimal transport-based domain adaptation and graph signal processing to enhance the transferability of GNNs. }

\textcolor{R2}{Building upon the aforementioned research~\cite{stability,transferability,GraphDA}, conducting further analysis on the relationship between transferability and discriminability of local adaptive filtering in GNNs can strengthen the theoretical foundation and pave the way for more effective and robust graph representation learning. Besides, for spectral-based GNNs, the scalability of spectral convolution and inductive learning setting remain crucial issues to be addressed. Nonetheless, due to the significant transferability and theoretical characteristics it offer, spectral-based GNNs continue to be an area of interest and hold promise for future advancements in GNNs.}

\section{Conclusion}
\label{sect::conclusion}
In this paper, we initially conduct a comprehensive analysis of the local patterns in graph data and the aggregability of near-neighbors. Building upon these observations, we reconsider spectral-based GNNs and introduce NFGNN, a novel approach for node-oriented spectral filtering motivated by the generalized translated operator. In contrast to previous approaches that utilize a global filter, NFGNN applies local spectral filtering by employing filters translated to specific nodes, effectively addressing the challenge of local patterns. Through the re-parameterization strategy, the node-oriented filtering is implemented in a simple and efficient way. The experimental results on several real-world graph datasets verify that our NFGNN achieves more remarkable performance over currently available alternatives.

%




\ifCLASSOPTIONcaptionsoff
  \newpage
\fi



%
{
	\bibliographystyle{IEEEtran}
	\bibliography{TPAMI_2021.bib}
}
\end{document}